\documentclass{mcom-l}

\usepackage{lscape} 
\usepackage{hyperref}
\usepackage{comment}
\usepackage{xfrac}
\usepackage{graphicx}
\usepackage{caption}
\usepackage{subcaption}
\usepackage{csquotes}

\usepackage[nocompress]{cite}

\numberwithin{equation}{section}
\usepackage{booktabs}
\usepackage{soul}
\makeatletter

\renewcommand\paragraph{\@startsection{paragraph}{4}{\z@}%
                                    {3.25ex \@plus1ex \@minus.2ex}%
                                    {-1em}%
                                    {\normalfont\normalsize\bfseries}}
\makeatother

\usepackage{bbm}
\usepackage{mathtools}
\usepackage[capitalize,nameinlink,noabbrev]{cleveref}
\RequirePackage[authoryear]{natbib}
\allowdisplaybreaks
\usepackage{todonotes}

\usepackage{enumitem}
\usepackage{xcolor}
\usepackage{hyperref}

\newtheorem{theorem}{Theorem}
\newtheorem{lemma}[theorem]{Lemma} 
\newtheorem{assumption}{Assumption} 
\newtheorem{proposition}[theorem]{Proposition} 
\newtheorem{remark}{Remark}
\newtheorem{corollary}[theorem]{Corollary}
\newtheorem{definition}{Definition}

\begin{document}

\title[CLTs for eigenvalues of graph Laplacians on data clouds ]{Central limit theorems for the eigenvalues of graph Laplacians on data clouds}


\author{Chenghui Li}
\address{University of Wisconsin Madison}
\curraddr{}
\email{cli539@wisc.edu}
\thanks{CL gratefully acknowledges support
from the IFDS at UW-Madison through NSF TRIPODS grant 2023239.}

\author{Nicolás García Trillos}
\address{University of Wisconsin Madison}
\curraddr{}
\email{garciatrillo@wisc.edu}
\thanks{NGT was supported by NSF-DMS grant 2236447}

\author{Housen Li}
\address{Georg-August-Universität Göttingen}
\curraddr{}
\email{housen.li@mathematik.uni-goettingen.de}
\thanks{HL is funded by the DFG (German Research Foundation) under Germany's Excellence Strategy, project EXC~2067 {Multiscale Bioimaging: from Molecular Machines to Networks of Excitable Cells} (MBExC)}

\author{Leo Suchan}
\address{Georg-August-Universität Göttingen}
\curraddr{}
\email{leo.suchan@uni-goettingen.de}
\thanks{LS is supported by the DFG (German Research Foundation) under project GRK 2088: \enquote{Discovering Structure in Complex Data}, subproject A1}

\subjclass[2010]{62G20 60F05 58J50 35P15 68R10 60D05 }

\date{}

\dedicatory{}

\begin{abstract}
Given i.i.d.\ samples $\X_n =\{ x_1, \dots, x_n \}$ from a distribution supported on a low dimensional manifold $\M$ embedded in Eucliden space, we consider the graph Laplacian operator $\Delta_n$ associated to an $\veps$-proximity graph over $\X_n$ and study the asymptotic fluctuations of its eigenvalues around their means. In particular, letting $\hat{\lambda}_l^\veps$ denote the $l$-th eigenvalue of $\Delta_n$, and under suitable assumptions on the data generating model and on the rate of decay of $\veps$, we prove that $\sqrt{n } (\hat{\lambda}_{l}^\veps - \E[\hat{\lambda}_{l}^\veps] )$ is asymptotically Gaussian with a variance that we can explicitly characterize. A formal argument allows us to interpret this asymptotic variance as the dissipation of a gradient flow of a suitable energy with respect to the Fisher-Rao geometry. This geometric interpretation allows us to give, in turn, a statistical interpretation of the asymptotic variance in terms of a Cramer-Rao lower bound for the estimation of the eigenvalues of certain weighted Laplace-Beltrami operator. The latter interpretation suggests a form of asymptotic statistical efficiency for the eigenvalues of the graph Laplacian. We also present CLTs for multiple eigenvalues and through several numerical experiments explore the validity of our results when some of the assumptions that we make in our theoretical analysis are relaxed.  

\medskip

\small	
  \textbf{\textit{Keywords:}} manifold learning, graph Laplacian, Laplace-Beltrami operator, CLT, Fisher-Rao geometry, gradient flow, Cramer-Rao lower bound. 
\end{abstract}

\maketitle

\section{Introduction}

In the past decades, researchers from a variety of disciplines have studied the use of spectral geometric methods, and in particular the use of spectral information from graph Laplacian matrices, to process, analyze, and interpret data sets. These methods have been used in supervised learning \cite{ando2006learning,belkin2006manifold,smola2003kernels}, clustering \cite{ng2001spectral,von2007tutorial}, dimensionality reduction \cite{belkin2001laplacian,coifman2005geometric}, and contrastive learning \cite{haochen2021provable,li2025consistency,li2023spectral}.

In the literature, it is possible to find multiple versions of graph Laplacians over data clouds but the most popular one is probably the unnormalized graph Laplacian, which we will focus on in this work. Let $\X_n= \{x_1,\dots,x_n\}$ be a size $n$ i.i.d. sample from a distribution $\rho$ supported on an unknown $m$-dimensional manifold $\M$ embedded in a high-dimensional Euclidean space $\R^d$. In operator form, the (unnormalized) graph Laplacian over the data set $\X_n$ is defined as
\begin{equation}\label{eq-def La^eps}
  \Delta_n u(x_i)=\frac{1}{n \eps^{m+2}} \sum_{j=1}^n \eta\left(\frac{|x_i-x_j|}{\eps}\right)\left(u(x_i)-u\left(x_j\right)\right),\quad  i=1,2,\dots,n,
\end{equation}
where $\eps > 0$ is a suitably chosen connectivity (or proximity) parameter that is related to $n$, $\eta$ is a non-negative and decreasing function, and $\lvert \cdot \rvert$ denotes the Euclidean norm in the ambient space $\R^d$. Through several works in the manifold learning literature, it is by now well known that different attributes of the operator $\Delta_n$ can be used as consistent estimators for analogous attributes of the differential operator
\begin{equation}
	    \Delta_{\rho} f = - \frac{1}{\rho} \divergence\left(   \rho^2 \nabla f  \right),
     \label{eq:Laplace-Beltrami Operator}
	\end{equation}
    where $\divergence$ and $\nabla$ denote the manifold's divergence and gradient operators. Most notably, and given the importance that eigenpairs of graph Laplacians have for machine learning algorithms, multiple works in the literature have explored the \textit{spectral convergence} of $\Delta_n$ toward $\Delta_\rho$; see our literature review in section \ref{sec:literature review} for examples of these works. 

A key recent contribution in this direction is the work \cite{trillos2024}, where the problem of estimating eigenpairs of $\Delta_\rho$ is formulated in rigorous statistical terms. In that paper, it is proved that, by choosing the connectivity parameter $\veps$ appropriately, eigenpairs of the graph Laplacian are essentially \textit{minimax optimal} \footnote{We emphasize the use of the word \textit{essentially},
as there are subtleties that need to be taken into account in this statement, all of which are discussed in detail in \cite{trillos2024}.} for the problem of estimating eigenpairs of the operator $\Delta_\rho$ when $(\M, \rho)$ belongs to certain class of sufficiently regular data generating models. While this type of non-asymptotic analysis is essential for theoretically guaranteeing the consistency of many algorithms used in machine learning, it does not address other statistical questions about graph-based quantities over data clouds that seem quite natural. For example, to the best of our knowledge, in the context of proximity graphs over data clouds no paper in the literature has studied the finer fluctuations of eigenvalues and eigenvectors of the graph Laplacian around their means; note that the results in \cite{trillos2024} (as well as in some other papers) can be interpreted as laws of large numbers for eigenpairs of $\Delta_n$.


\medskip

In this paper, we prove that, under suitable assumptions on the data generating model $(\M, \rho)$ (enough regularity and an eigengap assumption on the operator $\Delta_\rho$), and provided that $\veps$ decays to zero slowly enough with $n$, then, as $n \rightarrow \infty$, we have 
\[ \sqrt{n}(\hat{\lambda}_l^\veps - \E[\hat{\lambda}_l^\veps]) \dto \mathcal{N}(0, \sigma_l^2),  \]
for an asymptotic variance $\sigma_l^2= \sigma_l^2(\rho)$ that we can fully characterize in terms of the eigenpairs of $\Delta_\rho$ and for which we can provide consistent estimators; see \eqref{eq-def:sigma^2} for the definition of $\sigma_l^2$, and \cref{thm1} for the precise statement of our result.



In order to prove our CLT, we begin our analysis by finding a convenient decomposition of $\sqrt{n}(\hat{\lambda}_l^\veps - \E[\hat{\lambda}_l^\veps])$ as a sum of a \textit{bulk} term and an \textit{error} term. The bulk term is a U-statistic-like random variable that we show converges in distribution toward $\mathcal{N}(0, \sigma_l^2)$, following a common strategy in the literature of U-statistics where one attempts to reduce the characterization of the asymptotic distribution of the U-statistic to establishing a CLT for a triangular array of independent random variables. To analyze the error term, and specifically prove that it vanishes as $n \rightarrow \infty$, we use the Efron-Stein inequality to show that, under suitable assumptions on the rate of decay of the connectivity parameter $\veps$, the variance of the error term vanishes as $n \rightarrow \infty$. In order to use the Efron-Stein inequality, however, we must develop a careful leave-one-out perturbation analysis where we quantify the effect that removing one data point from the data set has on the eigenvalues and eigenvectors of the graph Laplacian. Our careful leave-one-out error analysis is crucial to our proofs, allowing us to get much sharper error estimates than a more standard Davis-Kahan type theorem that compares eigenvectors and eigenvalues of $\Delta_n$ and $\Delta_\rho$ (as is done in \cite{calder2019consistency}). With minor modifications to our analysis, we can also obtain CLTs for multiple eigenvalues. Indeed, under additional eigengap assumptions for the operator $\Delta_\rho$, we demonstrate in \cref{thm:multi normal distribution} that, as $n \rightarrow \infty$, 
\[  \sqrt{n}(\hat{\lambda}^\eps_1-\E[\hat{\lambda}^\eps_1],\dots,\hat{\lambda}^\eps_l-\E[\hat{\lambda}^\eps_l])\dto\mathcal{N}(0,\Sigma), \]
for a $l \times l$ covariance matrix $\Sigma$ that, again, we can fully characterize.

\medskip 

In addition to establishing our CLTs for the eigenvaluyes of the graph Laplacian, we investigate the quantity $\sigma_l^2= \sigma_l^2(\rho)$ and provide geometric and statistical interpretations for it. We first present a formal argument that allows us to view $\sigma_l^2= \sigma_l^2(\rho)$ as the \textit{dissipation} at $\rho$ along the \textit{gradient flow of a suitable energy with respect to the Fisher-Rao geometry} over the space of probability distributions on $\M$. This in particular allows us to think of the quantity $\sigma_l^2$ as the largest possible instantaneous change in the eigenvalue $\lambda_l= \lambda_l(\rho)$ when the distribution $\rho$ is infinitesimally perturbed. More importantly, this geometric characterization suggests, in turn, a deep statistical interpretation of $\sigma_l^2$. Indeed, we prove that if $\tilde{\lambda}_l$ is an \textit{unbiased} estimator for $\lambda_l(\rho)$, then its variance must satisfy the lower bound
\[ \Var_\rho(\tilde{\lambda}_l) \geq \frac{\sigma_l^2}{n};\]
see \cref{thm:CramerRao} for a precise statement. 
In other words, the terms $\sigma_l^2$ are precisely the constants appearing in the Cramer-Rao type inequalities for the non-parametric problems of estimating the eigenvalues of the differential operator $\Delta_\rho$. Although the graph Laplacian eigenvalue $\hat{\lambda}_l^\veps$ is not unbiased for $\lambda_l$ (in section \ref{sec:bias}, we present a discussion on the bias), it is certainly true that it is a consistent estimator for $\lambda_l$ (i.e., its bias is asymptotically equal to zero). Therefore, we believe that the fact that the asymptotic variance of $\hat{\lambda}_l^\veps$ coincides with the lower bound in the Cramer-Rao inequality described above \textit{suggests} a form of statistical efficiency that we find surprising, interesting, and worth of further investigation in the future. 


\medskip 

Our theoretical results are complemented with numerical simulations where we: 
\begin{enumerate}
    \item Examine the asymptotic normality of the eigenvalues of the graph Laplacian. The simulations show that, as the sample size $n$ increases, the distribution of the graph Laplacian's eigenvalues becomes increasingly close to a normal distribution, even when the eigengap condition is violated. The relaxation of this eigengap assumption thus seems possible, and remains an interesting open question.
    \item Compare the theoretical asymptotic variance $\sigma_l^2$ with an empirical estimate introduced in \Cref{sec:extension}. Our simulations show a strong agreement between the two, even for moderate sample sizes $n$.
    \item Investigate the possibility of centralizing the graph Laplacian eigenvalues around the corresponding eigenvalues of $\Delta_\rho$. Our theoretical discussion suggests that we should not expect a CLT with the expectation $\E[\hat{\lambda}_l^\veps]$ replaced with the true eigenvalue ${\lambda}_l$ of $\Delta_\rho$ when the manifold's intrinsic dimension exceeds $3$. Simulations confirm this restriction on dimension.
    \item Explore the structure of asymptotic covariance matrix. Simulations conducted beyond the regime where the eigengap condition holds suggest a relationship between the asymptotic dependency structure and the geometry of the underlying manifold. This observation warrants further investigation.    
\end{enumerate}



\subsection{Setup and Assumptions} We begin our discussion by stating some regularity assumptions on the data generating model $(\M, \rho)$ and by introducing some notation that we use in the sequel. 
    \begin{definition}[Manifold class $\MM_S$]
	We denote by $\MM_{S}$ the family of 
	$m$-dimensional manifolds $\M$ embedded in $\R^d$ that are smooth, compact, orientable, connected, and have no boundary, and in addition satisfy (for simplicity) that their total volume is 1 (according to their own volume forms.)
		\label{def:ManifoldClass}
	\end{definition} 
	In \cref{App:GeoBack}, we provide a brief review of some basic notions in Riemannian geometry that we use throughout the paper.

	\begin{definition}[Density class $\mathcal{P}_\M^\alpha$] 
		\label{def:DensityClass}
		For a given $\M \in \MM_S$ and $\alpha \in (0,1)$, we denote by $\mathcal{P}_\M^\alpha$ the class of probability density functions $\rho: \M \rightarrow \R$ that are $C^{2, \alpha}(\M)$ (i.e., twice-differentiable with $\alpha$-H\"older continuous second derivatives) and satisfy: 
		\begin{equation}\label{eq:rho bound}
		\rho_{\min} < \rho(x)  < \rho_{\max}, \quad  x\in \M, 
		\end{equation}
       for some $\rho_{min}, \rho_{max} >0$.		
       
	\end{definition}

\begin{assumption}
Let $\M \in \MM_S$ and $\rho \in \mathcal{P}_\M^\alpha$ for some $\alpha \in (0,1)$. Throughout the paper, we assume that $x_1, \dots, x_n$ are i.i.d. samples from $\rho$.
\label{assump:DataGenerating}
\end{assumption}



\medskip

After making our assumptions on the data generating model $(\M, \rho)$, we continue by stating the precise assumptions on the kernel $\eta$ and on the connectivity parameter $\veps$. Recall that, together with the data set $\X_n= \{ x_1, \dots, x_n\}$, these parameters determine the graph Laplacian $\Delta_n$.

\begin{assumption}
\label{assump:Kernel}
    The kernel $\eta : [0,\infty) \to [0,\infty)$ is assumed to be a non-increasing Lipschitz continuous function with support on the interval $[0, 1]$. We also assume, for convenience, that the kernel is normalized according to
\begin{equation}
    \frac{1}{2}\int_{\R^m} \eta(|v|) |v_1|^2\dd v=1,
    \label{eqn:AssumpKernel}
\end{equation}
where $v_1$ is the first coordinate of the vector $v \in \R^m$.
\end{assumption}

\begin{remark}
Condition \eqref{eqn:AssumpKernel} is a convenient assumption that allows us to ignore certain normalization factors in our analysis. Note that an arbitrary $\eta : [0,\infty) \to [0,\infty)$ that is non-decreasing, Lipschitz, and has support in $[0,1]$ can be turned into a kernel $\tilde{\eta}$ satisfying \cref{assump:Kernel} by considering
\[ \tilde{\eta}:= \frac{2}{\sigma_\eta} \eta, \]
where 
\[ \sigma_\eta:= \int_{\R^m} \eta(|v|) |v_1|^2\dd v.\]
\end{remark}

The following assumption on the connectivity parameter $\eps$ guarantees that it is small enough for $\Delta_n$ to localize, while large enough for the proximity graph to remain sufficiently connected.

\begin{assumption}\label{assum:Epsilon}
Throughout the paper we will implicitly assume that the connectivity parameter $\veps$ satisfies
\[ C_\M \left( \frac{\log(n)}{n} \right)^{1/m} \leq \veps \leq  \min \left\{1, \frac{R_\M}{2}, K_\M^{-1 / 2}, i_\M\right\},  \]
where $R_\M, K_\M$ are uniform upper bounds on the reach and on the absolute values of the sectional curvatures of $\M$. $i_\M$, on the other hand, is the injectivity radius of $\M$, and $C_\M$ is a large enough constant that depends on $\M$ and $\rho_{\mathrm{min}}, \rho_{\mathrm{max}}$.


\end{assumption}

\begin{remark}
The lower bound in \cref{assum:Epsilon} guarantees that, with very high probability (more precisely, probability at least $1-Cn\exp(-n\eps^m)$), the $\veps$-proximity graph implicit in the definition of the graph Laplacian is connected. In particular, we may assume without the loss of generality that $\hat{\lambda}_2^{\eps}>0$. The upper bound, on the other hand, allows us to use Riemannian normal coordinates (see \cref{App:GeoBack} for definitions) to write certain integrals and take advantage of useful Taylor expansions of geometric quantities. 

Note that \cref{assum:Epsilon} is implicitly satisfied if we assume that
\[ \left(\frac{\log n}{n} \right)^{\frac{1}{m}}\ll\eps\ll 1,   \]
where by this we mean that $\frac{1}{\veps}\left(\frac{\log n}{n} \right)^{\frac{1}{m}} \to 0 $ and $\veps \rightarrow 0$, as $n \rightarrow \infty$. The above asymptotic condition will, in turn, be satisfied by the assumptions that we make in our main theorems.

\end{remark}

Through the paper we will denote by $L^2(\X_n)$ the space of functions of the form $f: \X_n \rightarrow \R$ and endow this space with the inner product
\[ \langle f, g \rangle_{L^2(\X_n)} := \frac{1}{n}\sum_{x \in \X_n} f(x) g(x), \quad f, g \in L^2(\X_n).\]
It is well-known (and straightforward to verify) that $\Delta_n$ is a self-adjoint and positive semi-definite operator with respect to the inner product $\langle\cdot,\cdot\rangle_{L^2(\X_n)}$; see \cite{von2007tutorial}. In particular, we can list $\Delta_n$'s eigenvalues in increasing order as
\begin{align}
    0= \hat{\lambda}^\eps_1\le  \hat{\lambda}^\eps_2 \le \dots \le \hat{\lambda}^\eps_n
\end{align}
and we can find an associated orthonormal basis for $L^2(\X_n)$ consisting of eigenvectors of $\Delta_n$ that we will denote by $\hat{u}^\eps_1,\dots,\hat{u}^\eps_n$. Explicitly,
\[ \Delta_n \hat{u}_l^\veps = \hat{\lambda}_l^\veps \hat{u}_l^\veps,  \]
and we reiterate that $\hat{u}_l^\veps$ is normalized according to
\begin{equation}\label{eq:normalization of hatu}
    \|\hat{u}_l^\eps\|_{L^2(\X_n)}^2=\frac{1}{n}\sum_{x \in \X_n} \hat{u}_l^\eps(x)^2 = 1.
\end{equation}
In the sequel, we also use the notion of discrete $H^1$ semi-norm, which for a given $f\in L^2(\X_n)$ is defined as 
	\begin{equation}\label{e.def-H_1}
		\|f\|^2_{H^1(\X_n)} := \frac1{n^2\veps^{m+2} }\sum_{x\in \X_n}\sum_{y \in \X_n} \eta\left(\frac{|x-y|}{\veps}\right) (f(x) - f(y))^2.
	\end{equation}
        By a direct computation, this semi-norm can be rewritten as
        \begin{equation*}
            \|f\|^2_{H^1(\X_n)} = 2\langle \Delta_n f, f\rangle_{L^2(\X_n)}.
        \end{equation*}

\subsubsection{Weighted Laplace-Beltrami operators}

After stating our assumptions on the data-generating model $(\M, \rho)$ and on the other elements that determine the graph Laplacian $\Delta_n$, we discuss the operator $\Delta_\rho$ and introduce some associated objects and notation that we use in the sequel.

We use $L^2(\M)$ to denote the space of (equivalence classes of) measurable functions $f: \M \rightarrow \R$ that are square integrable. We endow this space with the inner product
	\[ \langle  f , g \rangle_{L^2(\M)}:= \int_{\M} f(\vx) g(\vx) \rho(\vx) \dx, \quad f, g \in L^2(\M).   \]
We use $H^1(\M)$ to denote the Sobolev space of $L^2(\M)$ functions that have distributional derivatives that are also square integrable. In particular, for an element $f \in H^1(\M)$ the quantity
	\[ \int_{\M}  |\nabla f(\vx)|^2 \rho^2(\vx)\   \dx  \]
	is finite. A weighted Dirichlet energy over $L^2(\M)$ can then be defined as:
	\begin{equation}\label{Equ: dirichlet energy}
	D(f):= \begin{cases} \int_{\M} |\nabla f(\vx)|^2 \rho^2(\vx)\dx, \quad \text{ if } f \in H^1(\M), \\  +\infty, \quad \text{ if } f \in L^2(\M) \setminus H^1(\M) . \end{cases}
	\end{equation}

The operator $\Delta_\rho$ defined in \eqref{eq:Laplace-Beltrami Operator} is easily seen to be self-adjoint and positive semi-definite with respect to the inner product $\langle  \cdot , \cdot \rangle_{L^2(\M)}$ (see, e.g., \cite{trillos2018variational}). Moreover, there is a complete orthonormal family $\{ u_l \}_{l \in \N}$ (w.r.t. $L^2(\M)$ inner product) of eigenfunctions of $\Delta_\rho$ with corresponding eigenvalues 
	\[ 0 = \lambda_1 < \lambda_2 \leq \dots\leq \lambda_l \leq \dots  \rightarrow \infty \]
	that may be repeated according to multiplicity. Explicitly, the eigenpair $(\lambda_l, u_l)$ satisfies the equation:
	\begin{equation}
	\label{eq:eigenform}
	\Delta_{\rho} u_l  =\lambda_l u_l.
	\end{equation}
	From basic regularity theory of elliptic partial differential equations (see, e.g., \cite{evans,FernndezReal2022}) it follows that equation \eqref{eq:eigenform} holds at every point $x \in \M$. Furthermore, following the discussion in \cref{sec:Regularity}, we see that, under the assumption that $\M \in \MM_S$ and $\rho \in \mathcal{P}_\M^\alpha$, the function $u_l$ is in $C^{3,\alpha}(\M)$.

    In terms of the Dirichlet energy defined in \eqref{Equ: dirichlet energy}, the eigenvalues of $\Delta_\rho$ can be written as 
\begin{equation}
	\label{eq:minmax principle for laplacian}
	\lambda_{l}=\min _{S \in \mathfrak{S}_{l}} \max _{f \in S \backslash\{0\}} \frac{D(f)}{\|f\|_{L^{2}(\M)}^{2}}=\min _{S \in \mathfrak{S}_{l}} \max _{\|f\|^2_{L^{2}(\M)}=1} D(f),
\end{equation}
where $\mathfrak{S}_{l}$ denotes the set of all linear subspaces of $L^2(\M)$ of dimension $l$. We reiterate that $u_l$ is normalized according to
 \begin{align}\label{eq:norm u_l=1}
     \int_\M u_l^2(x) \rho(x) \dx=1.
 \end{align}
The above implies
\[ \int_\M |\nabla u_l(x)|^2 \rho^2(x) \dx = \int_{\M} \Delta_\rho u_l (x) u_l(x) \rho(x) \dx = \lambda_l \int_\M u_l^2(x) \rho(x) \dx  = \lambda_l, \]
where the first equality follows from the definition of $\Delta_\rho$ and integration by parts.

\subsection{Main Results}
\label{sec:MainResults}

In order to state our main results, we need to impose one additional assumption on the data generating model $(\M, \rho)$. Indeed, to prove our CLT for $\hat{\lambda}_l^\veps$ we need to assume that the $l$-th eigenvalue of $\Delta_\rho$ is simple. This eigengap assumption will be discussed further in our numerical simulations in section \ref{sec:simulation}.

\begin{assumption}[Eigengap assumption I]\label{assum:eigengap for l}
    Given $l \in \N$, we assume 
    \begin{equation*}
        \min_{k:k\not= l} |\lambda_k-\lambda_l|\ge \gamma_l,
    \end{equation*}
    for some $\gamma_l>0$. In particular, the $l$-th eigenvalue of $\Delta_\rho$ is assumed to be simple.
\end{assumption}

We are ready to state our first main theorem.

\begin{theorem}\label{thm1}
    Under Assumptions \ref{assump:DataGenerating}, \ref{assump:Kernel}, and \ref{assum:eigengap for l}, if $\left(\frac{(\log n)^2}{n}\right)^{\frac{1}{m+4}}\ll\eps\ll 1$, then
    \begin{align*}
        \sqrt{n}\left(\hat{\lambda}_l^\eps-\E\left[\hat{\lambda}_l^\eps\right]\right)\dto \mathcal{N}\left(0,\sigma_l^2\right),
    \end{align*}
    as $n \rightarrow \infty$, where the variance $\sigma_l^2$ is defined as 
    \begin{equation}\label{eq-def:sigma^2}
    		\sigma_l^2:=\int_\M\Bigl(\lambda_l u_l(x)^2+\lambda_l  - 2|\nabla u_l(x) |^2\rho(x)\Bigr)^2\rho(x)\dx,
    \end{equation}
    for $(\lambda_l, u_l)$ the $l$-th eigenpair of the operator $\Delta_\rho$. We recall that $u_l$ was assumed to be normalized; see \eqref{eq:norm u_l=1}.
\end{theorem}
\cref{thm1} guarantees that, asymptotically, the $l$-th eigenvalue of the graph Laplacian $\Delta_n$ follows a Gaussian distribution with an asymptotic variance that depends on the eigenpair $(\lambda_l, u_l)$ of $\Delta_\rho$. We emphasize that our result requires \cref{assum:eigengap for l} to hold.  


\begin{remark}
When $l=1$, the result in \cref{thm1} is trivial since $u_1$ is a constant function and both $\lambda_1$ and $\hat{\lambda}_1^\veps$ are equal to zero.
\end{remark}

\begin{remark}
We remark that the right-hand side of \eqref{eq-def:sigma^2} is invariant under a sign change of $u_l$. In particular, the specific choice of normalized eigenvector $u_l$ in \eqref{eq-def:sigma^2} is unimportant under \cref{assum:eigengap for l} and because of this there is no consequential ambiguity when we say that $(\lambda_l, u_l)$ is \textbf{the} $l$-th eigenpair of $\Delta_\rho$. However, this wouldn't be the case if the eigenspace corresponding to $\lambda_l$ had dimension greater than one. Additional work would be needed to prove a CLT in this more general setting, and we leave this as an interesting open question that is worth exploring in the future. Note, however, that for a generic model $(\M, \rho)$ the eigenvalues of $\Delta_\rho$ are expected to be simple.
\end{remark}

\begin{remark}
Note that the limiting distribution in \eqref{thm1} does not depend on the specific kernel $\eta$ used to build the graph (as long as it satisfies \cref{assump:Kernel}).
\end{remark}

\begin{remark}
\label{rem:Minimaxity}
In the work \cite{trillos2024}, it was proved that, by choosing $\veps$ appropriately, graph Laplacian eigenpairs approximate the eigenpairs of the operator $\Delta_\rho$ at the rate $O\left( n^{-\frac{2}{m+4}}\right)$ (up to logarithmic terms)
when the error of approximation of eigenfunctions is measured in a $H^1(\M)$ sense and the data generating model is such that $\rho \in \mathcal{P}_\M^\alpha$ for some $\M \in \MM_S$ and $\alpha \in (0,1)$. Their results actually allow them to argue that, when $\veps$ is chosen according to $\veps \sim  n^{- \frac{1}{m+4}}$, graph Laplacian eigenpairs are essentially minimax optimal for the problem of estimating eigenpairs of $\Delta_\rho$ when $(\M, \rho)$ is in a certain family of sufficiently regular data generating models. Note that the choice $\veps \sim  n^{- \frac{1}{m+4}}$ is essentially the same (ignoring logarithmic terms) as the optimal choice for the classical density estimation problem when the error of approximation is measured in $L^2(\M)$; see \cite{Cencov:62SM,cencov2000statistical,khas1979lower,bretagnolle1979estimation}.  
\end{remark}

\begin{remark}\label{remark:possibly improved rate}
    We hypothesize that we can relax the condition for $\eps$ in \cref{thm1} to $(\frac{(\log n)^2}{n})^{\frac{1}{m+2}}\ll\eps\ll 1$, which matches the requirement on $\veps$ for our Davis-Kahan type bounds in Propositions \ref{lem:EigenVectorStability} and \ref{lem:EigenvalueStability} to hold. Relaxing the requirement for $\veps$ in this way would allow us to conclude that our CLTs hold for the choice of $\veps$ discussed in Remark \ref{rem:Minimaxity}; note that the choice $\veps \sim n^{-\frac{1}{m+4}}$ is a borderline case in our current result.

   In order to relax the condition on $\veps$ in our main theorems as described above, we would first need to improve the  analysis in \cite{calder2022lipschitz}. Indeed, in that paper, whose results we use in our proofs, it is required that $\eps\gg (\frac{(\log n)^2}{n})^{\frac{1}{m+4}}$ in order to deduce the type of almost $C^{0,1}$ and $L^\infty$ approximation error rates presented in \cref{lem:spectral convergence in literature} below. While we believe that some of the techniques used in \cite{trillos2024} would allow us to extend the results in \cite{calder2022lipschitz}, this extension certainly goes beyond the scope of this work and is thus left as an interesting open problem that is worth exploring in the future.
\end{remark}

Next, we generalize \cref{thm1} and present a CLT for multiple eigenvalues. We require the following stronger eigengap assumption.

\begin{assumption}[Eigengap assumption II]\label{assum:multiple CLT}
Let $l \in \N$. We assume that
    \begin{equation*}
        \min_{1\le j<k\leq l+1} |\lambda_j-\lambda_k|>\gamma_l
    \end{equation*}
    for some $\gamma_l>0$. In particular, all of the smallest $l$ eigenvalues of $\Delta_\rho$ are assumed to be simple.
\end{assumption}

\begin{theorem}\label{thm:multi normal distribution}
Suppose that Assumptions \ref{assump:DataGenerating}, \ref{assump:Kernel}, and \ref{assum:multiple CLT} hold. Let $\Sigma \in\R^{l\times l}$ be the matrix whose entries are given by
\begin{equation*}
    \Sigma_{jk} := \int_\M\Bigl(\lambda_j u_j(x)^2+\lambda_j  - 2|\nabla u_j(x) |^2\rho(x)\Bigr)\Bigl(\lambda_k u_k(x)^2+\lambda_k  - 2|\nabla u_k(x) |^2\rho(x)\Bigr)\rho(x) \dx.
\end{equation*}
Provided $\left(\frac{(\log n)^2}{n}\right)^{\frac{1}{m+4}}\ll\eps\ll 1$, we have
     \begin{equation*}
        \sqrt{n}(\hat{\lambda}^\eps_1-\E[\hat{\lambda}^\eps_1],\dots,\hat{\lambda}^\eps_l-\E[\hat{\lambda}^\eps_l])\dto\mathcal{N}(0,\Sigma),
    \end{equation*}
    as $n \rightarrow \infty$.
\end{theorem}


\subsubsection{Consistent estimators for the asymptotic variance}\label{sec:extension}
Since the asymptotic variance $\sigma_l^2$ in \cref{thm1} and the covariance matrix $\Sigma$ in \cref{thm:multi normal distribution} depend on the unknown density $\rho$ (as well as on $\M$), 
it is natural to attempt estimating them. In this section, we discuss some concrete estimators that we can prove are consistent. 

We first define a data driven surrogate for the term $|\nabla u_l|^2 \rho$ appearing in $\sigma_l^2$. Namely, we consider 
\begin{equation}
   \widehat{G}_{l,n}(x_i):= \frac{1}{2n\veps^{m+2}} \sum_{j=1}^n | \hat{u}_l (x_j) - \hat{u}_l(x_i)|^2 \eta\left(\frac{|x_i- x_j|}{\veps} \right), \quad i=1, \dots, n, 
   \label{eqn:SurrogateNormGradient}
\end{equation}
which, as we prove in Lemma \ref{lem:EstimatorNormGradient}, satisfies
\[ \max_{i=1, \dots, n} \left| \widehat{G}_{l,n}(x_i) - |\nabla u_l (x_i)|^2 \rho(x_i) \right| \asto 0,  \]
provided $\veps$ scales appropriately. The above motivates introducing the following estimators for $\sigma_l^2$ and $\Sigma \in \R^{l \times l}$, respectively:
    \begin{equation}\label{eq-def:hat variance}
       \widehat{\sigma}_l^2 := \frac{1}{n}\sum_{i=1}^n\Bigl(\hat{\lambda}^\eps_l\hat{u}_l(x_i)^2+\hat{\lambda}^{\eps}_l - 2 \widehat{G}_{l,n}(x_i)\Bigr)^2,
    \end{equation}
\begin{equation}\label{eq-def:hat Covariance}
       \widehat{\Sigma}_{jk} := \frac{1}{n}\sum_{i=1}^n\Bigl(\hat{\lambda}^\eps_j\hat{u}_j(x_i)^2+\hat{\lambda}^{\eps}_j - 2 \widehat{G}_{j, n}(x_i)\Bigr) \Bigl(\hat{\lambda}^\eps_k\hat{u}_k(x_i)^2+\hat{\lambda}^{\eps}_k - 2 \widehat{G}_{k, n}(x_i) \Bigr).
    \end{equation}


Under suitable assumptions on $\veps$,  $\widehat{\sigma}_l^2$ can be seen to be a consistent estimator for $\sigma_l^2$ and $\widehat{\Sigma}$ a consistent estimator for $\Sigma$. This is the content of the next theorem. 

\begin{theorem}
\label{thm:ConsistentEstimator}
  Suppose that $\left(\frac{(\log n)^2}{n}\right)^{\frac{1}{m+4}}\ll\eps\ll 1$ and that Assumptions \ref{assump:DataGenerating}, \ref{assump:Kernel}, and \ref{assum:eigengap for l} hold. Then
    \begin{equation*}
        \frac{ \widehat{\sigma}_l^2}{\sigma_l^2}\asto 1,
    \end{equation*}
    as $n \rightarrow \infty$.
    
 If, in addition, Assumption \ref{assum:multiple CLT} holds, then
 \[ \widehat{\Sigma} \asto  \Sigma,\]
    as $n\to\infty.$
\end{theorem}

A direct consequence of Slutsky's theorem and Theorems \ref{thm1}, \ref{thm:multi normal distribution}, and \ref{thm:ConsistentEstimator} is the following corollary.

\begin{corollary}
\label{cor1}
Suppose that $\left(\frac{(\log n)^2}{n}\right)^{\frac{1}{m+4}}\ll\eps\ll 1$. Under Assumptions \ref{assump:DataGenerating}, \ref{assump:Kernel}, and \ref{assum:eigengap for l}, we have
    \begin{align*}
  \frac{\sqrt{n}\left(\hat{\lambda}_l^\eps-\E\left[\hat{\lambda}_l^\eps\right]\right)}{\widehat{\sigma}_l}\dto \mathcal{N}\left(0,1\right),
    \end{align*}
    as $n \rightarrow \infty$, where the variance estimator $  \widehat{\sigma}_l^2$ was defined in \eqref{eq-def:hat variance}. 
    
    On the other hand, under Assumptions \ref{assump:DataGenerating}, \ref{assump:Kernel}, and \ref{assum:multiple CLT}, we have
\begin{equation*}
        \sqrt{n} (\widehat{\Sigma}^{\dagger})^{1/2} (\hat{\lambda}^\eps_1-\E[\hat{\lambda}^\eps_1],\dots,\hat{\lambda}^\eps_l-\E[\hat{\lambda}^\eps_l])\dto\mathcal{N}(0,\Sigma^\dagger \Sigma),
    \end{equation*}
    as $n \rightarrow \infty$, where the covariance estimator $  \widehat{\Sigma}$ was defined in \eqref{eq-def:hat Covariance} and $\widehat{\Sigma}^\dagger$ and $\Sigma^\dagger$ are, respectively, the pseudo-inverses of $\widehat{\Sigma}$ and $\Sigma$. 
    
    
\end{corollary}


\subsubsection{On centralization of $\hat{\lambda}_l^\veps$ around $\Delta_\rho$'s eigenvalues}
\label{ss:bias}

A natural question arising from our main theorems is whether it is possible to replace $\E[\hat{\lambda}_l^\veps]$ with $\lambda_l$ (i.e., the $l$-th eigenvalue of the operator $\Delta_\rho$) in the CLTs from \cref{sec:MainResults}. If this were possible, we would be able to obtain asymptotic confidence intervals for the eigenvalues of the operator $\Delta_\rho$. This naturally motivates analyzing the bias term $ \E[\hat{\lambda}_l^\veps] - \lambda_l$ and identifying conditions under which 
\begin{equation}
   \sqrt{n}\left| \E[\hat{\lambda}_l^\veps] -\lambda_l  \right| \rightarrow 0, 
   \label{eqn:VanishingBias}
\end{equation}
as $n \rightarrow \infty$. We believe that, for a \textit{generic} model $(\M, \rho)$, the above can only hold when the dimension restriction $m \leq 3$ is satisfied. In section \ref{ss:dim}, we use numerical experiments to explore this dimension restriction and in what follows we discuss some of the theoretical reasons why we believe this condition is necessary.  

First, using the fact that $\Delta_n$ is positive semi-definite, it is straightforward to verify that
\begin{equation}
(\hat{\lambda}_l^\veps - \lambda_l) \langle \hat{u}_l, u_l  \rangle_{L^2(\X_n)} =  \langle \Delta_n u_l - \Delta_\rho u_l , u_l  \rangle_{L^2(\X_n)}   +  \langle \Delta_n u_l - \Delta_\rho u_l , \hat{u}_l - u_l  \rangle_{L^2(\X_n)}.  
    \label{eqn:BiasDecomp}
\end{equation}
The expectation of the first term on the right hand side of the above expression satisfies the following estimate in terms of $\veps$.
\begin{theorem}\label{lem:spectral convergence bias control}
Under Assumptions \ref{assump:DataGenerating} and \ref{assump:Kernel}, if $ \left(\frac{\log n}{n}\right)^{\sfrac{1}{m}} \ll \veps \ll 1$, we have
\begin{equation}
   \left| \E[  \langle  \Delta_n u_l - \Delta_\rho u_l, u_l \rangle_{L^2(\X_n)} ] \right| \leq C_l \veps^2. 
   \label{eqn:UpperBias1}
\end{equation}
\end{theorem}

For $\veps$ in the more restrictive regime $\left(\frac{(\log n)^2}{n}\right)^{\frac{1}{m+4}}\ll\eps\ll 1$, we can prove that the other terms in \eqref{eqn:BiasDecomp} are negligible and thus the bias $\lambda_l  - \E[\hat{\lambda}_l^\veps] $ is determined by $\E[  \langle  \Delta_n u_l - \Delta_\rho u_l, u_l \rangle_{L^2(\X_n)} ]$. Therefore, Theorem \ref{lem:spectral convergence bias control} implies that the bias is quadratic in $\veps$. In precise terms, we have the following. 
\begin{corollary}
Under Assumptions \ref{assump:DataGenerating}, \ref{assump:Kernel}, and \eqref{assum:eigengap for l}, if $\left(\frac{(\log n)^2}{n}\right)^{\frac{1}{m+4}}\ll\eps\ll 1$, we have
\begin{equation}
   \left| \lambda_l - \E[ \hat{\lambda}_l^\veps] \right| \leq C_l \veps^2.
   \label{eqn:UpperBias2}
\end{equation}
\label{cor:BiasEstimate}
\end{corollary}

Now, as discussed in Remark \ref{rem:LowerBoundBias}, for a \textit{generic} model $(\M, \rho)$, we should also have the lower bound
\[ \left| \E[  \langle  \Delta_n u_l - \Delta_\rho u_l, u_l \rangle_{L^2(\X_n)} ] \right| \geq c_l \veps^2.  \]
If we consider $\veps$ smaller than $\left(\frac{(\log n)^2}{n}\right)^{\frac{1}{m+4}}$ (i.e., outside the regime in Corollary \ref{cor:BiasEstimate}), then it is reasonable to expect that either the other terms in \eqref{eqn:BiasDecomp} (we mean, other than the first term on the right hand side of \eqref{eqn:BiasDecomp}) contribute to a bias that is larger than $\veps^2$ (something that should intuitively happen as we approach the connectivity threshold $(\log(n)/n)^{1/m}$) or these other terms are still negligible and we thus have that the size of the bias continues to be determined by the term on the left hand side of \eqref{eqn:UpperBias2} (which is not smaller than $c_l\veps^2$). In either case, the above discussion suggests that the bias should not be smaller than $O(\veps^2)$. In particular, it would seem as if $ \sqrt{n }\veps^2 \rightarrow 0$ was a necessary condition for  \eqref{eqn:VanishingBias} to hold. This condition, in turn, would force $\veps$ to satisfy 
\begin{equation}
 \left(\frac{\log(n)}{n}\right)^{1/m} \ll \veps \ll  \frac{1}{n^{1/4}},  
    \label{eqn:EpsilonCentralization}
\end{equation}
which implicitly forces the manifold's dimension to satisfy $m \leq 3$.

As a second argument supporting a dimension restriction like the one discussed above, we note that this type of condition shows up in other related settings. For example, in the stochastic homogenization setting with a fixed Laplacian operator and a scaled random drift term discussed in \cite{gu2016scaling} there is no dimension restriction in the CLTs for the eigenvalues of the random operator when these are centralized around their expectations \cite{gu2016scaling}, while the condition $m\leq 3$ is required to centralize around the limiting operator's eigenvalues (see \cite{bal2008central}).

\begin{remark}
  While the above discussion suggests that centralization around $\lambda_l$ may only be possible when $m \leq 3$ (and $\veps$ satisfies \eqref{eqn:EpsilonCentralization}), we want to highlight that the admissible range of $\veps$ in our results from section \ref{sec:MainResults} is not compatible with the condition \eqref{eqn:EpsilonCentralization}. For this reason, at this stage we cannot provide a CLT for graph Laplacian eigenvalues centering around $\Delta_\rho$'s eigenvalues wvene for dimension $m=1$.
  
  Nevertheless, as already suggested in Remark \ref{remark:possibly improved rate}, we expect to be able to expand the range of $\veps$ for which our CLTs apply. Extending the range of $\veps$ in the way suggested in that remark would allow us to establish a CLT centered around $\lambda_l$ when $m =1$. Determining whether it is possible to extend the range of $\veps$ even further, and perhaps cover the full range \eqref{eqn:EpsilonCentralization}, is an open problem that we leave for future research.  
  


\end{remark}

\subsection{Statistical Interpretation of the Asymptotic Variance}
\label{sec:StateInterpretation}

In this section, we state the precise Cramer-Rao lower bound that we announced earlier in this introduction. We start by making the notion of unbiased estimator in this setting precise.

\begin{definition}
\label{def:UnbiasedEstimator}
Let $\M \in \MM_S$ and $\alpha \in (0,1)$. We say that $\tilde{\lambda}_l: \M^n \rightarrow \R$ is an unbiased estimator for $\lambda_l$ if for every $\rho \in\mathcal{P}_\M^\alpha$ we have
\[  \E _{\rho} [\tilde{\lambda}_l(x_1, \dots, x_n)] = \lambda_l(\rho), \]
where $\lambda_l(\rho)$ is the $l$-th eigenvalue of $\Delta_\rho$. Here, by $\E_\rho[ \cdot]$ we mean the expectation when we assume $x_1, \dots,x_n \sim \rho$.
\end{definition}

\begin{theorem}
For any unbiased estimator $\tilde{\lambda}_l$ for $\lambda_l$ (according to Definition \ref{def:UnbiasedEstimator}) we must have
\[ \Var_\rho(\tilde \lambda_l) \geq \frac{\sigma_l^2}{n}, \]
for all $\rho \in \mathcal{P}_\M^\alpha$ for which $\lambda_l(\rho)$ is simple (i.e., for which Assumption \ref{assum:eigengap for l} holds). In the above, $\sigma_l^2= \sigma_l^2(\rho)$ is as in \eqref{eq-def:sigma^2} and $\Var_\rho$ denotes variance when we assume $x_1, \dots, x_n \sim \rho$.
\label{thm:CramerRao}
\end{theorem}

It is not the goal of this paper to explicitly exhibit an example of an unbiased estimator for $\lambda_l$ or claim existence of one such estimator. We emphasize that, instead, our motivation for presenting Theorem \ref{thm:CramerRao} and Definition \ref{def:UnbiasedEstimator} is to provide a concrete statistical interpretation for the asymptotic variance in \cref{thm1}. Namely, Theorem \ref{thm:CramerRao} states that \textit{no unbiased estimator for $\lambda_l$ can have smaller variance than the asymptotic variance of $\hat{\lambda}_l^\veps$}. Since $\hat{\lambda}_l^\veps$ is a consistent estimator for $\lambda_l$, this \textit{suggests} a form of statistical efficiency for $\hat{\lambda}_l^\veps$ that we find quite interesting.

A rigurous proof of \cref{thm:CramerRao} will be presented at the end of section \ref{sec:simulation-geometry}. However, we find it to be illuminating to start that section with some informal computations that allow us to provide another, deep geometric interpretation for $\sigma_l^2$ and which, in turn, formally suggests the Cramer-Rao type lower bound from Theorem \ref{thm:CramerRao}. More concretely, we'll see that $\sigma_l^2$ can be formally interpreted as the dissipation along the gradient flow of a suitable energy with respect to the Fisher-Rao geometry. For the moment, and to avoid introducing new mathematical objects at this stage, we skip the details of this interpretation and postpone our discussion to section \ref{sec:simulation-geometry}, where we first present an informal background on the Fisher-Rao geometry and on gradient flows.

\subsection{Literature Review}\label{sec:literature review}

The study of graph Laplacians over data clouds and their convergence toward Laplace-Beltrami operators on manifolds in the large data limit has been a topic explored in multiple works over the past two decades, at least since the work \cite{belkin2005towards}. A comprehensive discussion of the large literature on this topic can be found in the recent work \cite{trillos2024}, and for this reason here we will limit ourselves to presenting a brief summary of the existing literature and explicitly enunciating some key estimates that we use in the proofs of our main results.

Most of the earlier works on consistency of graph Laplacians studied their \textit{pointwise consistency} under different assumptions on the data generating model. Some of these works include: \cite{singer2006graph,hein2005graphs,hein2007graph,belkin2005towards,ting2010analysis,GK}. Results on \textit{spectral convergence} of graph Laplacians toward Laplace-Beltrami operators on manifolds, which relate the eigenvalues and eigenvectors of graph Laplacians and Laplace-Beltrami operators, include \cite{Shi2015,BIK,trillos2019error,Lu2019GraphAT,calder2019improved,DunsonWuWu,WormellReich,cheng2022eigen,trillos2024,li2025consistency,trillos2023large}. These results provide, under different settings and assumptions, high probability error estimates for the approximation of eigenpairs of differential operators on manifolds (or multiple manifolds) with graph Laplacian eigenpairs. In statistical terms, these results imply \textit{laws of large numbers} for graph Laplacian eigenvalues and eigenvectors under suitable assumptions on the connectivity parameter $\veps$.




In this paper, we will repeatedly use the following estimates that are taken from \cite{trillos2024} and \cite[Theorem 2.6]{calder2022lipschitz}. These are bounds on the difference (under different norms) between the eigenpairs of $\Delta_n$ and $\Delta_\rho$ (i.e., spectral convergence) as well as on the size of the action of $\Delta_n - \Delta_\rho$ on the fixed regular function $u_l$ (i.e., pointwise convergence).

\begin{lemma}\label{lem:spectral convergence in literature}
Suppose that the data-generating model $(\M, \rho)$ satisfies Assumptions \ref{assump:DataGenerating} and \ref{assum:eigengap for l}, that $\eta$ satisfies Assumption \ref{assump:Kernel}, and that the connectivity parameter $\eps$ satisfies Assumption \ref{assum:Epsilon}. Then 
\begin{enumerate}
    \item  For any $t> 0$, with probability at least $1-Cn\eps^{-m}\exp\left(-cn\eps^mt^2\right)$ we have
    \begin{align}
    |\hat{\lambda}^\eps_l- \lambda_l|\le C_l(\eps^2+t),
    \label{eqn:ConRateEugenvalues}
    \end{align}
    where $C>1,0<c<1$ depend on $\rho$ and on $\M$, and $C_l>0$ is a constant that only depends on $l$, $\rho$, and $\M$. 
    \item Provided $C\eps^2 \log(\eps^{-1})\leq t\leq 1$, with probability at least $1-Cn\eps^{-m}\exp(-c n\eps^{m}t^2)$ we have
    \begin{equation}
        \begin{split}
            \|\hat{u}_l-u_l\|_{L^2(\X_n)}\leq C_lt,\\
            \|\hat{u}_l-u_l\|_{H^1(\X_n)}\leq C_lt,
        \end{split}
        \label{eqn:ConRtaeEugenVectorsL2H1}
    \end{equation}
    where the discrete $H^1$ semi-norm was introduced in \eqref{e.def-H_1}. Furthermore, with probability at least $1-C(\eps^{-6m}+n)\exp(-cn\eps^{m+4})$,
    \begin{equation}\label{eq:L_infty control}
        \lVert \hat{u}_l - u_l \rVert_{L^\infty(\X_n)} \leq C_l\eps,
    \end{equation}
    and
      \begin{equation}\label{eq:almost Lipschitz}
        \max_{x,y\in\X_n} \frac{\left|\hat{u}_l(x)-u_l(x)-(\hat{u}_l(y)-u_l(y))\right|}{|x-y|+\eps}\leq C_l\eps.
    \end{equation}
    \item With probability at least $1-C \eps^{-m}\exp(-c n\eps^{m+1}t^2)$, 
    \begin{equation}
    \label{eq:PointwsieConsistency}
        \begin{split}
            \|(\Delta_\rho-\Delta_n)u_l\|_{L^\infty(\X_n)}\leq C_l\left(\eps+\frac{t}{\sqrt{\eps}}\right).
        \end{split}
    \end{equation}
\end{enumerate}
In the above, $(\hat{\lambda}_l^\veps, \hat{u}_l^\veps)$ and $(\lambda_l, u_l)$ are the $l$-th eigenpairs of $\Delta_n$ and $\Delta_\rho$, respectively. 
    
\end{lemma}

\begin{remark}
Since there is always a sign ambiguity when dealing with eigenvectors, the statement of (2) in \cref{lem:spectral convergence in literature} has to be interpreted as follows: for a given normalized eigenvector $\hat{u}_l$ of $\Delta_n$ there is a choice for $u_l$ (a normalized eigenfunction of $\Delta_\rho$) for which the error estimates in (2) hold. In the sequel, we will implicitly assume that $\hat{u}_l$ and $u_l$ are already aligned. 
\end{remark}

\begin{remark}
Note that, by choosing $\eps$ appropriately, the above estimates imply a high probability convergence rate in \eqref{eq:PointwsieConsistency} of $O\bigl((\frac{n}{\log n})^{-\frac{1}{m+4}}\bigr)$, a $O\bigl((\frac{n}{\log n})^{-\frac{2}{m+4}}\bigr)$ convergence rate for eigenvalues \eqref{eqn:ConRateEugenvalues}, a $O\bigl((\frac{n}{\log n})^{-\frac{2}{m+4}}\bigr)$ convergence rate for eigenvectors in $L^2$ and discrete $H^1$ norms \eqref{eqn:ConRtaeEugenVectorsL2H1}, and a $O\bigl((\frac{n}{\log n})^{-\frac{1}{m+4}}\bigr)$ convergence rate for eigenvectors in $L^\infty$ \eqref{eq:L_infty control} and almost-Lipschitz \eqref{eq:almost Lipschitz} norms. 
\end{remark}

\subsubsection{Central limit theorems for the eigenpairs of other related random operators}

As discussed earlier, while several works have obtained high probability error estimates for the approximation of eigenpairs of Laplace-Beltrami operators with graph Laplacian eigenpairs, no result in the literature has discussed the corresponding CLTs of this approximation. There are, however, other random settings where CLTs have been obtained for the eigenpairs of related random operators. Here we present a brief summary of these works.

In a somewhat related setting to ours, in the context of homogenization in random media, CLTs for the eigenvalues of an operator that is the sum of a deterministic graph Laplacian over a (deterministic) grid graph and a scaled random drift term have been established. In particular, it is shown in  \cite{biskup2016eigenvalue} that, when centralized around their expectations, the eigenvalues of this type of discrete random operator are shown to satisfy a CLT without any dimension restriction; note that, for the setting considered in that paper, the $\veps^{m/2}$ scaling is analogous to our $n^{-1/2}$ scaling. For $m=1,2,3$, \cite{bal2008central} proved a CLT for the eigenvalues of related operators with centralization over the eigenvalue of a limiting homogenized continuum differential operator. The dimension restriction in that paper is in line with (and helps motivate) our discussion about centralization around $\lambda_l$ that we presented in subsection \ref{sec:bias}. As observed by \cite{gu2016scaling}, the limits of the local and global fluctuations of eigenfunctions are not the same for the random elliptic operators. In their setting, as well as in ours, the bias between the global fluctuations and local fluctuations is large compared with $\frac{1}{\sqrt{n}}$ when $m$ is large.

In random matrix theory, CLTs for eigenpairs of Laplacians associated to different ensembles of random graphs have been extensively investigated. For example, in a general random dot product graph setting, where the graph's adjacency matrix is constructed by first generating latent positions for the nodes and then taking an inner product of those latent positions to determine the likelihood of connectivity between the nodes, \cite{tang2016limit,athreya2018statistical,tang2025eigenvector} established  CLTs and fine estimates for eigenvalues and eigenvectors of the graph Laplacian. Related to this line of works, the paper \cite{tang2022asymptotically} obtained CLTs for different estimators of the block probability matrix in a stochastic block model setting and \cite{athreya2022eigenvalues} established CLTs for eigenvalues of the graph Laplacian for the same model. It is also worth mentioning the work \cite{fan2022asymptotic}, which derived CLTs for eigenvectors of certain random matrices with spikes. In general, the investigation of the fluctuations of eigenpairs of graph Laplacians under different random graph models is an active area of research in random matrix theory. We highlight that the setting studied in our paper is quite different from the ones of the previously mentioned papers. Indeed, in our setting the graph is completely determined by the position of the nodes (the data cloud) and hence there is no independence (or conditional independence) between the edges of our graph.



\subsection{Outline}

The rest of the paper is organized as follows. In section \ref{sec:Analysis}, we present the proofs of Theorems \ref{thm1} and \ref{thm:multi normal distribution}. We start by rewriting the difference $\hat{\lambda}_l^\veps - \E[\hat{\lambda}_l^\veps]$ conveniently as a sum of a ``bulk" term (closely related to a U-statistic) and an ``error term". The analysis of the bulk term is presented in section \ref{sec:random term}, while section \ref{secLlem-b_n control} contains the detailed analysis of the error term, which is the more technical part of the paper; a key element in the analysis of the error term is our leave-one-out stability estimates for graph Laplacians. In section \ref{sec:proof for multi eigenvalue CLT}, we discuss the CLT for multiple eigenvalues; we highlight that the most technical steps in the proof of this extension actually follow directly from the analysis for the single eigenvalue case. In section \ref{sec:ConsistentEstimatorVariance}, we present the details of the proof of \cref{thm:ConsistentEstimator}.
In section \ref{sec:bias}, we find some bounds on the bias of $\hat{\lambda}_l^\veps$ and in particular prove \cref{lem:spectral convergence bias control} and Corollary \ref{cor:BiasEstimate}.  

Section \ref{sec:simulation-geometry} is dedicated to the geometric and statistical interpretations of the asymptotic variance $\sigma_l^2$. We start in section \ref{sec:FisherRao} by presenting some informal background on the Fisher-Rao geometry and on general gradient flows with respect to this geometry. We continue with section \ref{sec:GeometricInterp}, where we informally characterize $\sigma_l^2= \sigma_l^2(\rho)$ as the dissipation along the gradient flow of a suitable energy with respect to the Fisher-Rao geometry. In section \ref{sec:CramerRaoAnalysis}, we build upon the interpretation provided in section \ref{sec:GeometricInterp} and present an informal derivation of the Cramer-Rao lower bound in \cref{thm:CramerRao}, and then proceed to provide the rigorous proof of \cref{thm:CramerRao}. 

In \cref{sec:simulation}, we validate our theoretical results and extend the investigation beyond the eigengap condition and the theoretical admissible range of $\varepsilon$ through simulations, which indicate the potential for relaxing these assumptions. In particular,  \cref{ss:gaus} empirically examines the asymptotic normality of $\sqrt{n}\hat{\lambda}^\varepsilon_2$, while \cref{ss:var} focuses on the asymptotic variance and its empirical estimation. In \cref{ss:dim}, we explore the centralization of eigenvalues and provide empirical evidence for the dimension restriction discussed earlier in this introduction. Finally, \cref{ss:cov} visualizes the asymptotic dependency structure among multiple eigenvalues and suggests an interesting connection to the underlying manifold geometry.

We wrap up the paper in section \ref{sec:Conclusions} with some conclusions.


\medskip

\textbf{Additional Notation.} Throughout the paper, we will use $C>1,0<c<1$ to denote constants that may depend on $\rho$ and on geometric quantities associated to the manifold $\M$; in our analysis, constants may change from one line to the next. 

For any two points $x,y\in\M$, we use $|x-y|$ and $d(x,y)$ to denote their Euclidean and geodesic distances ($\M$'s geodesic distance), respectively.

Given two quantities $a_n$ and $b_n$ that depend on $n$, we write $a_n\ll b_n$ (or $b_n \gg a_n$) to say that, as $n\to\infty$, we have $\frac{a_n}{b_n}\to 0$.

For a sequence of random variables $\{ X_n \}_{n \in \N}$, we write $X_n \dto X$ to say that the sequence converges to the random variable $X$ in distribution, whereas we write $X_n \pto X$ and $X_n \asto X$ to represent convergence in probability and almost surely, respectively. \nc

\section{Analysis}
\label{sec:Analysis}
We begin our analysis by rewriting the difference $\hat{\lambda}_l^\eps - \E[\hat{\lambda}_l^\eps]$ in a convenient way. Precisely, using the definition of $\hat{\lambda}^\eps_l$ and the fact that $\Delta_n$ is self-adjoint with respect to the inner product $\langle \cdot, \cdot \rangle_{L^2(\X_n)}$ we deduce
    \begin{equation}\label{eq:lambda-hatlambda}
        \begin{split}
            \hat{\lambda}^\eps_l - \E[\hat{\lambda}^\eps_l]
            =\frac{\left(\hat{\lambda}^\eps_l - \E[\hat{\lambda}^\eps_l]\right)\langle \hat{u}_l, u_l \rangle_{L^2(\X_n)}}{\langle \hat{u}_l, u_l \rangle_{L^2(\X_n)}}
            =&\frac{\langle \Delta_n \hat{u}_l,  u_l  \rangle_{L^2(\X_n)}- \E[\hat{\lambda}^\eps_l] \langle \hat{u}_l,u_l  \rangle_{L^2(\X_n)}}{\langle \hat{u}_l, u_l \rangle_{L^2(\X_n)}}\\
            =&\frac{\langle \hat{u}_l, \Delta_n u_l- \E[\hat{\lambda}^\eps_l] u_l \rangle_{L^2(\X_n)}}{\langle \hat{u}_l, u_l \rangle_{L^2(\X_n)}}.
        \end{split}
    \end{equation}
From the above expression, we obtain
    \begin{equation}\label{eq:decomposition}
            \sqrt{n}\left((\hat{\lambda}_l^\eps - \E[\hat{\lambda}^\eps_l])\langle \hat{u}_l, u_l \rangle_{L^2(\X_n)}-\E\left[(\hat{\lambda}_l^\eps - \E[\hat{\lambda}^\eps_l])\langle \hat{u}_l, u_l \rangle_{L^2(\X_n)}\right]\right)
            =\BB_n + \EE_n,
    \end{equation}
    where 
  \begin{equation}
\BB_n:= \sqrt{n}\left(\langle u_l, \Delta_n u_l- \E[\hat{\lambda}^\eps_l] u_l \rangle_{L^2(\X_n)}-\E\left[\langle u_l, \Delta_n u_l- \E[\hat{\lambda}^\eps_l] u_l \rangle_{L^2(\X_n)}\right]\right),
\label{def:BulkTerm}
  \end{equation}
and 
  \begin{equation}
\EE_n := \sqrt{n}\left(\langle \hat{u}_l-u_l, \Delta_n u_l- \E[\hat{\lambda}^\eps_l] u_l \rangle_{L^2(\X_n)}-\E\left[\langle \hat{u}_l-u_l, \Delta_n u_l- \E[\hat{\lambda}^\eps_l] u_l \rangle_{L^2(\X_n)}\right]\right).
\label{def:ErrorTerm}
  \end{equation}
In the sequel, we refer to $\BB_n$ as the \textit{bulk term} and to $\EE_n$ as the \textit{error term}. These names are motivated by the fact that, as stated below, the error term $\EE_n$ vanishes as $n \rightarrow \infty$, while the bulk term $\BB_n$ converges to the Gaussian distribution in Theorem \ref{thm1}. 

For our analysis, we'll also need to control the following quantities
    \begin{equation}\label{eq-def:a_n b_n}
    \begin{split}
    a_n:= \Var(\sqrt{n} \hat{\lambda}_l^\veps) = n\E\left[\left(\hat{\lambda}_l^\eps-\E[\hat{\lambda}_l^\eps]\right)^2\right], \quad 
    b_n:=\mathrm{Var}\left(\EE_n\right).
    \end{split}
    \end{equation} 
  As it turns out, by investigating the relationship between $a_n$ and $b_n$, we will be able to study the behavior of $\BB_n$ and $\EE_n$ as $n \rightarrow \infty$. In turn, this will allow us to characterize the asymptotic distribution of $\sqrt{n}( \hat{\lambda}^\eps_l - \E[\hat{\lambda}^\eps_l]) $. In precise mathematical terms, the crucial ingredients in the proof of \cref{thm1} are the next two theorems.

     \begin{theorem}\label{lem:b_n control}
    Let $a_n$ and $b_n$ be defined as in \eqref{eq-def:a_n b_n}, and suppose that Assumptions \ref{assump:DataGenerating}, \ref{assump:Kernel}, and \ref{assum:eigengap for l} hold. Provided that $\frac{(\log n)^2}{n^{\frac{1}{m+4}}} \ll\eps\ll 1$, we have 
    	 \begin{equation}\label{eq:b_n upper bound by a_n}
    	b_n\leq \frac{C_l}{n\eps^{m+2}}a_n+c_n,
    	\end{equation}
    	where $c_n$ satisfies $c_n \rightarrow 0$ as $n \rightarrow \infty$. 
    \end{theorem}

       \begin{theorem}\label{thm}
    	Under Assumptions \ref{assump:DataGenerating}, \ref{assump:Kernel}, and \ref{assum:eigengap for l}, the following statements hold: 
    	\begin{enumerate}
    		\item \label{enu:3} For the bulk term $\BB_n$ defined in \eqref{def:BulkTerm}, provided $n^{-\frac{1}{m}}\ll\eps\ll 1$, we have
    		\begin{align*}
    		\BB_n \dto \mathcal{N}\left(0,\sigma_l^2\right),
    		\end{align*}
    		where $\sigma_l^2$ was defined in \eqref{eq-def:sigma^2}. Moreover, 
    		\begin{equation*}
    		\mathrm{Var}\left( \BB_n \right)\to \sigma_l^2.
    		\end{equation*}
    		\item\label{enu:1.a}  For the error term $\EE_n$ defined in \eqref{def:ErrorTerm},
    		provided $\frac{(\log n)^2}{n^{\frac{1}{m+4}}} \ll\eps\ll 1$, we have
    		\begin{equation*}
    		\EE_n \pto 0.
    		\end{equation*}
            In fact, 
            \[ \mathrm{Var}(\EE_n) \rightarrow 0. \]
    	\end{enumerate}
    	
    \end{theorem}


  The first part of \cref{thm} follows from the analysis of a related U-statistic, while the second part relies on first establishing \cref{lem:b_n control}. With \cref{thm} in hand, we can complete the proof of \cref{thm1}, which we present next.

 \begin{proof}[Proof of \cref{thm1}]

   Let us first observe that
   \begin{align}
   \begin{split}
\sqrt{n} \left(  \hat{\lambda}_l^\eps - \E[\hat{\lambda}^\eps_l] \right) & \E[ \langle \hat{u}_l, u_l \rangle_{L^2(\X_n)}  ] 
\\  = & \sqrt{n}\left((\hat{\lambda}_l^\eps - \E[\hat{\lambda}^\eps_l])\langle \hat{u}_l, u_l \rangle_{L^2(\X_n)}-\E\left[(\hat{\lambda}_l^\eps - \E[\hat{\lambda}^\eps_l])\langle \hat{u}_l, u_l \rangle_{L^2(\X_n)}\right]\right)
\\& - \mathcal{I}_n, 
\end{split}
       \label{eqn:AuxThm5_2}
   \end{align}
where 
\[ \mathcal{I}_n:= \sqrt{n} \left((\hat{\lambda}_l^\eps - \E[\hat{\lambda}^\eps_l])\left(\langle \hat{u}_l, u_l \rangle_{L^2(\X_n)}-\E[\langle \hat{u}_l, u_l \rangle_{L^2(\X_n)}]\right)-\E\left[(\hat{\lambda}_l^\eps - \E[\hat{\lambda}^\eps_l])\langle \hat{u}_l, u_l \rangle_{L^2(\X_n)}\right] \right).\]
We analyze each of the terms on the right-hand side of \eqref{eqn:AuxThm5_2}. 

First, observe that by combining (\ref{enu:3}) and (\ref{enu:1.a}) in Theorem \ref{thm} with equation \eqref{eq:decomposition} and Slutsky's theorem we deduce
    \begin{equation}
 	\sqrt{n}\left((\hat{\lambda}_l^\eps - \E[\hat{\lambda}^\eps_l])\langle \hat{u}_l, u_l \rangle_{L^2(\X_n)}-\E\left[(\hat{\lambda}_l^\eps - \E[\hat{\lambda}^\eps_l])\langle \hat{u}_l, u_l \rangle_{L^2(\X_n)}\right]\right)\dto \mathcal{N}(0,\sigma_l^2).
    \label{eq:AuxThm5}
 	\end{equation}

To analyze the term $\mathcal{I}_n$, note that from \cref{lem:spectral convergence in literature} we obtain 
    \begin{equation*}
        cn\E\left[\left(\hat{\lambda}_l^\eps-\E[\hat{\lambda}_l^\eps]\right)^2\right]-c_n' \leq n\E\left[\left(\hat{\lambda}_l^\eps-\E[\hat{\lambda}_l^\eps]\right)^2\langle \hat{u}_l, u_l \rangle^2_{L^2(\X_n)}\right]\leq Cn\E\left[\left(\hat{\lambda}_l^\eps-\E[\hat{\lambda}_l^\eps]\right)^2\right]+c_n',
    \end{equation*}
    for $c_n'\to 0$ as $n \rightarrow \infty$. This, together with \eqref{eq:lambda-hatlambda}, implies
    \begin{align}
    \label{eq:decompose of variance}
    \begin{split}
      c a_n =   cn\E\left[\left(\hat{\lambda}_l^\eps-\E[\hat{\lambda}_l^\eps]\right)^2\right] & \leq n\E\left[\left(\hat{\lambda}_l^\eps-\E[\hat{\lambda}_l^\eps]\right)^2\langle \hat{u}_l, u_l \rangle^2_{L^2(\X_n)}\right] + c_n'
        \\ & \leq 2\mathrm{Var}\left(\BB_n\right) + 2\mathrm{Var}\left( \EE_n\right) + c_n'.
        \end{split}
    \end{align}
 	Using the above bound and Theorem \ref{thm} we deduce
 	\begin{equation}\label{eq:a_n upper bound by constant}
 	\limsup_{n \rightarrow \infty} a_n\leq C\sigma_l^2.
 	\end{equation}
 	From \eqref{eq:a_n upper bound by constant} and \cref{lem:spectral convergence in literature}, we have	
    \begin{multline*}
  	n\mathrm{Var}\left((\hat{\lambda}_l^\eps - \E[\hat{\lambda}^\eps_l])\left(\langle \hat{u}_l, u_l \rangle_{L^2(\X_n)}-\E[\langle \hat{u}_l, u_l \rangle_{L^2(\X_n)}]\right)-\E\left[(\hat{\lambda}_l^\eps - \E[\hat{\lambda}^\eps_l])\langle \hat{u}_l, u_l \rangle_{L^2(\X_n)}\right]\right)\\
 	\leq n\E\left[\left(\hat{\lambda}_l^\eps - \E[\hat{\lambda}^\eps_l]\right)^2\left(\langle \hat{u}_l, u_l \rangle_{L^2(\X_n)}-\E[\langle \hat{u}_l, u_l \rangle_{L^2(\X_n)}]\right)^2\right]\to 0,
 	\end{multline*}
and, since $\E[\mathcal{I}_n] =0$, we conclude that $\mathcal{I}_n \pto 0$. 
    
All the above can be combined with the fact that $\E[\langle \hat{u}_l, u_l \rangle_{L^2(\X_n)}]\to 1$ to conclude, using Slutsky's theorem, that $\sqrt{n}\left(\hat{\lambda}_l^\eps-\E\left[\hat{\lambda}_l^\eps\right]\right)\dto \mathcal{N}\left(0,\sigma_l^2\right)$, as we wanted to show.
 \end{proof}

In the next subsections, we study the bulk term $\BB_n$ first, and then proceed to establish \cref{lem:b_n control}, which will imply the second part of \cref{thm}, where we characterize the asymptotic behavior of the error term $\EE_n$.


\subsection{Bulk Term}\label{sec:random term}

In this section, we present the proof of (\ref{enu:3}) in Theorem \ref{thm}. First, we observe that 
\begin{equation*}
    \sqrt{n}\left(\langle u_l, u_l\rangle_{L^2(\X_n)}-\E[\langle u_l, u_l\rangle_{L^2(\X_n)}]\right) (\lambda_l-\E[\hat{\lambda}_l^{\eps}])\pto 0,
\end{equation*}
which follows from Slutsky's theorem because  $\lambda_l-\E[\hat{\lambda}_l^{\eps}]\to 0$ as $n\to\infty$ (thanks to \cref{lem:spectral convergence in literature}) and$ \sqrt{n}\left(\langle u_l, u_l\rangle_{L^2(\X_n)}-\E[\langle u_l, u_l\rangle_{L^2(\X_n)}]\right)$ converges to a Gaussian distribution (using the standard CLT for empirical averages of i.i.d. random variables). From this, we deduce
\begin{equation*}
    \sqrt{n}\left(\langle \Delta_\rho u_l- \E[\hat{\lambda}_l^\eps] u_l, u_l\rangle_{L^2(\X_n)}-\E[\langle \Delta_\rho u_l- \E[\hat{\lambda}_l^\eps] u_l, u_l\rangle_{L^2(\X_n)}]\right) \pto 0.
\end{equation*}
Therefore, to prove \eqref{enu:3} in \cref{thm} it suffices to show that
\begin{align}
\label{aux:BulkTerm}
                &\sqrt{n}\left(\langle u_l, (\Delta_\rho-\Delta_n)u_l\rangle_{L^2(\X_n)}-\E\left[\langle u_l, (\Delta_\rho-\Delta_n)u_l\rangle_{L^2(\X_n)}\right]\right)\dto \mathcal{N}\left(0,\sigma_l^2\right).
                \end{align}
Now, consider the U-statistic
\begin{equation*}
    \begin{split}
    U_n &:=\langle u_l, (\Delta_\rho-\frac{n}{n-1}\Delta_n) u_l \rangle_{L^2(\X_n)} \\
    &= \frac{1}{n} \sum_{i=1}^n \biggl(\lambda_l u_l(x_i)^2 - \frac{n}{n-1} u_l(x_i) \Bigl(\frac{1}{n} \sum_{j=1}^n \eps^{-(m+2)} \eta\Bigl(\tfrac{\abs{x_i - x_j}}{\eps}\Bigr) (u_l(x_i) - u_l(x_j))\Bigr)\biggr) \\
    &= \frac{1}{n(n-1)} \sum_{i=2}^n \sum_{j<i} \Bigl(\lambda_l \bigl(u_l(x_i)^2 + u_l(x_j)^2\bigr) - \eps^{-(m+2)} \eta\Bigl(\tfrac{\abs{x_i - x_j}}{\eps}\Bigr) \bigl(u_l(x_i) - u_l(x_j)\bigr)^2\Bigr) \\
    &= \frac{1}{\binom{n}{2}} \sum_{i=2}^n \sum_{j<i} g_n(x_i,x_j),
    \end{split}
\end{equation*}
for the symmetric function $g_n:\mathcal{M}\times\mathcal{M}\to\R$ given by
\begin{equation}\label{eq-def:g_n}
    g_n(x,y) \coloneqq \tfrac{\lambda_l}{2} (u_l(x)^2 + u_l(y)^2) - \tfrac{1}{2} \eps^{-(m+2)} \eta\Bigl(\tfrac{\abs{x_i - x_j}}{\eps}\Bigr) (u_l(x) - u_l(y))^2,
\end{equation}
which depends on $n$ through $\eps$. It is straightforward to see that \eqref{aux:BulkTerm} holds if and only if  $\sqrt{n} (U_n - \E[U_n]) \dto \mathcal{N}(0, \sigma_l^2)$. For this reason, we focus on establishing the latter.  
\begin{lemma}
    If $\eps \gg n^{-1/m}$, then $\sqrt{n} (U_n - \E[U_n])$ has the same asymptotic distribution as $2\sqrt{n} (\tfrac{1}{n}\sum_{i=1}^n \E[g_n(x_i,x_j)\mid x_i] - \E[g_n(x_i,x_j)])$.
\label{lem:reduction_u-statistic_to_univariate}
\end{lemma}

\begin{proof}
We seek to apply \cite[Lemma 3.1]{PowellStockStoker1989}. For this, we need to prove that $\E[g_n(x_i,x_j)^2]$ is $o(n)$. 

Observe that
\begin{align*}
  \E[ g_n(x_i, x_j)^2 ]  & \leq \frac{\lambda_l^2}{2} \E[ (u_l(x_i)^2 + u_l(x_j)^2)] +  \frac{1}{2\veps^{2m+4}} \E \left[ \eta\left( \frac{|x_i- x_j|}{\veps} \right)^2 ( u_l(x_i) - u_l(x_j))^4 \right]  
  \\& \leq C_l + \frac{C_l}{\veps^{m}},
\end{align*}
where in the second line we used the regularity of $u_l$, which implies that $u_l$ is uniformly bounded and also Lipschitz, so that, in particular, for $x, y$ within distance $\veps$ we have $(u_l(x) - u_l(y) )^4 \leq C_l \veps^4$. In turn, the term $C_l + \frac{C_l}{\veps^{m}}$ is $o(n)$ under the assumption that $\eps \gg n^{-1/m}$. The result follows from \cite[Lemma 3.1]{PowellStockStoker1989}.

\end{proof}

By \cref{lem:reduction_u-statistic_to_univariate}, to characterize the asymptotic distribution of $\BB_n$ we can  focus on establishing a CLT for a triangular array of independent random variables, and for this we can rely on Lindeberg's CLT.

\begin{proof}[Proof of \ref{enu:3} in \cref{thm}]

Thanks to \cref{lem:reduction_u-statistic_to_univariate}, and given that we have assumed $\eps \gg n^{-1/m}$, the asymptotic distribution of $\B_n$ is the same as that of $\sqrt{n}(U_n - \E[U_n])$ which is, in turn, the same as that of
\begin{equation}
    2\sqrt{n} \left(\frac{1}{n}\sum_{i=1}^n \Bigl(\E[g_n(x_i,x_j)\mid x_i] - \E[g_n(x_i,x_j)]\Bigr)\right) \eqqcolon 2\sqrt{n} \left(\frac{1}{n}\sum_{i=1}^n \Bigl(Y_{i,n} - \E[Y_{i,n}]\Bigr)\right),
    \label{eq:clt:u-statistic_reduction_result}
\end{equation}
where 
\[ Y_{i,n}\coloneqq \E[g_n(x_i,x_j)\mid x_i]\, \quad i=1, \dots, n.\]
Note that $Y_{1,n}, \dots, Y_{n,n}$ are i.i.d. random variables and thus it suffices to check the conditions to apply Lindeberg's CLT.

We start by computing $\E[Y_{i,n}]$. First, using \eqref{e.comparegeodesic}, the fact that $\eta$ was assumed Lipschitz, and the regularity of $u_l$, we observe that
\begin{equation*}
    \begin{split}
        &\Biggl|\E[Y_{i,n}] -\E\Bigl[\E\Bigl[\tfrac{\lambda_l}{2} (u_l(x_i)^2 + u_l(x_j)^2) - \tfrac{1}{2}\eps^{-(m+2)}\eta\Bigl(\tfrac{d(x_i , x_j)}{\eps}\Bigr) (u_l(x_i) - u_l(x_j))^2\Big\lvert x_i\Bigr]\Bigr]\Biggr|\\
        &\leq \tfrac{1}{2}\eps^{-(m+2)}\E\Bigl[\Bigl|\eta\Bigl(\tfrac{d(x_i , x_j)}{\eps}\Bigr)-\eta\Bigl(\tfrac{\abs{x_i - x_j}}{\eps}\Bigr) \Bigr|\bigl(u_l(x_i)-u_l(x_j)\bigr)^2\Bigr]\leq C\eps^{2},
    \end{split}
\end{equation*}
where we recall $d(\cdot, \cdot)$ denotes $\M$'s geodesic distance.
Continuing the computation, we use normal coordinates (see Appendix \ref{App:GeoBack}) to write
    \begin{align*}
        \begin{split}
        & \E\Bigl[\E\Bigl[\tfrac{\lambda_l}{2} (u_l(x_i)^2 + u_l(x_j)^2) - \tfrac{1}{2}\eps^{-(m+2)}\eta\Bigl(\tfrac{d(x_i , x_j)}{\eps}\Bigr) (u_l(x_i) - u_l(x_j))^2\Big\lvert x_i\Bigr]\Bigr] \\
        &= \lambda_l \E[u_l(x_i)^2] - \tfrac{1}{2}\eps^{-(m+2)}\E\Bigl[\eta\Bigl(\tfrac{d(x_i , x_j)}{\eps}\Bigr) \bigl(u_l(x_i)-u_l(x_j)\bigr)^2\Bigr] \\
        &= \lambda_l \int_\M u_l(x)^2 \rho(x) \dx -\tfrac{1}{2}\eps^{-(m+2)}\int_{\M}\int_{B_{\eps}(0)\subset T_{x}\M}\eta\Bigl(\tfrac{\abs{v}}{\eps}\Bigr) \\
        & \qquad \qquad \qquad \times \bigl(w_{x}(0)-w_{x}(v)\bigr)^2 p_{x}(0)p_{x}(v)J_{x}(v)\dd v\dd  x\\
        &=\lambda_l -\tfrac{1}{2}\eps^{-2}\int_{\M}\int_{B_1(0)\subset T_{x}\M}\eta\Bigl(\abs{v}\Bigr)\\
        &\qquad \qquad \qquad \times \bigl(w_{x}(0)-w_{x}(\eps v)\bigr)^2 p_{x}(0)p_{x}(\eps v)J_{x}(\eps v)\dd v\dd x\\
        &= \lambda_l  -\tfrac{1}{2}\int_{\M}\int_{B_1(0)\subset T_{x}\M}\eta\Bigl(\abs{v}\Bigr) \bigl(\nabla w_{x}(0)\cdot v\bigr)^2 p_{x}^2(0)\dd v\dd x+\BO(\eps)\\
        &= \lambda_l -\tfrac{1}{2}\int_{\M}\int_{B_1(0)\subset T_{x}\M}\eta\Bigl(\abs{v}\Bigr) \bigl(\nabla w_{x}(0)\cdot v\bigr)^2 p_{x}^2(0)\dd v\dd x+\BO(\eps)\\
        &=\lambda_l - \int_{\M}|\nabla u_l(x)|^2\rho^2(x) \dd x  + O(\eps)
        \\ & =  O(\eps).
    \end{split}
    \end{align*}
In the second equality, we used $w_x(v):=u_l(\exp_x(v))$ and $p_x(v)=\rho(\exp_x(v))$, i.e., $w$ and $p$ are, respectively, the functions $f$ and $\rho$ expressed in normal coordinates at $x$; see Appendix \ref{App:GeoBack}. To go from the third to the fourth line we used  Taylor expansions for $w_x$ and $p_x$, as well as \eqref{e.Jactaylor}. To go from the sixth to the seventh line we used that $\eta$ is normalized according to \eqref{eqn:AssumpKernel}. The bottom line is that 
\[ \E[Y_{i,n}] = O(\veps). \]

Now, we turn our attention to analyzing $\E[Y_{i,n}^2]$. A 
similar computation as before reveals
\begin{align*}
    \E[Y_{i,n}^2] &=  \E\Bigl[\Bigl(\tfrac{\lambda_l}{2} u_l(x_i)^2+\E\Bigl[\tfrac{\lambda_l}{2} u_l(x_j)^2\Bigr] \\
    &\qquad -\E\Bigl[ \tfrac{1}{2}\eps^{-(m+2)}\eta\Bigl(\tfrac{d(x_i , x_j)}{\eps}\Bigr) (u_l(x_i) - u_l(x_j))^2\mid x_i\Bigr]\Bigr)^2\Bigr]+O(\eps).\\
\end{align*}
On the other hand,
\begin{align*}
    &\E\Bigl[\Bigl(\tfrac{\lambda_l}{2} u_l(x_i)^2+\E\Bigl[\tfrac{\lambda_l}{2} u_l(x_j)^2\Bigr] \\
    &\qquad -\E\Bigl[ \tfrac{1}{2}\eps^{-(m+2)}\eta\Bigl(\tfrac{d(x_i , x_j)}{\eps}\Bigr) (u_l(x_i) - u_l(x_j))^2\mid x_i\Bigr]\Bigr)^2\Bigr]\\
    &=\E\Bigl[\Bigl(\tfrac{\lambda_l}{2} u_l(x_i)^2+ \tfrac{\lambda_l}{2}  \\
    &\qquad  -\int_{B_{\eps}(0)\subset T_{x_i}\M} \tfrac{1}{2}\eps^{-(m+2)}\eta\Bigl(\tfrac{\abs{v}}{\eps}\Bigr) (w_{x_i}(0) - w_{x_i}(v))^2 p_{x_i}(v) J_{x_i}(v)\dd v\Bigr)^2\Bigr]\\
    &=\E\Bigl[\Bigl(\tfrac{\lambda_l}{2} u_l(x_i)^2+\tfrac{\lambda_l}{2} \\
    &\qquad  -\int_{B_1(0)\subset T_{x_i}\M} \tfrac{1}{2}\eps^{-2}\eta\Bigl(\abs{v}\Bigr) (w_{x_i}(0) - w_{x_i}(\eps v))^2 p_{x_i}(\eps v) J_{x_i}(\eps v) \dd v\Bigr)^2\Bigr]\\
    &= \E\Bigl[\Bigl(\tfrac{\lambda_l}{2} u_l(x_i)^2+\tfrac{\lambda_l}{2}  \\
    &\qquad -\int_{B_1(0)\subset T_{x_i}\M} \tfrac{1}{2}\eta\Bigl(\abs{v}\Bigr) \Bigl(\nabla w_{x_i}(0)\cdot v\Bigr)^2p_{x_i}(0)\dd v+\BO(\eps)\Bigr)^2\Bigr]\\
    &=\E\Bigl[\Bigl(\tfrac{\lambda_l}{2} u_l(x_i)^2+ \tfrac{\lambda_l}{2}  -  |\nabla u_l(x_i) |^2\rho(x_i)+O(\eps)\Bigr)^2\Bigr]\\
    &=\E\Bigl[\Bigl(\tfrac{\lambda_l}{2} u_l(x_i)^2+\tfrac{\lambda_l}{2}  - |\nabla u_l(x_i) |^2\rho(x_i)\Bigr)^2\Bigr]+O(\eps)\\
    &=\int_\M\Bigl(\tfrac{\lambda_l}{2} u_l(x)^2+\tfrac{\lambda_l}{2}  - |\nabla u_l(x) |^2\rho(x)\Bigr)^2\rho(x)\dd x+O(\eps).
\end{align*}

By combining the above equations, we deduce
\begin{equation}
\begin{split}
    \Var[Y_{i,n}] 
    &= \E[Y_{i,n}^2] - \E[Y_{i,n}]^2 \\
    &= \int_\M\Bigl(\tfrac{\lambda_l}{2} u_l(x)^2+\tfrac{\lambda_l}{2}  - |\nabla u_l(x) |^2\rho(x)\Bigr)^2\rho(x)\dd x+O(\eps).
\end{split}
\label{eq:clt:variance_convergence}
\end{equation}
In particular 
\[ s_n^2:= \frac{1}{n}\sum_{i=1}^n \Var[Y_{i,n}] = \Var[Y_{1,n}] \rightarrow \frac{\sigma_l^2}{4}   , \quad \text{ as } n \rightarrow \infty.\]
To conclude that 
\[ 2\sqrt{n} \left(\frac{1}{n}\sum_{i=1}^n \Bigl(Y_{i,n} - \E[Y_{i,n}]\Bigr)\right)  \dto \mathcal{N}(0, \sigma_l^2 ),\] 
and by virtue of \eqref{eq:clt:u-statistic_reduction_result} deduce the desired result, it thus suffices to check Lindeberg's condition, which, recall, requires that 
\begin{equation}
    \frac{1}{s_n^2} \sum_{i=1}^n \E\bigl[(Y_{i,n} - \E[Y_{i,n}])^2\mathbbm{1}_{\{\abs{Y_{i,n} - \E[Y_{i,n}]} > \delta s_n\}}\bigr] \nto 0,
\label{eq:lindeberg}
\end{equation}
for every $\delta > 0$. This, however, follows immediately from the fact that the random variables $Y_{i,n}$ are uniformly bounded in $n$. Indeed, note that
\begin{align*}
|Y_{i,n}| & = \left| \E\Bigl[\tfrac{\lambda_l}{2} (u_l(x_i)^2 + u_l(x_j)^2) - \tfrac{1}{2}\eps^{-(m+2)}\eta\Bigl(\tfrac{d(x_i , x_j)}{\eps}\Bigr) (u_l(x_i) - u_l(x_j))^2\Big\lvert x_i\Bigr] \right|
\\ & \leq C_l,
\end{align*}
where in the above we have used the fact that $u_l$ is Lipschitz to argue that $(u_l(x)- u_l(y))^2 \leq C_l \veps^2$ for all $x,y$ within distance $\veps$ from each other.

\end{proof}

\subsection{Proof of \cref{lem:b_n control} and Analysis of the Error Term}\label{secLlem-b_n control}

In this section, we present the proof of \cref{lem:b_n control} and of (\ref{enu:1.a}) in \cref{thm}. As our analysis will reveal, (\ref{enu:1.a}) in \cref{thm} will follow once we establish the bound \eqref{eq:b_n upper bound by a_n} for $b_n= \mathrm{Var}(\EE_n)$.

\subsubsection{Preliminary decomposition for the error term}

We start our analysis by recalling the following perturbation inequality used to control the variance of a function of independent random variables. 
    \begin{lemma}[Efron-Stein inequality]\label{lem:efron stein}
        Suppose $x_1,\dots,x_n,x_1',\dots,x_n'$ are independent with $x_i,x_i'$ having the same distribution. Let $\X_n=(x_1,\dots,x_n)$ and $\X_n^{(i)}=(x_1,\dots,x_{i-1},x_i',x_{i+1},\dots,x_n)$ and let $\beta:\M^{n}\to\R$ be a measurable function. Then
        \begin{equation*}
            \mathrm{Var}(\beta(x_1,\dots,x_n))\leq \frac{1}{2}\sum_{i=1}^n \E\left[\left(\beta(\X_n)-\beta(\X_n^{(i)})\right)^2\right].
        \end{equation*}
    \end{lemma}
    The proof of the Efron-Stein inequality is available in \cite{bousquet2011advanced}. Other versions of this inequality can be used to obtain sharper bounds on the tail probability of a function of independent random variables, but since our ultimate goal in \cref{lem:b_n control} is to prove the convergence in probability of a sequence of random variables toward zero, it will suffice to show that the sequence of associated variances goes to zero by using the version of the Efron-Stein inequality stated above.

\begin{proof}[Proof of \cref{lem:b_n control}]
    In what follows we consider a measurable function $\beta:\M^n\to \R$ satisfying
    \begin{equation*}
        \beta(x_1,\dots,x_n)= \langle \hat{u}^{\X_n}_l-u_l,\E[\hat{\lambda}_l^\veps] u_l - \Delta^{\X_n}_n u_l\rangle_{L^2(\X_n)},
    \end{equation*}
    where $\Delta^{\X_n}_n$ is the graph Laplacian constructed with the points $\X_n = \{ x_1, \dots, x_n \}$, and $\hat{u}^{\X_n}_l$ is a corresponding normalized $l$-th eigenvector satisfying 
    \[\langle \hat{u}_l^{\X_n}, u_l \rangle_{L^2(\X_n)} \geq 0 .\] 
    Existence of such measurable function follows from standard measurable selection theorems.    
    
    Notice that $b_n$ can be written as $b_n = n \mathrm{Var}(\beta(\X_n))$ for $\X_n$ an i.i.d sample of $\rho$. Therefore, using the Efron-Stein inequality, i.e., \cref{lem:efron stein}, to prove \cref{lem:b_n control} it will suffice to show that
    \begin{equation}\label{eq:claim b_nto 0}
        b_n\leq \frac{n}2\sum_{i=1}^n \E[(\beta(\X_n)-\beta(\X_n^{(i)}))^2]\leq \frac{C}{n\eps^{m+2}}a_n+c_n,
    \end{equation} 
    for $\X_n$ and $\X_n^{(i)}$ as in the Efron-Stein inequality and $c_n \rightarrow 0$ as $n \rightarrow \infty$.
    In turn, we can focus on establishing
    \begin{equation}\label{eq-A1}
        \begin{split}
            A_1&:=\frac{1}{n}\sum_{i=1}^n \E\Bigl[\Bigl(\sum_{j:j\not= i}\left(\hat{u}_l^{\X_n}(x_j)-u_l(x_j)\right)(\E[\hat{\lambda}_l^\veps] u_l(x_j) - \Delta^{\X_n}_n u_l(x_j))\\
            &\qquad \qquad -\sum_{j:j\not= i}\left(\hat{u}_l^{\X_n^{(i)}}(x_j)-u_l(x_j)\right)(\E[\hat{\lambda}_l^\veps] u_l(x_j) - \Delta^{\X_n^{(i)}}_n u_l(x_j))\Bigr)^2\Bigr]
            \\ & \le \frac{C}{n\eps^{m+2}}a_n+c_n,
        \end{split}
    \end{equation} 





    \begin{equation}\label{eq-A2}
        A_2:=\frac{1}{n}\sum_{i=1}^n\E\left[\left((\hat{u}_l^{\X_n}(x_i)-u_l(x_i))((\Delta_\rho-\Delta_n^{\X_n})u_l(x_i))\right)^2\right]\to 0,
    \end{equation} 
    \begin{equation}\label{eq-A3}
        A_3:=\frac{1}{n}\sum_{i=1}^n\E\left[\left((\hat{u}_l^{\X^{(i)}_n}(x_i')-u_l(x_i'))((\Delta_\rho-\Delta_n^{\X^{(i)}_n})u_l(x'_i))\right)^2\right]\to 0,
    \end{equation} 
    \begin{equation}\label{eq-A4}
        A_4:=\frac{(\E[\hat{\lambda}_l^\veps]-\lambda_l)^2}{n}\sum_{i=1}^n\E\left[\left((\hat{u}_l^{\X_n}(x_i)-u_l(x_i))u_l(x_i)\right)^2\right]\to 0,
    \end{equation}
    and 
    \begin{equation}\label{eq-A5}
        A_5:=\frac{(\E[\hat{\lambda}_l^\veps]-\lambda_l)^2}{n}\sum_{i=1}^n\E\left[\left((\hat{u}_l^{\X^{(i)}_n}(x_i')-u_l(x_i'))u_l(x'_i)\right)^2\right]\to 0,
    \end{equation}
    because a straightforward computation using the triangle inequality reveals that
    \begin{equation*}
         n\sum_{i=1}^n \E[(\beta(\X_n)-\beta(\X_n^{(i)}))^2] \leq 5(A_1+A_2+A_3+A_4+A_5).
    \end{equation*}
  Notice that $A_2=A_3$ and $A_4=A_5$.

In what follows we focus on proving \eqref{eq-A2}, \eqref{eq-A4}, and \eqref{eq-A1}.

\subsubsection{Proof of \eqref{eq-A2}}
For \eqref{eq-A2}, let us first apply Cauchy-Schwarz inequality to obtain
    \begin{equation*}
        \begin{split}
            &\frac{1}{n}\sum_{i=1}^n\E\left[\left((\hat{u}_l^{\X_n}(x_i)-u_l(x_i))((\Delta_\rho-\Delta_n^{\X_n})u_l(x_i))\right)^2\right]\\
        &\quad  \quad \quad   \leq \E[\|u_l-\hat{u}_l^{\X_n}\|_{L^2(\X_n)}^2\|(\Delta_\rho-\Delta_n^{\X_n})u_l\|_{L^\infty(\X_n)}^2]. 
        \end{split}
    \end{equation*}
   Now, thanks to \cref{lem:spectral convergence in literature}, with probability at least 
\[ 1-Cn\eps^{-m} \exp(-c n\eps^{m+1}t^2),\]
     we have
     \begin{equation*}
         \begin{split}
             \|(\Delta_\rho-\Delta_n^{\X_n})u_l\|_{L^\infty(\X_n)}\leq C_l \left(\frac{t}{\sqrt{\eps}}+\eps\right), \quad 
             \|u_l-\hat{u}_l^{\X_n}\|_{L^2(\X_n)}\leq C_l(t+\eps^2).
         \end{split}
     \end{equation*}
     In addition, from \eqref{eq:normalization of hatu} we have the deterministic bound
     \begin{equation*}
         \|\hat{u}^{\X_n}_l\|_{L^\infty(\X_n)}\leq \sqrt{n},
     \end{equation*}
     which implies 
     \begin{align*}
        & \left(\left(\hat{u}_l^{\X_n}(x_i)-u_l(x_i)\right) \left((\Delta_\rho-\Delta_n^{\X_n})u_l(x_i)\right)\right)^2 \leq Cn\max_{x_i\in\X_n}\left|(\Delta_\rho-\Delta_n^{\X_n})u_l(x_i)\right|^2 
         \\ & \qquad \leq Cn\max_{x_i\in\X_n}\left|\Delta_\rho u_l(x_i)\right|^2+n\max_{x_i\in\X_n}\left|\Delta_n^{\X_n} u_l(x_i)\right|^2
         \\ & \qquad \leq Cn + n\frac{C}{\eps^{2m+2}}\|\nabla u_l\|_{L^\infty(\M)}^2
         \\& \qquad \leq Cn \eps^{-(2m+2)},
     \end{align*}
     for all $i = 1, \dots, n$. Combining the above bounds, we obtain 
    \begin{equation}\label{eq:A_2 proof}
        \begin{split}
            A_2&= \E\Bigl[\left((\hat{u}_l^{\X_n}(x_i)-u_l(x_i))((\Delta_\rho-\Delta_n^{\X_n})u_l(x_i))\right)^2\\
            &\qquad\qquad  \times \mathbbm{1}_{ \|(\Delta_\rho-\Delta_n^{\X_n})u_l\|_{L^\infty(\X_n)}\leq  C_l(\frac{t}{\sqrt{\eps}}+\eps),
             \|u_l-\hat{u}_l\|_{L^2(\X_n)}\leq C_l(t+\eps^2)}\Bigr]\\
             &\quad +\E\Bigl[\left((\hat{u}_l^{\X_n}(x_i)-u_l(x_i))((\Delta_\rho-\Delta_n^{\X_n})u_l(x_i))\right)^2\\
             &\qquad\qquad  \times \mathbbm{1}_{ \|(\Delta_\rho-\Delta_n^{\X_n})u_l\|_{L^\infty(\X_n)}> C_l(\frac{t}{\sqrt{\eps}}+\eps) \text{ or }
             \|u_l-\hat{u}_l\|_{L^2(\X_n)}> C_l(t+\eps^2)}\Bigr]\\
            &\leq\frac{Cn^2}{\eps^{3m+2}}C\exp(-cn\eps^{m+1}t^2) + C_l \left(\frac{t^2}{\sqrt{\eps}}+ t\eps + t \veps^{3/2} + \veps^3\right)^2.
        \end{split}
    \end{equation}
Taking $t= \frac{\veps^{1/4}}{d_n}$, with $d_n$ a poly-logarithmic function in $n$ (in particular $d_n \rightarrow \infty$ as $n \rightarrow \infty$), and recalling that for our choice of $\veps$ we have $ (\frac{\log n}{n})^{\frac{1}{m+3/2}} \ll \veps \ll 1$, it follows that the right hand side of the above expression goes to zero as $n \rightarrow \infty$. 
    


\subsubsection{Proof of \eqref{eq-A4}}

That $A_4\to 0$ follows directly from \cref{lem:spectral convergence in literature} because we have $(\E[\hat{\lambda}^{\veps}_l]-\lambda_l)\to 0$ and 
\begin{align*}
\E\left[ \frac{1}{n} \sum_{i=1}^n\left((\hat{u}_l^{\X_n}(x_i)-u_l(x_i))u_l(x_i)\right)^2\right] & \leq  C_l \E\left[  \lVert \hat{u}_l^{\X_n} -u_l \rVert^2_{L^2(
\X_n
)}\right] 
\\& \leq C_l \E \left[ 2\lVert \hat{u}_l^{\X_n} \rVert^2_{L^2(\X_n)}  + 2 \lVert u_l \rVert^2_{L^2(\X_n)} \right] 
\\& \leq C_l. 
\end{align*}

\subsubsection{Further decomposition of $A_1$}
In order to prove \eqref{eq-A1}, and complete in this way the proof of \cref{lem:b_n control}, it will be convenient to decompose $A_1$ into two simpler terms that we will analyze one by one. 

To begin, for a given $i=1, \dots, n$ we introduce the graph Laplacian $\Delta_{n-1}^{(i)}$, which acts on a scalar function $f$ according to
    \begin{equation}\label{eq-def:Delta_n-1}
        \Delta^{(i)}_{n-1} f(x_j) := \frac{1}{n\eps^{m+2}}\sum_{k\in [n]\backslash \{i\}} \eta\left(\frac{|x_k-x_j|}{\eps}\right) (f(x_j)-f(x_k)).
    \end{equation}
    Note that $\Delta_{n-1}^{(i)}$ leaves out the terms coming from $x_i$ and $x_i'$ in the expressions for $\Delta_n^{\X_n}$ and $\Delta_n^{\X_n^{(i)}}$, respectively. More precisely, for any fixed $i$ and any $j\not= i $, we can write
    \begin{equation}\label{eq:Expansion}
        \begin{split}
            \Delta_n^{\X_n} f(x_j) = \Delta_{n-1}^{(i)} f(x_j) + \frac{1}{n\eps^{m+2}}\eta\left(\frac{|x_i-x_j|}{\eps}\right) (f(x_j)-f(x_i))  ,\\
            \Delta_n^{\X_n^{(i)}} f(x_j) = \Delta_{n-1}^{(i)} f(x_j) +\frac{1}{n\eps^{m+2}}\eta\left(\frac{|x_i'-x_j|}{\eps}\right) (f(x_j)-f(x_i')) .
        \end{split} 
    \end{equation}
   Using the definition for $\Delta_{n-1}^{(i)}$ and \eqref{eq:Expansion}, it follows
   \[ A_1\leq 2B_1 + 2B_2, \]
   where
    \begin{equation}\label{eq:B_1}
            B_1 := \frac{1}{n}\sum_{i=1}^n \E\Bigg[ \Bigg(\sum_{j:j\not= i} (\hat{u}^{\X_n^{(i)}}_l(x_j)-\hat{u}^{\X_n}_l(x_j))(\Delta^{(i)}_{n-1} u_l(x_j)-\E[\hat{\lambda}_l^\veps]u_l(x_j))\Bigg)^2\Bigg],
    \end{equation}
    and
    \begin{multline}\label{eq:B_2}
        B_2 := \frac{1}{n}\sum_{i=1}^n \E\Biggl[ \Biggr(\sum_{j:j\not= i}(u_l(x_j)-\hat{u}^{\X_n}_l(x_j))\frac{1}{n\eps^{m+2}}\eta\left(\frac{|x_i-x_j|}{\eps}\right) (u_l(x_i)-u_l(x_j))\\
        - \sum_{j:j\not= i}(u_l(x_j)-\hat{u}^{\X^{(i)}_n}_l(x_j))\frac{1}{n\eps^{m+2}}\eta\left(\frac{|x_i'-x_j|}{\eps}\right) (u_l(x_i')-u_l(x_j))\Biggr)^2\Biggr].
    \end{multline}
    We will show that 
    \begin{equation}\label{eq:claim B_1}
        B_1\leq \frac{C_l}{n\eps^{m+2}}a_n+c_n,
    \end{equation}
    and 
    \begin{equation}\label{eq:claim B_2}
        B_2\to 0.
    \end{equation}
   The verification of \eqref{eq:claim B_1} and \eqref{eq:claim B_2} will be the focus of the analysis in the remainder of this section, as they would immediately imply the desired bound for $A_1$.

\subsubsection{Proof of \eqref{eq:claim B_2}}
For $B_2$, observe that
    \begin{equation}
        \begin{split}\label{eq:B_3 upper bound}
            B_2 &\leq 4\E\Biggl[ \Biggr(\sum_{j:j\not= i}(u_l(x_j)-\hat{u}^{\X_n}_l(x_j))\frac{1}{n\eps^{m+2}}\eta\left(\frac{|x_i-x_j|}{\eps}\right) (u_l(x_i)-u_l(x_j))\Biggr)^2\Biggr] \\
        &\leq 8\E\Biggl[ \Biggr(\sum_{j:j\not= i}\bigl(u_l(x_j)-\hat{u}^{\X_n}_l(x_j)-(u_l(x_i)-\hat{u}^{\X_n}_l(x_i))\bigr)\frac{1}{n\eps^{m+2}}\eta\left(\frac{|x_i-x_j|}{\eps}\right) (u_l(x_i)-u_l(x_j))\Biggr)^2\Biggr]\\
        &\qquad + 8\E\Biggl[ \Biggr(\sum_{j:j\not= i}\bigl(u_l(x_i)-\hat{u}^{\X_n}_l(x_i)\bigr)\frac{1}{n\eps^{m+2}}\eta\left(\frac{|x_i-x_j|}{\eps}\right) (u_l(x_i)-u_l(x_j))\Biggr)^2\Biggr] \\
        &= 8\E\Biggl[ \Biggr(\sum_{j:j\not= i}\bigl(u_l(x_j)-\hat{u}^{\X_n}_l(x_j)-(u_l(x_i)-\hat{u}^{\X_n}_l(x_i))\bigr)\frac{1}{n\eps^{m+2}}\eta\left(\frac{|x_i-x_j|}{\eps}\right) (u_l(x_i)-u_l(x_j))\Biggr)^2\Biggr]\\
        &\qquad + 8\E\Biggl[ (u_l(x_i) - \hat{u}_l^{\X_n}(x_i) )^2 (\Delta_n^{\X_n} u_l(x_i))^2 \Biggr].
        \end{split}
    \end{equation}    
    
    For the second term on the right-hand side of \eqref{eq:B_3 upper bound}, using the fact that the $x_i$ are i.i.d., we have
    \begin{align*}
        & \E\Biggl[ (u_l(x_i) - \hat{u}_l^{\X_n}(x_i) )^2 (\Delta_n^{\X_n} u_l(x_i))^2 \Biggr] =     \frac{1}{n} \sum_{i=1}^n\E\Biggl[ (u_l(x_i) - \hat{u}_l^{\X_n}(x_i) )^2 (\Delta_n^{\X_n} u_l(x_i))^2 \Biggr]
\\ &   \quad \leq  \E\Biggl[ \lVert u_l - \hat{u}_l^{\X_n}  \rVert^2_{L^2(\X_n)} \lVert \Delta_n^{\X_n} u_l  \rVert_{L^\infty(\X_n)}^2  \Biggr]
\\&  \quad \leq  \E\Biggl[ \lVert u_l - \hat{u}_l^{\X_n}  \rVert^2_{L^2(\X_n)} \lVert \Delta_n^{\X_n} u_l  \rVert_{L^\infty(\X_n)}^2  \one_{\{ \lVert u_l - \hat{u}_l^{\X_n} \rVert_{L^2(\X_n)} \leq C_l t  \text{ and } \lVert \Delta_n^{\X_n} u_l \rVert_{L^\infty(\X_n)} \leq C_l \} } \Biggr]
\\&  \qquad +   \E\Biggl[ \lVert u_l - \hat{u}_l^{\X_n}  \rVert^2_{L^2(\X_n)} \lVert \Delta_n^{\X_n} u_l  \rVert_{L^\infty(\X_n)}^2  \one_{\{ \lVert u_l - \hat{u}_l^{\X_n} \rVert_{L^2(\X_n)} > C_l t  \text{ or } \lVert \Delta_n^{\X_n} u_l \rVert_{L^\infty(\X_n)} > C_l \} } \Biggr]
\\&  \quad \leq  C_lt^2 + C_ln\eps^{-(m+1)}(\exp(-c n\eps^{m}t^2) + \exp(-c n\veps^{m+2})),
\end{align*}
where the last inequality follows from \cref{lem:spectral convergence in literature}. When we choose $t=\frac{1}{d_n}$ for a suitable $d_n\to \infty$, and use the fact that $1\gg \eps\gg \frac{(\log n)^2}{n^{\frac{1}{m+2}}}$, the above upper bound can be seen to converge to $0$ as $n \rightarrow \infty$.

For the first term on the right-hand side of \eqref{eq:B_3 upper bound}, we use Jensen's inequality, the fact that $u_l$ is Lipschitz (in particular, $\sup_{x,y:|x-y|\leq \eps} |u_l(x)-u_l(y)|\leq C_l\eps$), the fact that $\eta$ is bounded, and the fact that the $x_i$ are i.i.d., to deduce
    \begin{equation*}
        \begin{split}
            &\E\Biggl[ \Biggr(\sum_{j:j\not= i}\bigl(u_l(x_j)-\hat{u}^{\X_n}_l(x_j)-(u_l(x_i)-\hat{u}^{\X_n}_l(x_i))\bigr)\frac{1}{n\eps^{m+2}}\eta\left(\frac{|x_i-x_j|}{\eps}\right) (u_l(x_i)-u_l(x_j))\Biggr)^2\Biggr]
            \\ &\quad\leq C_l\E\Biggl[ \left(\frac{N_i}{n \veps^m }\right) \frac{1}{n\veps^{m+2}}\sum_{j:j\not= i}\eta\left(\frac{|x_i-x_j|}{\eps}\right) \biggl(u_l(x_j)-\hat{u}^{\X_n}_l(x_j)-(u_l(x_i)-\hat{u}^{\X_n}_l(x_i))\biggr)^2\Biggr]
            \\ &\quad \leq C_l\E\Biggl[ \left(\frac{N}{n \veps^m }\right) \frac{1}{n^2\veps^{m+2}} \sum_{i=1}^n\sum_{j:j\not= i}\eta\left(\frac{|x_i-x_j|}{\eps}\right) \biggl(u_l(x_j)-\hat{u}^{\X_n}_l(x_j)-(u_l(x_i)-\hat{u}^{\X_n}_l(x_i))\biggr)^2\Biggr]
            \\&\quad =C_l \E\left[ \left(\frac{N}{n \veps^m }\right) \|u_l-\hat{u}_l^{\X_n}\|^2_{H^1(\X_n)}\right],
        \end{split}
    \end{equation*}
where $N_i := \# \{ j \text{ s.t. } |x_i- x_j| \leq \veps \}$ and $N := \max_{i=1, \dots, n} N_i$. In turn, the latter expression can be bounded by
\begin{align*}
          C_l\E\left[ \left(\frac{N}{n \veps^m }\right) \|u_l-\hat{u}_l^{\X_n}\|^2_{H^1(\X_n)}\right]  & =C_l \E\left[\|u_l-\hat{u}_l^{\X_n}\|^2_{H^1(\X_n)}\mathbbm{1}_{ \{ \|u_l-\hat{u}_l^{\X_n}\|_{H^1(\X_n)}\leq C\eps^2+t \text{ and } N\leq C n\veps^m \}}\right]\\
            &\quad +C_l \E\left[\|u_l-\hat{u}_l^{\X_n}\|^2_{H^1(\X_n)}\mathbbm{1}_{\{\|u_l-\hat{u}_l^{\X_n}\|^2_{H^1(\X_n)}> C\eps^2+t \text{ or } N> C n\veps^m \}}\right]\\
            & \leq C_l(t+\eps^2)^2+C_ln\eps^{-3m-2} \exp(-cn\eps^{m}t^2) , 
\end{align*}
 where the last inequality follows from \cref{lem:spectral convergence in literature}, a standard concentration argument to control $N$ by a constant times $n \veps^m$, and the deterministic estimate
    $$\|u_l-\hat{u}_l\|_{H^1(\X_n)}^2\leq \frac{C_l}{n^2\eps^{m+2}}\sum_{i,j} (|u_l(x_i)|+|\hat{u}_l(x_i)|+|u_l(x_j)|+|\hat{u}_l(x_j)|)^2\leq \frac{C_l}{\eps^{m+2}}.$$
    Choosing $t=\frac{1}{d_n}$ for proper $d_n\to \infty$ and $1\gg \eps\gg \frac{(\log n)^2}{n^\frac{1}{m}}$, we conclude that $C_l\E\left[ \left(\frac{N}{n \veps^m }\right) \|u_l-\hat{u}_l^{\X_n}\|^2_{H^1(\X_n)}\right]  \rightarrow 0$. Putting everything together, we deduce \eqref{eq:claim B_2}.

\subsubsection{Leave-one-out stability estimates of graph Laplacians}


In order to analyze $B_1$, we observe that, since the summands in the outer sum in the definition of $B_1$ are all equal, we can write
    \begin{equation}\label{eq:B_1 computation}
        B_1
         =\E\Biggl[ \Biggl| \sum_{j:j\not= i} \left(\hat{u}_l^{\X_n^{(i)}}(x_j)-\hat{u}^{\X_n}_l(x_j)\right)\Bigl(\Delta^{(i)}_{n-1}u_l(x_j)-\E[\hat{\lambda}_l^\veps] u_l(x_j)\Bigr)\Biggr|^2\Biggr],
    \end{equation}
    for any fixed $i$, which we can thus assume to be, without the loss of generality, $i=n$. In order to analyze $B_1$ in the above form,  a crucial step in our proof is the development of sufficiently sharp estimates for how the spectrum of a graph Laplacian changes when removing a single point from the data set. The estimates that we present in this subsection are of interest in their own right.

 To state our stability estimates precisely, we begin by introducing some notation. Notice that the operator $\Delta_{n-1}^{(i)}$ acts on functions whose support is $\X_{n-1}:=\X_n\backslash\{x_i\}$. We will denote by $L^2(\X_{n-1})$ this space of functions. Given $k=1, \dots, n-1$, we denote by $f_k$ the $k$-th eigenvector of $\Delta_{n-1}^{(i)}$ (normalized according to the $L^2(\X_{n-1})$ inner product), and denote its corresponding eigenvalue as $\overline{\lambda}_k.$



In our first stability result, we quantify how eigenvectors may change by the removal of a data point.
\begin{proposition}[Stability of eigenvectors]
\label{lem:EigenVectorStability}
Let $f_l$ be the $l$-th eigenvector of $\Delta_{n-1}^{(i)}$ and let 
\begin{equation}
     \hat{v}^{\X_n}_l :=  \frac{\hat{u}^{\X_n}_l}{\lVert  \hat{u}^{\X_n}_l \rVert_{L^2(\X_{n-1})}  } \quad  
    \label{eqn:DefHatV}
\end{equation}
be the normalized version of $\hat{u}^{\X_n}_l$  with respect to the $\lVert \cdot \rVert_{L^2(\X_{n-1})}$-norm. Then, with probability at least $1-C(\eps^{-6m}+n)\exp(-cn\eps^{m+4})$,
\begin{equation}\label{eq:angle}
    \lVert f_l - \hat{v}_l^{\X_n} \rVert_{L^2(\X_{n-1})} \leq \frac{C_l}{n\eps^{\frac{m+2}{2}}}. 
\end{equation}

\end{proposition}

\begin{proof}

Let $S$ be the subspace of $L^2\left(\X_{n-1}\right)$ spanned by the eigenvector $f_l$ of $\Delta_{n-1}^{(i)}$. Let $P_S$ be the orthogonal projection onto $S$ and $P_S^{\perp}$ be the orthogonal projection onto the orthogonal complement of $S$ in $L^2(\X_{n-1})$. In what follows, whenever we consider the action of one of these operators on a function defined over $\X_n$, we will interpret it as the action of the operator on the restriction of the function to $\X_{n-1}$. 

From the spectral theorem and the fact that $\hat{v}_l^{\X_n}$ is an eigenvector of $\Delta_n^{\X_n}$, it follows
$$
P_S^{\perp} \Delta_{n}^{\X_n} \hat{v}_l^{\X_n}=\hat{\lambda}_l^{\X_n} P_S^{\perp} \hat{v}_l^{\X_n}= \hat{\lambda}_l^{\X_n} \sum_{j : j \neq l}\langle \hat{v}_l^{\X_n},f_j\rangle_{L^2\left(\X_{n-1}\right)} f_j.
$$
Likewise, we can write $\Delta_{n-1}^{(i)} \hat{v}_l^{\X_n}$ as
\[ P_S^{\perp} \Delta_{n-1}^{(i)} \hat{v}_l^{\X_n} = \sum_{j:j \neq l} \overline{\lambda}_j\langle \hat{v}_l^{\X_n}, f_j\rangle_{L^2\left(\X_{n-1}\right)} f_j.\]
Subtracting these two expressions, we deduce
\begin{align}
\begin{split}
    \min_{j:j\neq l} \left\{\left|\overline{\lambda}_j-\hat{\lambda}_l^{\X_n}\right|\right\}\left\|P_S^{\perp} \hat{v}_l^{\X_n} \right\|_{L^2\left(\X_{n-1}\right)} 
    & \leq\left\|P_S^{\perp}\left(\Delta_n^{\X_n} \hat{v}_l^{\X_n} -\Delta_{n-1}^{(i)} \hat{v}_l^{\X_n}\right)\right\|_{L^2\left(\X_{n-1}\right)}
    \\ & \leq \left\|\Delta_n^{\X_n} \hat{v}_l^{\X_n} -\Delta_{n-1}^{(i)} \hat{v}_l^{\X_n}\right\|_{L^2\left(\X_{n-1}\right)}.
\end{split}
\label{eq:projection inequality}
\end{align}
To bound the right-hand side, note that, thanks to \eqref{eq:Expansion}, 
\begin{align}\label{eq:L2 norm of different Delta}
    \begin{split}
       & \left\|  \Delta_n^{\X_n} \hat{v}_l^{\X_n} -\Delta_{n-1}^{(i)} \hat{v}_l^{\X_n}\right\|^2_{L^2\left(\X_{n-1}\right)}
        \\& \qquad  \leq \frac{1}{n-1 }\sum_{j=1}^{n-1}\left(\frac{1}{n\eps^{m+2}} \eta\left(\frac{|x_j-x_i|}{\eps}\right) \left(\hat{v}_l^{\X_n}(x_i)-\hat{v}_l^{\X_n}(x_j)\right)\right)^2
        \\&\qquad  \leq \frac{4}{n^3\eps^{2m+4}}\sum_{j=1}^{n-1} \eta\left(\frac{|x_j-x_i|}{\eps}\right)^2\left(\hat{v}_l^{\X_n}(x_i)-\hat{v}_l^{\X_n}(x_j)-(u_l(x_i)-u_l(x_j))\right)^2
        \\&\qquad \quad + \frac{4}{n^3\eps^{2m+4}}\sum_{j=1}^{n-1} \eta\left(\frac{|x_j-x_i|}{\eps}\right)^2\left(u_l(x_i)-u_l(x_j)\right)^2.
    \end{split}
\end{align}
Using the fact that $u_l$ is Lipschitz, it is straightforward to deduce that, with probability at least $1-C\eps^{-m}\exp(-cn\eps^{m}),$ the second term on the right hand side of the above expression is bounded by
\begin{equation*}
    \frac{4}{n^3\eps^{2m+4}}\sum_{j} \eta\left(\frac{|x_j-x_i|}{\eps}\right)^2(u_l(x_i)-u_l(x_j))^2\leq \frac{C_l}{n^2\eps^{m+2}}.
\end{equation*}

On the other hand, using \cite[Theorem 2.6]{calder2022lipschitz}, with probability at least $1-C(\eps^{-6m}+n)\exp(-cn\eps^{m+4})$ we have $\|\hat{u}_l^{\X_n}\|_{L^\infty(\X_n)}\leq C_l$, and, in particular, also
\begin{equation}\label{eq:control on  to n-1}
    1-\frac{C_l}{n}\leq \|\hat{u}_l^{\X_n}\|_{L^2(\X_{n-1})}\leq \frac{n}{n-1}.
\end{equation}
In the event where $\|\hat{u}_l^{\X_n}\|_{L^\infty(\X_n)}\leq C_l$ we thus have
\begin{equation}
\left| \hat{v}_{l}^{\X_n}(x_i) - \hat{v}_{l}^{\X_n}(x_j)  - \left(\hat{u}_{l}^{\X_n}(x_i) - \hat{u}_{l}^{\X_n}(x_j)\right)    \right| \leq \frac{C_l}{n}
\label{eq:almost lischitzHatv}
\end{equation}
for all $j$. In addition, using \cref{lem:spectral convergence in literature}, with probability at least $1-C(\eps^{-6m}+n)\exp(-cn\eps^{m+4}),$ we have 
\begin{equation}\label{eq:almost lischitz}
    \left|\hat{u}_l^{\X_n}(x_i)-\hat{u}_l^{\X_n}(x_j)-(u_l(x_i)-u_l(x_j))\right|\leq C_l \eps^2,
\end{equation}
for all $|x_i-x_j|\leq \eps$. The bottom line is that, with probability at least $1-C(\eps^{-6m}+n)\exp(-cn\eps^{m+4}),$ 
\begin{multline*}
    \frac{4}{n^3\eps^{2m+4}}\sum_{j} \eta\left(\frac{|x_j-x_i|}{\eps}\right)^2\left(\hat{u}_l^{\X_n}(x_i)-\hat{u}_l^{\X_n}(x_j)-(u_l(x_i)-u_l(x_j))\right)^2
    \leq \frac{C_l}{n^2\eps^{m}}.
\end{multline*}

Combining the above, we deduce that, with probability at least $1-C(\eps^{-6m}+n)\exp(-cn\eps^{m+4})$,
\begin{equation*}
    \left\|\Delta_n^{\X_n}\hat{v}_l^{\X_n}-\Delta_{n-1}^{(i)} \hat{v}_l^{\X_n}\right\|_{L^2\left(\X_{n-1}\right)}\leq \frac{C_l}{n\eps^{\frac{m+2}{2}}}.
\end{equation*}
\cref{assum:eigengap for l} and \cref{lem:spectral convergence in literature} imply that, with probability at least $1-C\eps^{-m}\exp(-c_l n\eps^{m})$, 
\begin{equation}\label{eq:eigengap assumption}
    \min_{j\neq l} \left\{\left|\overline{\lambda}_j-\hat{\lambda}_l^{\X_n}\right|\right\}\ge \frac{\gamma_l}{2} >0.
\end{equation}
From the above, \eqref{eq:projection inequality}, \eqref{eq:L2 norm of different Delta}, and \eqref{eq:eigengap assumption} we deduce
\begin{equation}\label{eq:Davis Kahan projection}
    \|P_S^\perp \hat{v}_l^{\X_n}\|_{L^2(\X_{n-1})}\leq \frac{C_l}{n\eps^{\frac{m+2}{2}}},
\end{equation}
 with probability at least $1-C(\eps^{-6m}+n)\exp(-cn\eps^{m+4})$.

Next, we use \eqref{eq:Davis Kahan projection} to control $\|f_l-\hat{v}_l^{\X_n}\|_{L^2(\X_{n-1})}^2$. First, without the loss of generality we may assume that the sign of $f_l$ has been chosen so that $f_l$ is aligned with $P_S\hat{v}_l^{\X_n}$. Precisely, we may assume that
\begin{equation*}
    \langle f_l, P_S\hat{v}_l^{\X_n}\rangle_{L^2(\X_{n-1})} = \| f_l\|_{L^2(\X_{n-1})} \| P_S \hat{v}_l^{\X_n}\|_{L^2(\X_{n-1})} = \| P_S \hat{v}_l^{\X_n}\|_{L^2(\X_{n-1})} .
\end{equation*}
A direct computation then reveals that, with probability at least $1-C(\eps^{-6m}+n)\exp(-cn\eps^{m+4})$,
\begin{align}
\begin{split}
    \|f_l-\hat{v}_l^{\X_n}\|_{L^2(\X_{n-1})}^2 &= \|f_l-P_S\hat{v}_l^{\X_n}\|_{L^2(\X_{n-1})}^2 + \|P_S^\perp\hat{v}_l^{\X_n}\|_{L^2(\X_{n-1})}^2
    \\ & = \|f_l\|_{L^2(\X_{n-1})}^2 + \|P_S\hat{v}_l^{\X_n}\|_{L^2(\X_{n-1})}^2 - 2 \| f_l\|_{L^2(\X_{n-1})} \| P_S \hat{v}_l^{\X_n}\|_{L^2(\X_{n-1})} \\ & \quad +\|P_S^\perp\hat{v}_l^{\X_n}\|_{L^2(\X_{n-1})}^2
    \\ & =1+\|P_S\hat{v}_l^{\X_n}\|_{L^2(\X_{n-1})}^2- 2\|P_S\hat{v}_l^{\X_n}\|_{L^2(\X_{n-1})} +\|P_S^\perp\hat{v}_l^{\X_n}\|_{L^2(\X_{n-1})}^2
    \\ &  = 1+ \lVert \hat{v}_l^{\X_n}  \rVert^2_{L^2(\X_{n-1})}- 2\sqrt{  \lVert \hat{v}_l^{\X_n}  \rVert^2_{L^2(\X_{n-1})}- \|P_S^\perp\hat{v}_l^{\X_n}\|_{L^2(\X_{n-1})}^2 } 
    \\ & \leq \frac{C_l}{n^2\eps^{m+2}},
    \end{split}
    \label{eqn:EignevcHighProbStability}
\end{align}
where the last inequality follows from \eqref{eq:Davis Kahan projection} and the elementary inequality $ (1-t)  \leq \sqrt{1- t} $ that holds for all $t \in [0,1]$.

\end{proof}

In our second stability result, we quantify how the eigenvalues of the graph Laplacian may change by the removal of a data point. As we discuss in \cref{rem:AboutStabiltyEigenvalues}, the bound that we obtain below is much sharper than the bound that could be obtained from more standard estimates.

\begin{proposition}[Eigenvalue stability]
\label{lem:EigenvalueStability}
Recall that $\overline{\lambda}_l$ is the $l$-th eigenvalue of $\Delta_{n-1}^{(i)}$. Then, with probability at least $1-Cn(\eps^{-6m}+n)\exp(-n\eps^{m+4})$,
\begin{multline}\label{eq:eigenvalue gap}
  \left|   (n-1) (\overline{\lambda}_l - \hat{\lambda}^{\X_n}_l)  - \hat{\lambda}^{\X_n}_l u_l^2(x_i)  + \frac{1}{n\eps^{m+2}}\sum_{k:k\not= i}\eta\left(\frac{|x_i-x_k|}{\eps}\right)  (u_l(x_i)-u_l(x_k))^2 \right| \\
    \leq C_l \left(\eps+\frac{1}{n\eps^{\frac{m+2}{2}}}\right).
\end{multline}
\end{proposition}

\begin{proof}
Consider $\hat{v}_l^{\X_n}$ as in \eqref{eqn:DefHatV} and recall that
\begin{equation*}
    \begin{split}
        \Delta_{n-1}^{(i)} \hat{v}_l^{\X_n} = \sum_{j=1}^{n-1} \overline{\lambda}_j \langle f_j , \hat{v}^{\X_n}_l\rangle_{L^2(\X_{n-1})} f_j, \quad \Delta_{n}^{\X_n} \hat{v}_l^{\X_n} = \hat{\lambda}^{\X_n}_l \hat{v}_l^{\X_n}.
    \end{split}
\end{equation*}
Subtracting the above equations, we obtain
\begin{equation}\label{eq:Delta-Delta}
    (\Delta_n^{\X_n}-\Delta_{n-1}^{(i)})\hat{v}_l^{\X_n} = \sum_{j=1}^{n-1} (\hat{\lambda}^{\X_n}_l-\overline{\lambda}_j) \langle f_l ,\hat{v}^{\X_n}_j\rangle_{L^2(\X_{n-1})} f_j.
\end{equation}
Taking the $L^2(\X_{n-1})$-inner product against $f_l$ on both sides of the above equation and recalling \eqref{eq:Expansion}, we obtain
\begin{align*}
   (n-1) (\hat{\lambda}^{\X_n}_l-\overline{\lambda}_l) &= \frac{ (n-1) \langle f_l,(\Delta_n^{\X_n}-\Delta_{n-1}^{(i)})\hat{v}^{\X_n}_l\rangle_{L^2(\X_{n-1})}}{\langle f_l ,\hat{v}^{\X_n}_l\rangle_{L^2(\X_{n-1})}} 
    \\ & =-\frac{\frac{1}{n\eps^{m+2}}\sum_{k:k\not= i}\eta(\frac{|x_i-x_k|}{\eps}) f_l(x_k) (\hat{v}^{\X_n}_l(x_i)-\hat{v}^{\X_n}_l(x_k))}{\langle f_l ,\hat{v}^{\X_n}_l\rangle_{L^2(\X_{n-1})}}.
\end{align*}
    In what follows, we extend $f_l$ to the whole $\X_n$ by defining 
    \[ f_l(x_i) := \frac{\sum_{j:j\not =i}\mathbbm{1}_{|x_i-x_j|\leq \eps} f_l(x_j)}{\sum_{j:j\not= i} \mathbbm{1}_{|x_i-x_j|\leq \eps}}.\] 
    Note that the denominator does not vanish with very high probability. Adding and subtracting terms, we deduce
    \begin{align}\label{eq:decomposition for Davis Kahan for eigenvalue}
    \begin{split}
    & \frac{\frac{1}{n\eps^{m+2}}  \sum_{k:k\not= i}\eta(\frac{|x_i-x_k|}{\eps}) f_l(x_k) (\hat{v}^{\X_n}_l(x_i)-\hat{v}^{\X_n}_l(x_k))}{\langle f_l ,\hat{v}^{\X_n}_l\rangle_{L^2(\X_{n-1})}}
    \\& =\frac{\frac{1}{n\eps^{m+2}}\sum_{k:k\not= i}\eta(\frac{|x_i-x_k|}{\eps}) f_l(x_i) (\hat{v}^{\X_n}_l(x_i)-\hat{v}^{\X_n}_l(x_k))}{\langle f_l ,\hat{v}^{\X_n}_l\rangle_{L^2(\X_{n-1})}}
    \\ & \quad +\frac{\frac{1}{n\eps^{m+2}}\sum_{k:k\not= i}\eta(\frac{|x_i-x_k|}{\eps}) (u_l(x_k)-u_l(x_i)) (u_l(x_i)-u_l(x_k))}{\langle f_l ,\hat{v}^{\X_n}_l\rangle_{L^2(\X_{n-1})}}
    \\ & \quad +\frac{\frac{1}{n\eps^{m+2}}\sum_{k:k\not= i}\eta(\frac{|x_i-x_k|}{\eps}) (f_l(x_k)-f_l(x_i)-u_l(x_k)+u_l(x_i)) (u_l(x_i)-u_l(x_k))}{\langle f_l ,\hat{v}^{\X_n}_l\rangle_{L^2(\X_{n-1})}}
    \\ & \quad +\frac{\frac{1}{n\eps^{m+2}}\sum_{k:k\not= i}\eta(\frac{|x_i-x_k|}{\eps}) (f_l(x_k)-f_l(x_i)) (\hat{v}^{\X_n}_l(x_i)-\hat{v}^{\X_n}_l(x_k)-u_l(x_i)+u_l(x_k))}{\langle f_l ,\hat{v}^{\X_n}_l\rangle_{L^2(\X_{n-1})}},
    \end{split}
\end{align}
 and our goal is now to analyze the right-hand side of the above expression term by term.

For the first term in \eqref{eq:decomposition for Davis Kahan for eigenvalue}, using the definition of $\Delta_{n}^{\X_n}$ we have, with probability at least $1-Cn(\eps^{-6m}+n)\exp(-n\eps^{m+4})$,
\begin{multline*}
     \frac{\frac{1}{n\eps^{m+2}}\sum_{k:k\not= i}\eta(\frac{|x_i-x_k|}{\eps}) f_l(x_i) (\hat{v}^{\X_n}_l(x_i)-\hat{v}^{\X_n}_l(x_k))}{\langle f_l ,\hat{v}^{\X_n}_l\rangle_{L^2(\X_{n-1})}}= \frac{f_l(x_i)\hat{\lambda}^{\X_n}_l \hat{v}^{\X_n}_l(x_i)}{\langle f_l ,\hat{v}^{\X_n}_l\rangle_{L^2(\X_{n-1})}}\\
     =O(\eps) + \frac{u_l(x_i)\hat{\lambda}^{\X_n}_l u_l(x_i)}{1-\langle f_l ,f_l-\hat{v}^{\X_n}_l\rangle_{L^2(\X_{n-1})}}= O(\eps)+ \left(1+O\left(\frac{1}{n\eps^{\frac{m+2}{2}}}\right)\right)\hat{\lambda}^{\X_n}_l u_l^2(x_i),
\end{multline*}
where in the second equality we used the $L^\infty$ control between $f_l$ and $u_l$ in \cref{lem:spectral convergence in literature} and in the last equality we used \eqref{eqn:EignevcHighProbStability}.

For the second term in \eqref{eq:decomposition for Davis Kahan for eigenvalue}, using \eqref{eqn:EignevcHighProbStability} we obtain
\begin{align*}
    & \frac{\frac{1}{n\eps^{m+2}} \sum_{k:k\not= i}\eta(\frac{|x_i-x_k|}{\eps})  (u_l(x_k)-u_l(x_i)) (u_l(x_i)-u_l(x_k))}{\langle f_l ,\hat{v}^{\X_n}_l\rangle_{L^2(\X_{n-1})}}
    \\& \quad  =-\frac{\frac{1}{n\eps^{m+2}}\sum_{k:k\not= i}\eta(\frac{|x_i-x_k|}{\eps}) (u_l(x_i)-u_l(x_k))^2}{\langle f_l ,\hat{v}^{\X_n}_l\rangle_{L^2(\X_{n-1})}}
    \\ & \quad =-\frac{\frac{1}{n\eps^{m+2}}\sum_{k:k\not= i}\eta(\frac{|x_i-x_k|}{\eps}) (u_l(x_i)-u_l(x_k))^2}{1-\langle f_l ,f_l-\hat{v}^{\X_n}_l\rangle_{L^2(\X_{n-1})}}
    \\ & \quad = - \left(1+O\left(\frac{1}{n\eps^{\frac{m+2}{2}}}\right)\right) \frac{1}{n\eps^{m+2}}\sum_{k:k\not= i}\eta\left(\frac{|x_i-x_k|}{\eps}\right)  (u_l(x_i)-u_l(x_k))^2.
\end{align*}

For the third term in \eqref{eq:decomposition for Davis Kahan for eigenvalue}, note that, similarly to \eqref{eq:almost lischitz}, using \cref{lem:spectral convergence in literature} we have, with probability at least $1-C(\eps^{-6m}+n)\exp(-cn\eps^{m+4}),$
\begin{equation}\label{eq:almost lipschitz2}
    |f_l(x_i)-f_l(x_k)-(u_l(x_i)-u_l(x_k))|\leq C_l \eps^2,
\end{equation}
 for $|x_k- x_i| \leq \veps$. Thanks to this and the regularity of $u_l$, we deduce that, with probability at least $1-C(\eps^{-6m}+n)\exp(-cn\eps^{m+4}),$ 
\begin{equation*}
    \left|\frac{\frac{1}{n\eps^{m+2}}\sum_{k:k\not= i}\eta(\frac{|x_i-x_k|}{\eps}) (f_l(x_k)-f_l(x_i)-u_l(x_k)+u_l(x_i)) (u_l(x_i)-u_l(x_k))}{\langle f_l ,\hat{v}^{\X_n}_l\rangle_{L^2(\X_{n-1})}}\right|
    \leq C_l\eps.
\end{equation*}

Finally, for the fourth term in \eqref{eq:decomposition for Davis Kahan for eigenvalue}, we may use \eqref{eq:almost lischitzHatv} and \eqref{eq:almost lischitz} to conclude that 
\begin{equation*}
    \left|\frac{\frac{1}{n\eps^{m+2}}\sum_{k:k\not= i}\eta(\frac{|x_i-x_k|}{\eps}) (f_l(x_k)-f_l(x_i)) (\hat{v}^{\X_n}_l(x_i)-\hat{v}^{\X_n}_l(x_k)-u_l(x_i)+u_l(x_k))}{\langle f_l ,\hat{v}^{\X_n}_l\rangle_{L^2(\X_{n-1})}}\right| \leq C_l \eps.
\end{equation*}

Combining the above computations yields the desired bound.
\end{proof}

\begin{remark}
    Note that in \cref{lem:EigenVectorStability} and \cref{lem:EigenvalueStability}, we crucially use  \cref{assum:eigengap for l}. 
\end{remark}

\begin{remark}
\label{rem:AboutStabiltyEigenvalues}    \cref{lem:EigenVectorStability} and \cref{lem:EigenvalueStability} are the desired versions of the Davis-Kahan theorems that we use to quantify the differences between eigenvalues and eigenvectors of graph Laplacians. Note that our leave-one-out analysis is much sharper than the more standard Davis-Khan theorem presented in \cite{calder2019improved}, which compares eigenpairs of graph Laplacian and Laplace-Beltrami operators. Indeed, in \cite{calder2019improved} it is shown that $\|\hat{u}_l-u_l\|_{L^2(\X_n)}\leq C_l\eps$ (in \cite{trillos2024} this was upgraded to quadratic convergence in the regime $\left(\frac{(\log n)^2}{n}\right)^{\frac{1}{m+4}}\ll\eps\ll 1$; see \cref{lem:spectral convergence in literature}) while for our leave-one-out analysis we obtain upper bounds that are $o\left(\frac{1}{\sqrt{n}}\right)$. These estimates are crucial for our central limit theorems to hold.
\end{remark}

\nc

\subsubsection{Control of \eqref{eq:B_1 computation} and proof of \eqref{eq:claim B_1}}
    
    With our stability results from the previous subsection in hand, we are ready to analyze the term \eqref{eq:B_1 computation}. First, observe that 
  \begin{align*}
    \begin{split}
      B_1  & \leq 4 \E\Biggl[ \Biggl|\sum_{j:j\not= i}  (f_l(x_j)-\hat{u}^{\X_n}_l(x_j))\Bigl(\Delta_{n-1}^{(i)} u_l(x_j)-\E[\hat{\lambda}_l^{\veps}] u_l(x_j)\Bigr)\Biggr|^2\Biggr]
      \\ & = 4 \E\Biggl[ \Biggl|\sum_{j:j\not= i}  (f_l(x_j)-\hat{u}^{\X_n}_l(x_j))\Bigl(\Delta_{n-1}^{(i)} u_l(x_j)-\E[\hat{\lambda}_l^{\veps}] u_l(x_j)\Bigr)\Biggr|^2 \one_{G_n} \Biggr]
      \\ &  \quad + 4 \E\Biggl[ \Biggl|\sum_{j:j\not= i}  (f_l(x_j)-\hat{u}^{\X_n}_l(x_j))\Bigl(\Delta_{n-1}^{(i)} u_l(x_j)-\E[\hat{\lambda}_l^{\veps}] u_l(x_j)\Bigr)\Biggr|^2 \one_{G_n^c}\Biggr],
             \end{split}
    \end{align*} 
where $G_n $ is the event where \eqref{eq:angle}, \eqref{eq:eigenvalue gap}, \eqref{eq:almost Lipschitz}, \eqref{eq:L_infty control}, and \eqref{eq:PointwsieConsistency} (with $t=\veps^{3/2}$) hold. Note that, thanks to Propositions \ref{lem:EigenVectorStability} and \ref{lem:EigenvalueStability}, as well as Lemma \ref{lem:spectral convergence in literature}, we have
\[ \mathbb{P}( G_n^c) \leq C_ln(\eps^{-6m}+n)\exp(-cn\eps^{m+4})   .\]
Now, it is straightforward to derive the deterministic bound  
\[ \Biggl|\sum_{j:j\not= i}  (f_l(x_j)-\hat{u}^{\X_n}_l(x_j))\Bigl(\Delta_{n-1}^{(i)} u_l(x_j)-\E[\hat{\lambda}_l^{\veps}] u_l(x_j)\Bigr)\Biggr|^2\leq C_l \frac{n^3}{\veps^{2m+2}},   \]
which can be used to bound the second term on the right-hand side of \eqref{eq:claim 1} by $\frac{C_l n^3}{ \veps^{2m +2}} \mathbb{P}(G_n^c)$. Therefore,
\begin{align}
\label{eq:claim 1}
\begin{split}
B_1  & \leq  4 \E\Biggl[ \Biggl|\sum_{j:j\not= i}  (f_l(x_j)-\hat{u}^{\X_n}_l(x_j))\Bigl(\Delta_{n-1}^{(i)} u_l(x_j)-\E[\hat{\lambda}_l^{\veps}] u_l(x_j)\Bigr)\Biggr|^2 \one_{G_n} \Biggr] 
\\ & \quad + C_l \frac{n^4}{\veps^{2m +2}} (\eps^{-6m}+n)\exp(-n\eps^{m+4}).
\end{split}
\end{align}

In what follows we will control the term 
\[\Biggl|\sum_{j:j\not= i}  (f_l(x_j)-\hat{u}^{\X_n}_l(x_j))\Bigl(\Delta_{n-1}^{(i)} u_l(x_j)-\E[\hat{\lambda}_l^{\veps}] u_l(x_j)\Bigr)\Biggr|^2\] 
in the event $G_n$. First, observe that
\begin{align}
\label{eqn: Aux234}
\begin{split}
& \Biggl|\sum_{j:j\not= i}  (f_l(x_j)-\hat{u}^{\X_n}_l(x_j))\Bigl(\Delta_{n-1}^{(i)} u_l(x_j)-\E[\hat{\lambda}_l^{\veps}] u_l(x_j)\Bigr)\Biggr|^2
\\ & \quad \leq  2\Biggl|\sum_{j:j\not= i}  (f_l(x_j)-\hat{v}^{\X_n}_l(x_j))\Bigl(\Delta_{n-1}^{(i)} u_l(x_j)-\E[\hat{\lambda}_l^{\veps}] u_l(x_j)\Bigr)\Biggr|^2
\\ & \qquad  + 2\Biggl|\sum_{j:j\not= i}  (\hat{v}^{\X_n}_l(x_j)-\hat{u}^{\X_n}_l(x_j))\Bigl(\Delta_{n-1}^{(i)} u_l(x_j)-\E[\hat{\lambda}_l^{\veps}] u_l(x_j)\Bigr)\Biggr|^2,
\end{split}
\end{align}
where $\hat{v}_l^{\X_n}$ is the rescaled version of $\hat{u}_l^{\X_n}$ from \eqref{eqn:DefHatV}. To bound the second term, observe that, since we are in the event $G_n$, we have  $\lVert \hat{v}_n^{\X_n} -  \hat{u}_n^{\X_n}  \rVert_{L^\infty(\X_{n})} \leq \frac{C_l}{n}$ and thus 
\begin{align*}
\Biggl|\sum_{j:j\not= i}  (\hat{v}^{\X_n}_l(x_j)-\hat{u}^{\X_n}_l(x_j))\Bigl(\Delta_{n-1}^{(i)} u_l(x_j)-\E[\hat{\lambda}_l^{\veps}] u_l(x_j)\Bigr)\Biggr|^2 &  \leq  C_l  \lVert \Delta_{n-1}^{(i)} u_l  - \E[\hat{\lambda}_l^\veps] u_l \rVert_{L^2(\X_n)}^2   
\\ & \leq  C_l\veps^2 + c_n'',   
\end{align*}
for some $c_n'' \rightarrow 0$. The last inequality follows from \eqref{eq:PointwsieConsistency} and the fact that $|\lambda_l - \E[\lambda_l^\veps]| \rightarrow 0$.

Next, we seek to bound the first term on the right hand side of \eqref{eqn: Aux234}. We
use the fact that $f_l$ is an eigenvector of $\Delta_{n-1}^{(i)}$, that $\Delta_{n-1}^{(i)}$ is self-adjoint with respect to the $L^2(\X_{n-1})$ inner product, and  \eqref{eq:Expansion}, to deduce
    \begin{equation}\label{eq:Davis Kahan}
        \begin{split}
            &\sum_{j:j\not= i} (f_l(x_j)-\hat{v}^{\X_n}_l(x_j))\Bigl(\Delta^{(i)}_{n-1}u_l(x_j)-\E[\hat{\lambda}_l^\veps] u_l(x_j)\Bigr)\\
            =&(n-1) (\overline{\lambda}_l-\E[\hat{\lambda}_l^\veps])\langle f_l,u_l \rangle_{L^2(\X_{n-1})} + (n-1)\E[\hat{\lambda}_l^\veps] \langle\hat{v}^{\X_n}_l,u_l \rangle_{L^2(\X_{n-1})}  \\
            &- (n-1)\langle\hat{v}^{\X_n}_l,\Delta_{n-1}^{(i)} u_l \rangle_{L^2(\X_{n-1})} \\
            =&(n-1) (\overline{\lambda}_l-\E[\hat{\lambda}_l^\veps])\langle f_l,u_l \rangle_{L^2(\X_{n-1})} + (n-1)\E[\hat{\lambda}_l^\veps] \langle\hat{v}^{\X_n}_l,u_l \rangle_{L^2(\X_{n-1})}  \\
            &- (n-1)\langle \Delta_{n-1}^{(i)}\hat{v}^{\X_n}_l, u_l \rangle_{L^2(\X_{n-1})} \\
            =&(n-1) (\overline{\lambda}_l-\E[\hat{\lambda}_l^\veps])\langle f_l,u_l \rangle_{L^2(\X_{n-1})} + (n-1)\E[\hat{\lambda}_l^\veps] \langle\hat{v}^{\X_n}_l,u_l \rangle_{L^2(\X_{n-1})}  \\
            &- (n-1)\langle \Delta_{n}^{\X_n}\hat{v}^{\X_n}_l, u_l \rangle_{L^2(\X_{n-1})} \\
            &+\sum_{j:j\not= i} u_l(x_j) \frac{1}{n\eps^{m+2}} \eta\Big(\frac{|x_i-x_j|}{\eps}\Big) \bigl(\hat{v}^{\X_n}_l(x_j)-\hat{v}^{\X_n}_l(x_i)\bigr)\\
             =&(n-1) (\overline{\lambda}_l-\E[\hat{\lambda}_l^\veps])\langle f_l,u_l \rangle_{L^2(\X_{n-1})} + (n-1)\E[\hat{\lambda}_l^\veps] \langle\hat{v}^{\X_n}_l,u_l \rangle_{L^2(\X_{n-1})}  \\
            &- (n-1)\hat{\lambda}^{\X_n}_l\langle\hat{v}^{\X_n}_l, u_l \rangle_{L^2(\X_{n-1})} \\
            &+\sum_{j:j\not= i} u_l(x_j) \frac{1}{n\eps^{m+2}} \eta\Big(\frac{|x_i-x_j|}{\eps}\Big) \bigl(\hat{v}^{\X_n}_l(x_j)-\hat{v}^{\X_n}_l(x_i)\bigr)\\
            =&(n-1) (\overline{\lambda}_l-\E[\hat{\lambda}_l^\veps])\langle f_l,u_l \rangle_{L^2(\X_{n-1})} + (n-1)(\E[\hat{\lambda}_l^\veps]-\hat{\lambda}^{\X_n}_l) \langle\hat{v}^{\X_n}_l,u_l \rangle_{L^2(\X_{n-1})} \\
            &+\sum_{j:j\not= i} u_l(x_i) \frac{1}{n\eps^{m+2}} \eta\Big(\frac{|x_i-x_j|}{\eps}\Big) \bigl(\hat{v}^{\X_n}_l(x_j)-\hat{v}^{\X_n}_l(x_i)\bigr)\\
            &-\sum_{j:j\not= i} (u_l(x_i)-u_l(x_j)) \frac{1}{n\eps^{m+2}} \eta\Big(\frac{|x_i-x_j|}{\eps}\Big) \bigl(\hat{v}^{\X_n}_l(x_j)-\hat{v}^{\X_n}_l(x_i)\bigr)\\
            =&(n-1) (\overline{\lambda}_l-\E[\hat{\lambda}_l^\veps])\langle f_l,u_l \rangle_{L^2(\X_{n-1})} + (n-1)(\E[\hat{\lambda}_l^\veps]-\hat{\lambda}^{\X_n}_l) \langle\hat{v}^{\X_n}_l,u_l \rangle_{L^2(\X_{n-1})} \\
            &-\hat{\lambda}_l^{\X_n}u_l(x_i)    \hat{v}_l^{\X_n}(x_i) \\
            &-\sum_{j:j\not= i} (u_l(x_i)-u_l(x_j)) \frac{1}{n\eps^{m+2}} \eta\Big(\frac{|x_i-x_j|}{\eps}\Big) \bigl(\hat{v}^{\X_n}_l(x_j)-\hat{v}^{\X_n}_l(x_i)\bigr),        
            \end{split}
    \end{equation}
    where the last equality follows from
    \begin{equation*}
        \sum_{j:j\not= i} \frac{1}{n\eps^{m+2}} \eta\Big(\frac{|x_i-x_j|}{\eps}\Big) (\hat{v}^{\X_n}_l(x_i)-\hat{v}^{\X_n}_l(x_j)) = \Delta_n^{\X_n} \hat{v}_l^{\X_n}(x_i)  =\hat{\lambda}_l^{\X_n}\hat{v}_l^{\X_n}(x_i).
    \end{equation*}
   We can rewrite the first two terms on the right hand side of \eqref{eq:Davis Kahan} as
    \begin{equation*}
        \begin{split}
            &(n-1) (\overline{\lambda}_l-\E[\hat{\lambda}_l^\veps])\langle f_l,u_l \rangle_{L^2(\X_{n-1})} + (n-1)(\E[\hat{\lambda}_l^\veps]-\hat{\lambda}^{\X_n}_l) \langle\hat{v}^{\X_n}_l,u_l \rangle_{L^2(\X_{n-1})}\\
            =& (n-1) (\overline{\lambda}_l-\E[\hat{\lambda}_l^\veps])\langle f_l,u_l \rangle_{L^2(\X_{n-1})} -(n-1) (\overline{\lambda}_l-\E[\hat{\lambda}_l^\veps])\langle\hat{v}_l^{\X_n},u_l \rangle_{L^2(\X_{n-1})}\\
            & + (n-1)(\E[\hat{\lambda}_l^\veps]-\hat{\lambda}^{\X_n}_l) \langle\hat{v}^{\X_n}_l,u_l \rangle_{L^2(\X_{n-1})}+(n-1) (\overline{\lambda}_l-\E[\hat{\lambda}_l^\veps])\langle\hat{v}_l^{\X_n},u_l \rangle_{L^2(\X_{n-1})}\\
            =& (n-1) (\overline{\lambda}_l-\E[\hat{\lambda}_l^\veps])\langle f_l-\hat{v}_l^{\X_n},u_l \rangle_{L^2(\X_{n-1})} +(n-1) (\overline{\lambda}_l-\hat{\lambda}^{\X_n}_l) \langle\hat{v}^{\X_n}_l,u_l \rangle_{L^2(\X_{n-1})}\\
            = &(n-1) (\hat{\lambda}_l^{\X_n}-\E[\hat{\lambda}_l^\veps])\langle f_l-\hat{v}_l^{\X_n},u_l \rangle_{L^2(\X_{n-1})} +(n-1) (\overline{\lambda}_l-\hat{\lambda}^{\X_n}_l) \langle\hat{v}^{\X_n}_l,u_l \rangle_{L^2(\X_{n-1})}\\
            & + (n-1) (\overline{\lambda}_l-\hat{\lambda}_l^{\X_n})\langle f_l-\hat{v}_l^{\X_n},u_l \rangle_{L^2(\X_{n-1})}.
        \end{split}
    \end{equation*}
    Using \eqref{eq:angle}, we obtain
    \begin{align*}
    \begin{split}
       & \left|(n-1) (\hat{\lambda}_l^{\X_n}-\E[\hat{\lambda}_l^{\veps}])\langle f_l-\hat{v}_l^{\X_n},u_l \rangle_{L^2(\X_{n-1})}\right|
        \\ & \quad  \leq (n-1) |\hat{\lambda}_l^{\X_n}-\E[\hat{\lambda}_l^{\veps}]|\cdot \| f_l-\hat{v}_l^{\X_n}\|_{L^2(\X_{n-1})} \|u_l\|_{L^2(\X_{n-1})} 
        \\ & \quad  \leq \frac{C_l}{n\eps^{\frac{m+2}{2}}} n|\hat{\lambda}_l^{\X_n}-\E[\hat{\lambda}_l^{\veps}]|
        \\& \quad = \frac{C_l}{\sqrt{n\eps^{m+2}}} \sqrt{n}|\hat{\lambda}_l^{\X_n}-\E[\hat{\lambda}_l^{\veps}]|,
        \end{split}
    \end{align*}
   while
    \begin{equation*}
        \begin{split}
            \left| (n-1) (\overline{\lambda}_l-\hat{\lambda}_l^{\X_n})\langle f_l-\hat{v}_l^{\X_n},u_l \rangle_{L^2(\X_{n-1})}\right|\leq C_l\|f_l-\hat{v}_l^{\X_n}\|_{L^2(\X_{n-1})}\leq \frac{C_l}{n \veps^{\frac{m+2}{2}}},
        \end{split}
    \end{equation*}
     thanks to \eqref{eq:eigenvalue gap} (since, in particular, \eqref{eq:eigenvalue gap} implies $(n-1) |\hat{\lambda}^{\X_n}_l-\overline{\lambda}_l| \leq C_l.$)
    
  On the other hand, thanks to \cref{lem:spectral convergence in literature}, \cref{lem:EigenvalueStability}, with vey high probaility we have
    \begin{multline*}
        \Biggl|(n-1) (\overline{\lambda}_l-\hat{\lambda}^{\X_n}_l) \langle\hat{v}^{\X_n}_l,u_l \rangle_{L^2(\X_{n-1})}
        -\hat{\lambda}_l^{\X_n}u_l(x_i)    \hat{v}_l^{\X_n}(x_i) \\
            -\sum_{j:j\not= i} (u_l(x_i)-u_l(x_j)) \frac{1}{n\eps^{m+2}} \eta\Big(\frac{|x_i-x_j|}{\eps}\Big) \bigl(\hat{v}^{\X_n}_l(x_j)-\hat{v}^{\X_n}_l(x_i)\bigr)\Biggr|\\
            =\Biggl|(n-1) (\overline{\lambda}_l-\hat{\lambda}^{\X_n}_l) \left(\|u_l\|^2_{L^2(\X_{n-1})}+\langle\hat{v}_l^{\X_n}-u_l,u_l\rangle_{L^2(\X_{n-1})}\right)
        -\hat{\lambda}_l^{\X_n}u_l(x_i)    \hat{v}_l^{\X_n}(x_i) \\
            -\sum_{j:j\not= i} (u_l(x_i)-u_l(x_j)) \frac{1}{n\eps^{m+2}} \eta\Big(\frac{|x_i-x_j|}{\eps}\Big) \bigl(\hat{v}^{\X_n}_l(x_j)-\hat{v}^{\X_n}_l(x_i)\bigr)\Biggr|\\
            =\Biggl|(n-1) (\overline{\lambda}_l-\hat{\lambda}^{\X_n}_l) \left(1+O(\eps)\right)
        -\hat{\lambda}_l^{\X_n}u_l(x_i)     (\hat{v}_l^{\X_n}(x_i)-u_l(x_i))-\hat{\lambda}_l^{\X_n}u^2_l(x_i) \\
            -\sum_{j:j\not= i} (u_l(x_i)-u_l(x_j)) \frac{1}{n\eps^{m+2}} \eta\Big(\frac{|x_i-x_j|}{\eps}\Big) (\hat{v}^{\X_n}_l(x_j)-\hat{v}^{\X_n}_l(x_i)-(u_l(x_j)-u_l(x_i)))\\
         +\sum_{j:j\not= i} (u_l(x_i)-u_l(x_j))^2 \frac{1}{n\eps^{m+2}} \eta\Big(\frac{|x_i-x_j|}{\eps}\Big)\Biggr| \\
           =\Biggl|(n-1) (\overline{\lambda}_l-\hat{\lambda}^{\X_n}_l) \left(1+O(\eps)\right)
        +O(\eps)-\hat{\lambda}_l^{\X_n}u^2_l(x_i) \\
            +O(\eps)  +  \sum_{j:j\not= i} (u_l(x_i)-u_l(x_j))^2 \frac{1}{n\eps^{m+2}} \eta\Big(\frac{|x_i-x_j|}{\eps}\Big)\Biggr| \leq C_l \eps.
    \end{multline*}
Returning to \eqref{eq:Davis Kahan}, the above computations imply that in the event $G_n$ we have
\begin{align*}
\begin{split}
  \left| \sum_{j:j\not= i} (f_l(x_j)-\hat{v}^{\X_n}_l(x_j))\Bigl(\Delta^{(i)}_{n-1}u_l(x_j)-\E[\hat{\lambda}_l^\veps] u_l(x_j)\Bigr)\right| & \leq \frac{C_l}{\sqrt{n \veps^{m+2}}} \sqrt{n} | \hat{\lambda}_l^{\X_n} - \E[\hat{\lambda}_l^\veps]| 
  \\ & \quad + C_l \left( \veps + \frac{1}{n\veps^{\frac{m+2}{2}}} \right). 
\end{split}
\end{align*}
In turn, plugging this in \eqref{eqn: Aux234} and then in \eqref{eq:claim 1}, we obtain
\[ B_1 \leq  \frac{C_la_n}{n \veps^{m+2}} + c_n,\]
where 
\[ c_n :=  C_l \left(  \veps + \frac{1}{n \veps^{\frac{m+2}{2}}}\right)^2 + c_n''  + C_l \frac{n^4}{\veps^{2m +2}} (\eps^{-6m}+n)\exp(-cn\eps^{m+4}).  \]
Note that, by our assumptions on $\veps$, we indeed have $c_n \rightarrow 0$ as $n \rightarrow \infty$. This finishes the proof of \eqref{eq:claim B_1} and hence also of \cref{lem:b_n control}.


\end{proof}

\subsubsection{Proof of (\ref{enu:1.a}) in Theorem \ref{thm}} 
Recall that \eqref{eq:decompose of variance} (which follows from \cref{lem:spectral convergence in literature}) gives us 
\[ c a_n \leq 2 \mathrm{Var}(\BB_n) + 2 b_n + c_n'.  \]
Combining the above with \eqref{eq:b_n upper bound by a_n}, we deduce
\begin{equation}
\label{eqn:Aux_2_Thm}
   c a_n \leq 2 \mathrm{Var}(\BB_n) + \frac{2C}{n \eps^{m+2}} a_n + c_n + c_n',
\end{equation}
for $c_n \rightarrow 0$ as $n \rightarrow \infty$. Now, since for $\eps$ as in \cref{thm1} we have $\frac{C}{n\eps^{m+2}}\to 0$, it follows that for all large enough $n$ we have $\frac{2C}{n \eps^{m+2}} a_n \leq \frac{c}{2} a_n$. Plugging this back in \eqref{eqn:Aux_2_Thm} and reorganizing we obtain
\[ \frac{c}{2} a_n \leq 2 \mathrm{Var}(\BB_n) + 2c_n + c_n' \]
for all large enough $n$. Hence,
\[ \limsup_{n \rightarrow \infty} a_n \leq C \sigma_l^2, \]
where we used (\ref{enu:3}) in \cref{thm}. Returning to \eqref{eq:b_n upper bound by a_n}, we can now deduce 
\[  \limsup_{n \rightarrow \infty} b_n \leq \limsup_{n \rightarrow \infty} \left( \frac{C}{n \eps^{m+2}} a_n + c_n   \right) =0. \]
That is, 
\[ \lim_{n \rightarrow \infty} \mathrm{Var}(\EE_n) =0.\]
Since $\E[\EE_n]=0$, it follows that $\EE_n \pto 0$ as $n \rightarrow \infty$, concluding in this way the proof.




    \nc

\subsection{Proof of \cref{thm1}.}\label{sec:main proof}

\subsection{CLT for Multiple Eigenvalues}\label{sec:proof for multi eigenvalue CLT}
In this section, we present the proof of \cref{thm:multi normal distribution}, which is based on a quite straightforward adaptation of the proof of \cref{thm1}. 

First, note that it suffices to prove that, for fixed scalars $a_1, \dots, a_l \in \R$,
\[ \sqrt{n}\sum_{k=1}^l a_k ( \hat{\lambda}_k^\veps - \E[ \hat{\lambda}_k^\veps ]  ) \dto \mathcal{N}\left (0, \sum_{j,k} \overline{\sigma}_{j,k} a_ja_k  \right).  \] 
To see this, we may first follow the same steps as in the proof of Theorem \ref{thm1} to conclude that it suffices to prove that
\[  \sqrt{n} \sum_{k=1}^l a_k (  (\hat{\lambda}_k^\veps  - \E[\hat{\lambda}_k^\veps] ) \langle u_k , \hat{u}_k \rangle_{L^2(\X_n)} - \E[(\hat{\lambda}_k^\veps  - \E[\hat{\lambda}_k^\veps] ) \langle u_k , \hat{u}_k \rangle_{L^2(\X_n)}]   )  \dto \mathcal{N} \left(0, \sum_{j,k} \overline{\sigma}_{j,k} a_ja_k  \right). \]
Using \eqref{eq:lambda-hatlambda}, the left-hand side in the above can be written as the sum of a bulk term:
\[ \sum_{k=1}^l  a_k \sqrt{n}\left(    \langle u_k , \Delta_n u_k - \E[\hat{\lambda}_k^\veps] u_k \rangle_{L^2(\X_n)}  - \E[ \langle u_k , \Delta_n u_k - \E[\hat{\lambda}_k^\veps] u_k \rangle_{L^2(\X_n)}  ]  \right)\]
and an error term:
\[ \sum_{k=1}^l  a_k \sqrt{n}\left(    \langle  \hat{u}_k - u_k , \Delta_n u_k - \E[\hat{\lambda}_k^\veps] u_k \rangle_{L^2(\X_n)}  - \E[ \langle \hat{u}_k -  u_k , \Delta_n u_k - \E[\hat{\lambda}_k^\veps] u_k \rangle_{L^2(\X_n)}  ]  \right).\]
However, by \eqref{enu:1.a} in \cref{thm}, each of the summands in the above error term goes to zero in probability, and thus the full error term vanishes asymptotically.  By Slutsky's theorem, it would thus suffice to analyze the bulk term, but this is done in a completely analogous way as in the proof of \eqref{enu:3} in \cref{thm}.

\subsection{Consistent Estimators for Asymptotic Variance}
\label{sec:ConsistentEstimatorVariance}

In this section we prove Theorem \ref{thm:ConsistentEstimator}. First, we establish the following consistency result.

\begin{lemma}
\label{lem:EstimatorNormGradient}
Under Assumptions \ref{assump:DataGenerating}, \ref{assump:Kernel}, \ref{assum:eigengap for l},  and provided $\left(\frac{(\log n)^2}{n}\right)^{\frac{1}{m+4}}\ll\eps\ll 1$, we have
\[\max_{i=1, \dots, n} \left| \widehat{G}_{l,n}(x_i) - |\nabla u_l (x_i)|^2 \rho(x_i) \right| \asto 0. \]
\end{lemma}
\begin{proof}
Thanks to Lemma \ref{lem:spectral convergence in literature} (specifically \eqref{eq:almost Lipschitz}), with probability at least $1-C(\eps^{-6m}+n)\exp(-cn\eps^{m+4})$ we have
\[ |\hat{u}_l(x_i) - \hat{u}_l(x_j) - ( u_l(x_i) - u_l(x_j))| \leq C \veps(|x_i-x_j| + \veps), \quad \forall i,j.    \]
This, together with the assumptions on $\eta$ and \eqref{e.comparegeodesic}, implies that, with probability at least $1-C(\eps^{-6m}+n)\exp(-cn\eps^{m+4})$, we have
\[  \max_{i=1, \dots, n } | \widehat{G}_{l,n}(x_i) - \widetilde{G}_{l,n}(x_i)   |  \leq C \veps,   \]
where
\[     \widetilde{G}_{l,n}(x_i):= \frac{1}{2n\veps^{m+2}} \sum_{j=1}^n | {u}_l (x_j) - {u}_l(x_i)|^2 \eta\left(\frac{d(x_i,x_j)}{\veps} \right), \quad i=1, \dots, n. \]

Now, a standard application of Bernstein's inequality combined with a union bound reveals that, with probability at least $1- Cn \exp(-cnt^2\veps^m)$,
\[  \left| \widetilde{G}_{l,n}(x_i)  - \frac{1}{2\veps^{m+2}} \int_{\M} | u_l(x_i) - u_l(y)|^2\eta\left(\frac{d(x_i, y)}{\veps} \right) \rho(y) \dd y    \right| \leq t, \]
for all $i=1, \dots, n$. 

Finally, note that, for any given $x \in \M$, we have
\begin{align*}
\frac{1}{2\veps^{m+2}} & \int_{\M} | u_l(x) - u_l(y)|^2\eta\left(\frac{d(x, y)}{\veps} \right) \rho(y) \dd y 
\\ & =  \frac{1}{2\veps^{2}} \int_{B_1(0) \subseteq T_x\M } | u_l(x) - u_l(\exp_x(\veps v))|^2 \eta(|v|) \rho(\exp_x(\veps v)) J_x(\veps v) \dd v
\\ & =  \frac{1}{2} \rho(x) \int_{B_1(0) \subseteq T_x\M } | \langle \nabla u_l(x), v \rangle   |^2 \eta(|v|)\dd v    + O(\veps)
\\ & = |\nabla u_l(x)|^2 \rho(x) + O(\veps),
\end{align*}
where to go from the second to the third line we used a Taylor expansion for $\rho$, $u_l$ and $J_x$ (see \eqref{e.Jactaylor}). 

Putting together the above estimates, we conclude that
\[ \max_{i=1, \dots, n} \left| \widehat{G}_{l,n}(x_i) - |\nabla u_l (x_i)|^2 \rho(x_i) \right| \leq C(t+ \veps)    \]
with probability at least $1-C(\eps^{-6m}+n)\exp(-cn\eps^{m+4}) - Cn \exp(-cnt^2\veps^m).$ It remains to use the assumptions on $\veps$ to combine the above concentration bound with the Borel-Cantelli lemma and obtain the desired consistency.
\end{proof}

With the above lemma in hand, we are now ready to prove Theorem \ref{thm:ConsistentEstimator}.

\begin{proof}[Proof of Theorem \ref{thm:ConsistentEstimator}]
Thanks to Lemma \ref{lem:spectral convergence in literature} (specifically \eqref{eq:L_infty control}), Lemma \ref{lem:EstimatorNormGradient}, and the Borel-Cantelli lemma we have
\[ \left| \widehat{\sigma}_l^2 -  \frac{1}{n}\sum_{i=1}^n\Bigl({\lambda}_l {u}_l(x_i)^2+{\lambda}_l - 2  |\nabla u_l (x_i)|^2 \rho(x_i) \Bigr)^2 \right| \asto 0.  \]
On the other hand, by the law of large numbers we have
\[   \left|   \frac{1}{n}\sum_{i=1}^n\Bigl({\lambda}_l{u}_l(x_i)^2+{\lambda}_l - 2  |\nabla u_l (x_i)|^2 \rho(x_i) \Bigr)^2  -{\sigma}_l^2  \right| \asto 0.    \]
This implies $\widehat{\sigma}_l^2 \asto \sigma_l^2$. The same argument can be used to prove the consistency of $\widehat{\Sigma}$.
\end{proof}

\subsection{On Centralization Around $\lambda_l$}\label{sec:bias}
Here we present the (very short) proof of Theorem \ref{lem:spectral convergence bias control} and provide an additional remark on centralization of $\hat{\lambda}_l^\veps$ around $\lambda_l$. We also provide the proof of Corollary \ref{cor:BiasEstimate}, which provides an upper bound on the bias $|\lambda_l - \E[\hat{\lambda}_l^\veps]|$ when $\veps$ scales appropriately.  

\begin{proof}[Proof of Theorem \ref{lem:spectral convergence bias control}]
Using \cite[Proposition 3.8]{trillos2024} with $g= u_l$, and noticing that due to the regularity of $u_l$ we have $[u_l]_{1, \veps} \leq C_l$, it follows that
\begin{equation*}
    \left|\E\left[\langle u_l, (\Delta_\rho-\Delta_n) u_l \rangle_{L^2(\X_n)}\right]\right|\leq C_l\eps^2,
\end{equation*}   
proving in this way \eqref{eqn:UpperBias1}. 
\end{proof}

\begin{remark}
From the analysis in the proof of \cite[Proposition 3.8]{trillos2024}, it follows that, unless there are additional cancellations in the computation in the proof of that proposition (which won't occur generically), the leading order term of 
\[\E\left[\langle u_l, (\Delta_\rho-\Delta_n) u_l \rangle_{L^2(\X_n)}\right]\] is indeed $O(\veps^2)$ and not something smaller. 
\label{rem:LowerBoundBias}
\end{remark}

\begin{proof}[Proof of Corollary \ref{cor:BiasEstimate}]
Using \eqref{eqn:BiasDecomp}, we deduce 
\begin{align*}
\left| (\E[\hat{\lambda}_l^\veps] -\lambda_l)  - \E \left[  \langle \Delta_n u_l - \Delta_\rho u_l , u_l  \rangle_{L^2(\X_n)}  \right]\right| & \leq 
\E \left[ |\hat{\lambda}_l^\veps - \lambda_l| | \langle u_l- \hat{u}_l, u_l  \rangle_{L^2(\X_n)}|  \right]
\\& \quad  + \E\left[ |\langle \Delta_n u_l - \Delta_\rho u_l , \hat{u}_l - u_l  \rangle_{L^2(\X_n)}|  \right].
\end{align*}
Using Lemma \ref{lem:spectral convergence in literature} and the assumptions on $\veps$ it is straightforward to verify that the right hand side of the above expression is $o(\veps^2)$. Combining with \eqref{eqn:UpperBias1}, we deduce the desired  upper bound on the bias.

\end{proof}

\nc

\section{Geometric and statistical interpretations of asymptotic variance}\label{sec:simulation-geometry}

In this section, we explore the different interpretations for the asymptotic variance $\sigma_l^2$ that we announced in the introduction. In the first two subsections, which are largely informal, we present a geometric characterization for $\sigma_l^2$. This description, in turn, suggests why a Cramer-Rao lower bound like the one enunciated in Theorem \ref{thm:CramerRao} should hold. We present the rigurous proof of Theorem \ref{thm:CramerRao} at the end of section \ref{sec:CramerRaoAnalysis}. 

\subsection{Background on Fisher-Rao Geometry}
\label{sec:FisherRao}

Consider the space of Borel probability measures $\mathcal{P}(\M)$ over $\M$ endowed with the following Fisher-Rao metric.

\begin{definition}
Given $\mu$ and $\nu$ two Borel probability measures over $\M$, we define
\[ \mathrm{FR}^2(\mu, \nu):= \inf_{t \in [0,1] \mapsto (\mu_t, \xi_t)} \int_0^1  \int _{\M} \xi_t^2 d\mu_t(x) dt, \]
where the $\inf$ ranges over all solutions $ t\mapsto (\mu_t, \xi_t)$ to the equation
\begin{equation}
 \partial_t \mu_t = \xi_t \mu_t, 
 \label{eqn:NewCont}
 \end{equation}
where $\xi_t$ is a scalar function in $L^2(\mu_t)$ such that $\int \xi_t d \mu_t=0 $ for all $t$, and $\mu_0=\mu$ and $\mu_1 = \nu$. Equation \eqref{eqn:NewCont} must be interpreted in the distributional sense. That is, for every test function $\phi$ we must have
\[ \frac{d}{dt}\int_\M \phi(x) d \mu_t(x) =\int_\M \xi_t \phi(x) d\mu_t(x). \]
\end{definition}

It is not difficult to show that $\mathrm{FR}$ defined in this way induces a metric in the space of Borel probability measures over $\M$. Moreover, from the definition, we can see that $\mathrm{FR}$ has an associated Riemannian-like structure. Indeed, for a given $\mu \in \mathcal{P}(\M)$ we can formally define its \textit{tangent plane} as
\[\mathcal{T}_\mu \mathcal{P}(\M) := \{ \xi \in L^2(\mu) \: : \: \int_\M  \xi d\mu =0 \}, \] 
which we endow with the inner product
\[\langle \xi , \beta \rangle_\mu := \int_\M \xi \beta d\mu, \quad \xi, \beta \in \mathcal{T}_\mu\mathcal{P}(\M).\]

\subsubsection{Gradient flows w.r.t.\ Fisher-Rao geometry}

Thanks to the formal Riemannian structure of the Fisher-Rao geometry mentioned above, we can introduce a formal notion of gradient for an energy $E : \mathcal{P}(\M) \rightarrow \R \cup \{ \infty \}$ at a given $\mu$. In particular, the gradient of $E$ at $\mu$ is the element $\nabla_{\mathrm{FR}} E(\mu) $  in $ \mathcal{T}_\mu \mathcal{P}(\M)$ for which the chain rule 
\[ \frac{d}{dt} E(\mu_t) = \langle \nabla_{\mathrm{FR}} E(\mu_t), \xi_t \rangle_{\mu_t}   \]
holds for all trajectories
\[ \partial_t \mu_t = \xi_t \mu_t.\]

\begin{definition}[First variation of an energy; see Definition 7.12 in \cite{santambrogio2015optimal}]
Given $E : \mathcal{P}(\M) \rightarrow \R \cup \{ \infty \}$, we say that $\frac{\delta E}{\delta \mu} $ is the first variation of $E$ at $\mu$ if 
\[ \frac{d}{dt}\mid_{t=0} E(\mu + t \zeta) = \int_\M \frac{\delta E}{\delta \mu}(x) d \zeta(x)   \]
for all perturbations $\zeta$ of the measure $\mu$.
\end{definition}

From the above definition, and under suitable assumptions on the energy $E$ (that we avoid making precise), one can obtain the following relation between first variation and gradient of an energy $E$:
\[ \nabla_{\mathrm{FR}} E(\mu) := \frac{\delta E }{\delta \mu} - \overline{\frac{\delta E }{\delta \mu}}, \]
where
\[ \overline{\frac{\delta E }{\delta \mu}}:= \int_\M \frac{\delta E }{\delta \mu} (x) d \mu(x); \]
note that the term $\overline{\frac{\delta E }{\delta \mu}}$ is subtracted from the first variation to guarantee that the gradient indeed belongs to $\mathcal{T}_\mu \mathcal{P}(\M)$. In particular, the gradient flow of $E$ with respect to the Fisher-Rao geometry is characterized by the equation
\[ \partial_t \mu_t = - \left ( \frac{\delta {E}}{\delta \mu}(\cdot; \mu_t) - \int \frac{\delta {E}}{\delta \mu}(x; \mu_t) d\mu_t(x) \right) \mu_t,\]
 where we have added an argument to the first variation to emphasize that it depends on the measure at which we evaluate it. 
 
\subsection{Geometric Interpretation of the Asymptotic Variance}
\label{sec:GeometricInterp}
We continue our informal discussion and consider the objective function
\[ F_l(\rho):= \lambda_l(\rho), \]
defined for those measures with a density $\rho$ for which the operator $\Delta_\rho$ makes sense. From now on, we identify measures with their densities.

Following the computation in \cite[Appendix C.1]{trillos2024} (which we present again in our Appendix \ref{app:EstimatesPerturbationTheory} for the convenience of the reader), we can observe that the first variation of $F_l$ at $\rho$ is given by
\[ \frac{\delta F_l }{\delta \rho} = - \lambda_l u_l^2 + 2 |\nabla u_l |^2 \rho^2,   \]
where $u_l= u_l(\rho)$ is the $l$-th eigenfunction of $\Delta_\rho$. Thus, from our discussion in section \ref{sec:FisherRao}, we should have
\begin{equation}
 \nabla_{\mathrm{FR}} F_l(\rho)=  - \lambda_l u_l^2 - \lambda_l  + 2 |\nabla u_l|^2 \rho.    
 \label{eq:GradEqn}
\end{equation}
So, if we consider the formal gradient flow equation
\begin{equation}
   \dot{\rho}_t= -\nabla_{\mathrm{FR}} F_l(\rho_t) \rho_t , \quad t >0,
   \label{eqn:GradFlow}
\end{equation}
then we obtain, from the chain rule, 
\[ \frac{d}{dt}  F_l(\rho_t) :=  -\langle \nabla_{\mathrm{FR} } F_l(\rho)   ,  \nabla_{\mathrm{FR}}  F_l(\rho)  \rangle_{\rho} =- \int_\M (   - \lambda_l u_l^2 - \lambda_l  + 2 |\nabla u_l|^2 \rho  )^2 \rho \dx = -\sigma_l^2(\rho),     \]
where we dropped the dependence on $t$ for simplicity. In other words, $\sigma_l^2(\rho)$ is the \textit{dissipation of the energy} $F_l$ \textit{along the gradient flow equation \eqref{eqn:GradFlow}}.

There is a more intuitive way to interpret the above computation. Indeed, suppose that, around a given $\rho$, we linearize the energy $F_l$ according to
\[ \mathcal{L}F_l(\xi) :=  F_l(\rho) + \langle  \nabla_{\mathrm{FR}} F_l (\rho) , \xi \rangle_\rho, \quad \xi \in \mathcal{T}_\rho \mathcal{P}(\M). \]
Here, the variable $\xi$ can be interpreted as an infinitesimal perturbation of $\rho$. In this interpretation, it makes sense to ask for the perturbation of $\rho$ that reduces the (linearization of the) energy $F_l$ the most. Precisely, one could consider the problem
\[ \min _{\xi \in \T_\rho \mathcal{P}(\M)  \text{ s.t. } \lVert\xi  \rVert_\rho \leq 1} \mathcal{L}F_l(\xi) - F_l(\rho) . \]
which is minimized when we take $\xi =- \frac{ \nabla_{\mathrm{FR}} F_l (\rho) }{\lVert \nabla_{\mathrm{FR}} F_l (\rho) \rVert_\rho}$. For this perturbation, the infinitesimal change in energy is precisely $\sigma^2_l(\rho)$.

\subsection{Cramer-Rao Lower Bound for Eigenvalue Estimation Problem}
\label{sec:CramerRaoAnalysis}
In this section, we prove Theorem \ref{thm:CramerRao}. However, before we provide a rigurous proof, we first discuss how the formal computations from the previous sections already suggest the desired Cramer-Rao lower bound. In the actual proof of the theorem at the end of this section, we justify some of the following steps.

Suppose $\tilde \lambda_l$ is an \textit{unbiased} estimator for $\lambda_l$. It follows that
\[ F_l(\rho)= \lambda_l(\rho) = \E_{\rho}[ \tilde{\lambda}_l(x_1, \dots, x_n)  ]   = \int_\M \cdots \int_\M  \tilde{\lambda}_l(x_1, \dots, x_n) \rho(x_1) \dots \rho(x_n ) \dd x_1 \dots \dd x_n \]
for all sufficiently regular $\rho$. Consider now an arbitrary trajectory 
\[ \partial_t {\rho}_t = \xi_t \rho_t, \quad  \xi_t \in \mathcal{T}_{\rho_t} \mathcal{P}(\M).\]
Then, on the one hand, by the chain rule we have
\[ \frac{d}{dt} F_l(\rho_t) =  \langle \nabla_{\mathrm{FR}} F_l(\rho_t) , \xi_t \rangle_{\rho_t}.   \]
On the other hand, a direct computation reveals
\[ \frac{d}{dt} \int_\M \cdots \int_\M  \tilde{\lambda}_l(x_1, \dots, x_n) \rho_t(x_1) \dots \rho_t(x_n ) \dd x_1 \dots \dd x_n  = \E_{\rho_t}\left[\tilde{\lambda}_l(x_1, \dots, x_n) \left(\sum_{i=1}^n \xi_t(x_i)\right)    \right] . \]
That is, 
\[   \langle \nabla_{\mathrm{FR}} F_l(\rho) , \xi \rangle_{\rho}   =  \E_\rho\left[\tilde{\lambda}_l(x_1, \dots, x_n) \left(\sum_{i=1}^n \xi(x_i)\right)\right],  \]
where we drop the $t$ dependence for simplicity. Using the fact that $ \E_\rho[\xi] =0$ (because $\xi \in \T_\rho \mathcal{P}(\M) $) and the Cauchy-Schwartz inequality on the right-hand side of the above expression, we obtain
\[ \langle \nabla_{\mathrm{FR}} F_l(\rho) , \xi \rangle_{\rho}   \leq  \sqrt{n} \sqrt{\Var_\rho(\tilde{\lambda}_l)} \lVert \xi \rVert_\rho.        \]
Reorganizing, we obtain
\begin{equation*}
\sqrt{\Var_\rho(\tilde{\lambda}_l)} \geq \frac{1}{\sqrt{n}} \frac{\langle \nabla_{\mathrm{FR}} F_l(\rho) , \xi \rangle_{\rho}}{ \lVert \xi \rVert_\rho}. 
\end{equation*}
Since $\xi$ was arbitrary, we can maximize on the right hand side of the above expression to deduce
\begin{equation}
\Var_\rho(\tilde{\lambda}_l) \geq  \frac{ \lVert \nabla_{\mathrm{FR}} F_l(\rho) \rVert_\rho^2}{n} = \frac{\sigma_l^2(\rho)}{n}. 
\label{eq:CramerRao}
\end{equation}

After the above informal discussion, which was intended to highlight the geometric nature of the Cramer-Rao lower bound, we proceed to present the rather direct proof of Theorem \ref{thm:CramerRao}. 

\begin{proof}[Proof of Theorem \ref{thm:CramerRao}]
Suppose that $\tilde{\lambda}_l$ is an unbiased estimator for $\lambda_l$ as in Definition \ref{def:UnbiasedEstimator} and let $\rho \in \mathcal{P}_\M^\alpha$ be such that $\lambda_l(\rho)$ is simple. For $t$ such that $|t|$ is sufficiently small, consider
\[\rho_t:= \rho + t ( - \lambda_l u_l^2 - \lambda_l  + 2 |\nabla u_l|^2 \rho) \rho ,  \]
where $\lambda_l = \lambda_l(\rho)$ and $u_l= u_l(\rho)$. By the discussion in Appendix \ref{sec:Regularity}, if $\rho \in C^{2, \alpha}(\M)$, then $u_l \in C^{3,\alpha}(\M)$. It follows that $\rho_t \in \mathcal{P}_\M^\alpha$ for all $t$ for which $|t|$ is small enough.

We use the fact that $\tilde \lambda_l$ is unbiased to conclude that
\[\lambda_l(\rho_t) = \int_\M \cdots \int_\M  \tilde{\lambda}_l(x_1, \dots, x_n) \rho_t(x_1) \dots \rho_t(x_n ) \dd x_1 \dots \dd x_n,\]
for all $t$ such that $|t|$ is sufficiently small. Taking derivative in $t$ on both sides of the above equality and evaluating at $t=0$, we deduce
\[ \sigma^2_l(\rho) = \E_{\rho}\left[ \tilde{\lambda}_l(x_1 \dots, x_n) \left( \sum_{i=1}^n \xi(x_i) \right) \right], \]
where the left hand side follows from \eqref{eqn:PerturbationEigenvalues}, and where $\xi(x):= -\lambda_l u_l^2(x) - \lambda_l + 2 |\nabla u_l(x)|^2 \rho(x)$. Using the fact that $\int_\M \xi(x) \rho(x) \dx =0$ and Cauchy-Schwartz inequality, it follows that
\[ \sigma^2_l(\rho) \leq \sqrt{n} \sqrt{\Var_\rho(\tilde{\lambda}_l)} \sigma_l(\rho).\]     
This establishes the desired lower bound for the variance of $\tilde{\lambda}_l$.
\end{proof}


\section{Numerical Experiments}\label{sec:simulation}

In this section, we validate our theoretical findings with empirical evidence from finite-sample simulations and explore some potential relaxations of our theoretical assumptions. As representative examples, we consider data independently and uniformly sampled from an $m$-torus and an $m$-sphere of (intrinsic) dimension $m\in \mathbb{N}$. For every experimental setting, we independently repeat the simulation $1000$ times. In each repetition, an $\eps$-graph is constructed from $n$ data points, using $\eps = \eps_n \propto n^{-1/(m+4)}$ (unless stated otherwise), which lies on the lower boundary of the required range of $\eps$ in our theoretical analysis; we use the kernel $\eta= \mathbf{1}_{[0,1]}$. The code used to perform the numerical experiments is publicly available on GitHub at \url{https://github.com/chl781/GraphLaplacianCLT}.



\subsection{Asymptotic Normality and Rate of Convergence}\label{ss:gaus} 
We begin by examining the shape of the finite-sample distribution of $\sqrt{n}\hat{\lambda}^\varepsilon_l$ (in particular, for $l = 2$) and how quickly it converges to a normal distribution. Specifically, we consider data sampled from an ellipse (i.e., stretched 1-sphere), a 2-sphere and a 2-torus embedded in the ambient space $\mathbb{R}^3$, with varying sample sizes $n \in \{1000, 2000, 4000\}$. The radii of the ellipse are $1$ and $\sqrt{2}$, and the radii of the sphere and the torus are the same, 
and equal to one. The shapes of distributions are estimated using standard kernel density estimators (with default parameters in MATLAB), and their proximity to a normal distribution is assessed via QQ-plots; see Figures~\ref{fig:den-normal} and~\ref{fig:qq-normal}.

\begin{figure}[htbp]
  \centering
  \begin{subfigure}[t]{0.32\linewidth}
    \centering
    \includegraphics[width=\linewidth]{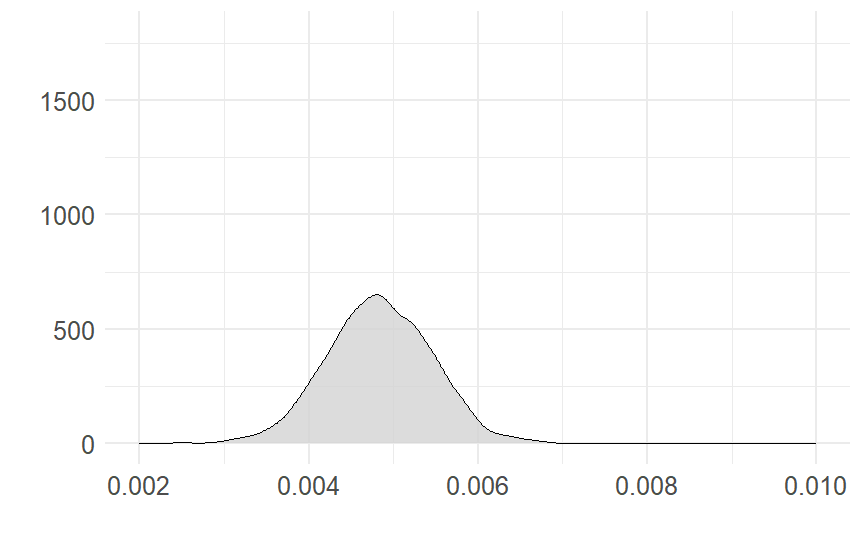}
    \caption{Ellipse, $ n = 1000$}
    \label{fig:circle-den-1k}
  \end{subfigure}\hfill
  \begin{subfigure}[t]{0.32\linewidth}
    \centering
    \includegraphics[width=\linewidth]{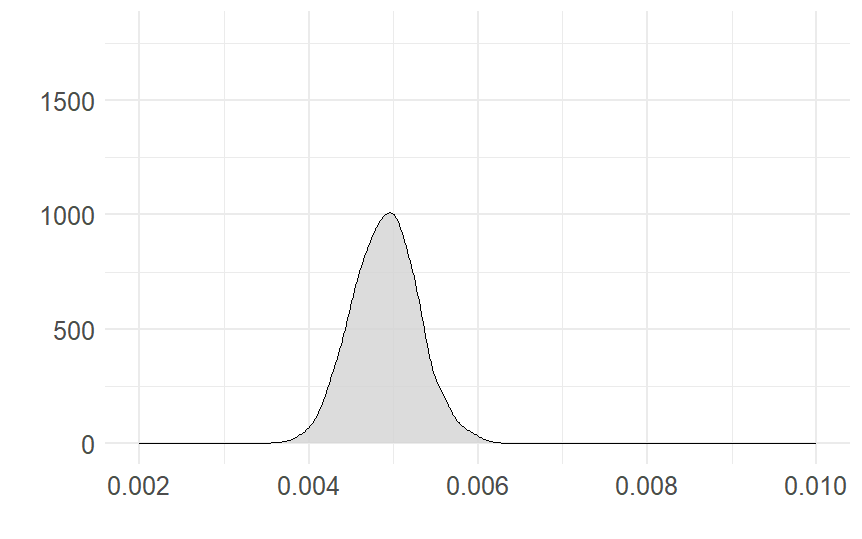}
    \caption{Ellipse, $ n = 2000$}
    \label{fig:circle-den-2k}
  \end{subfigure}\hfill
  \begin{subfigure}[t]{0.32\linewidth}
    \centering
    \includegraphics[width=\linewidth]{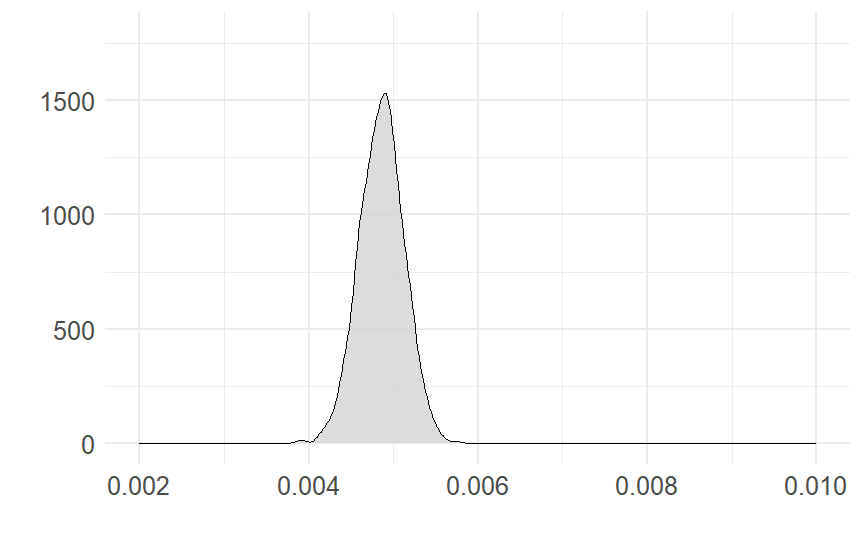}
    \caption{Ellipse, $ n = 4000$}
    \label{fig:circle-den-4k}
  \end{subfigure}
   \vspace{1em} 
  \begin{subfigure}[t]{0.32\linewidth}
    \centering
    \includegraphics[width=\linewidth]{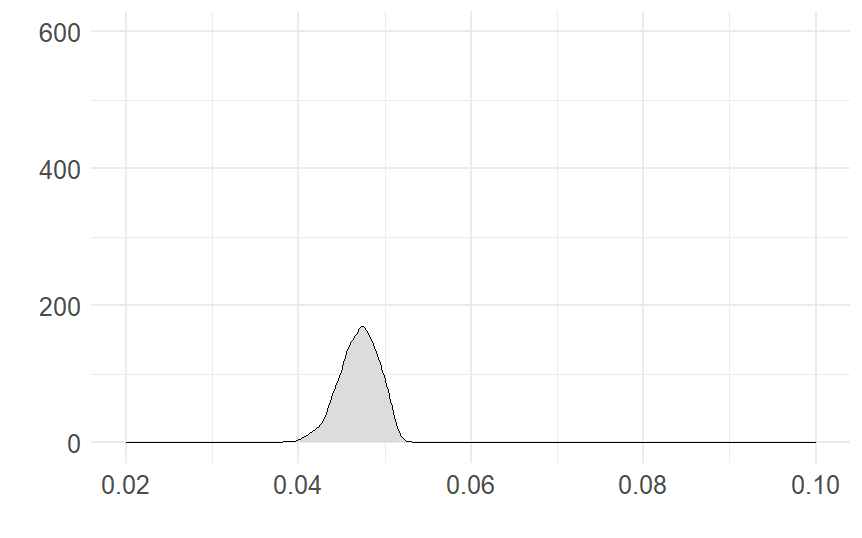}
    \caption{Sphere, $ n = 1000$}
    \label{fig:sphere-den-1k}
  \end{subfigure}\hfill
  \begin{subfigure}[t]{0.32\linewidth}
    \centering
    \includegraphics[width=\linewidth]{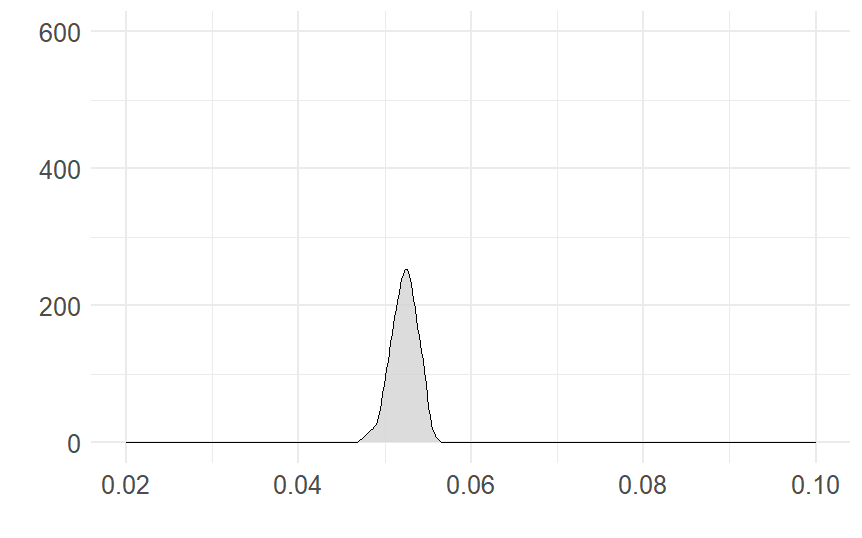}
    \caption{Sphere, $ n = 2000$}
    \label{fig:sphere-den-2k}
  \end{subfigure}\hfill
  \begin{subfigure}[t]{0.32\linewidth}
    \centering
    \includegraphics[width=\linewidth]{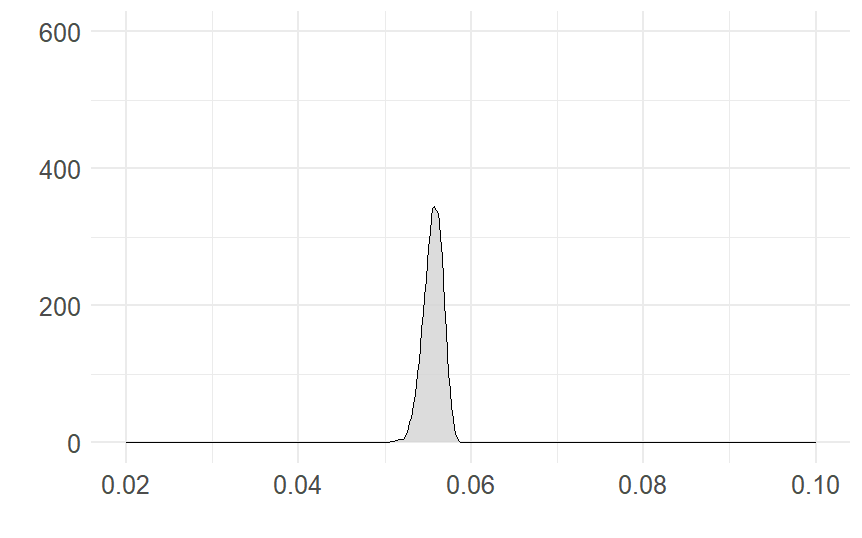}
    \caption{Sphere, $ n = 4000$}
    \label{fig:sphere-den-4k}
  \end{subfigure}
   \vspace{1em} 
  \begin{subfigure}[t]{0.32\linewidth}
    \centering
    \includegraphics[width=\linewidth]{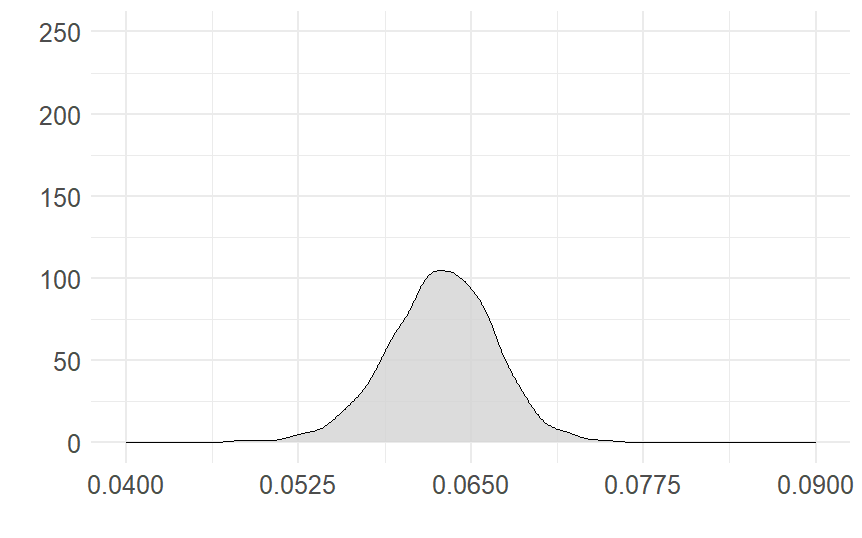}
    \caption{Torus, $n=1000$}
    \label{fig:torus-den-1k}
  \end{subfigure}\hfill
  \begin{subfigure}[t]{0.32\linewidth}
    \centering
    \includegraphics[width=\linewidth]{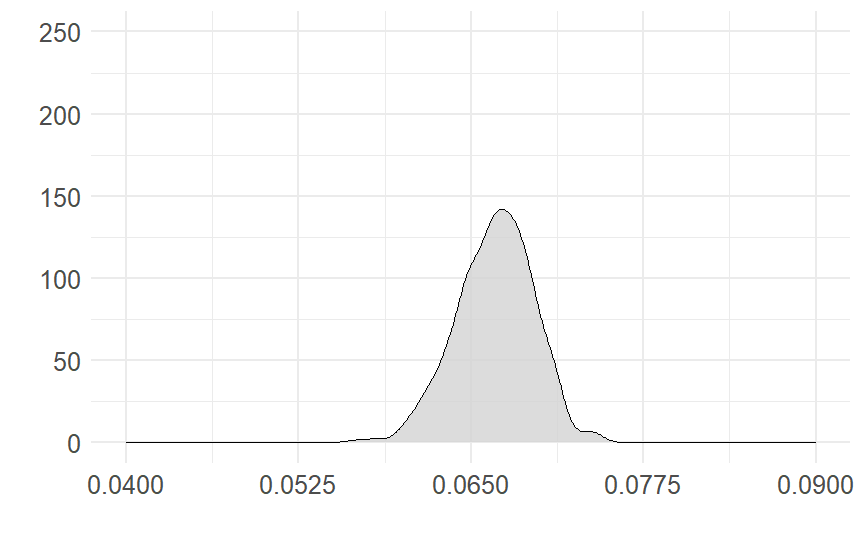}
    \caption{Torus, $n=2000$}
    \label{fig:torus-den-2k}
  \end{subfigure}\hfill
  \begin{subfigure}[t]{0.32\linewidth}
    \centering
    \includegraphics[width=\linewidth]{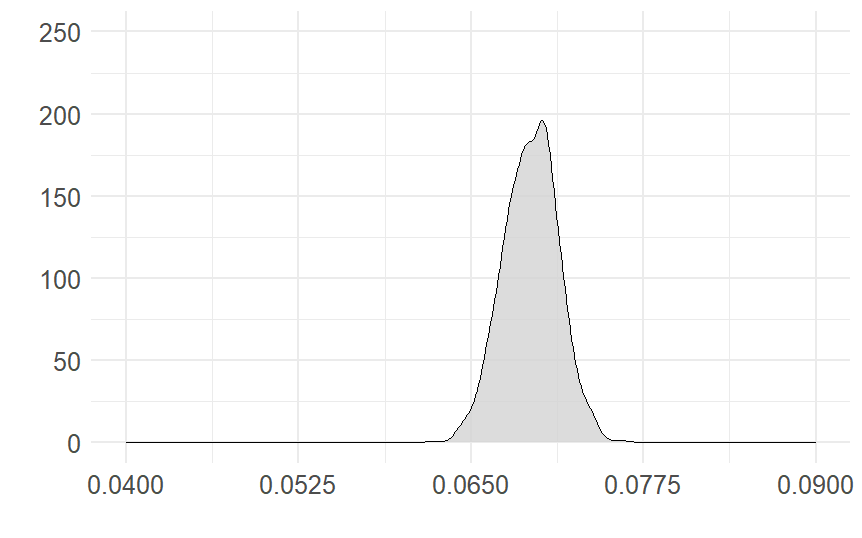}
    \caption{Torus, $n=4000$}
    \label{fig:torus-den-4k}
  \end{subfigure}
  \caption{Kernel density estimates of $\sqrt{n}\hat{\lambda}_2^\varepsilon$ at different sample sizes~$n$, computed for data sampled from an ellipse (top) with axes of size 1 and $\sqrt{2}$, a 2-sphere (middle) and a 2-torus (bottom). Here the $\eps$-graphs are constructed using $\eps = n^{-1/(m+4)}$ for the ellipse and 2-sphere, and $\eps = 6 n^{-1/(m+4)}$ for the 2-torus.}
  \label{fig:den-normal}
\end{figure}

\begin{figure}[htbp]
  \centering
  \begin{subfigure}[t]{0.32\linewidth}
    \centering
    \includegraphics[width=\linewidth]{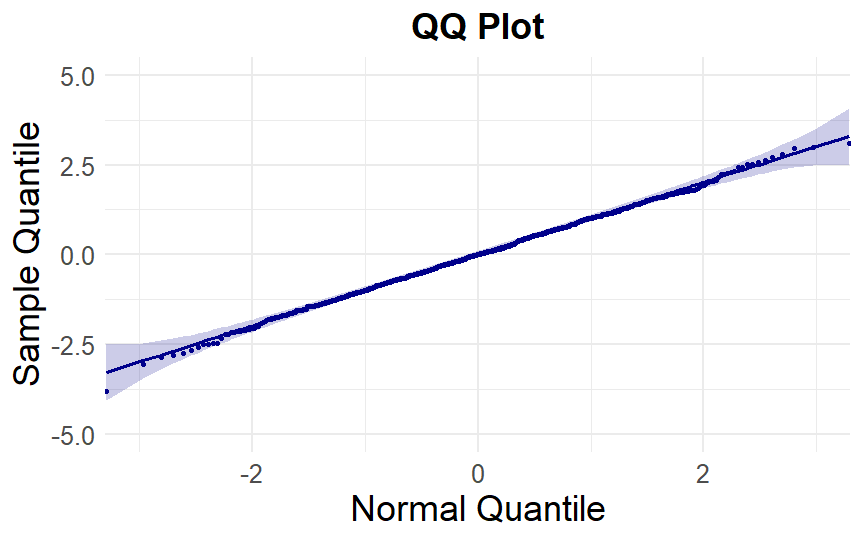}
    \caption{Ellipse, $n = 1000$}
    \label{fig:circle-qq-1k}
  \end{subfigure}\hfill
  \begin{subfigure}[t]{0.32\linewidth}
    \centering
    \includegraphics[width=\linewidth]{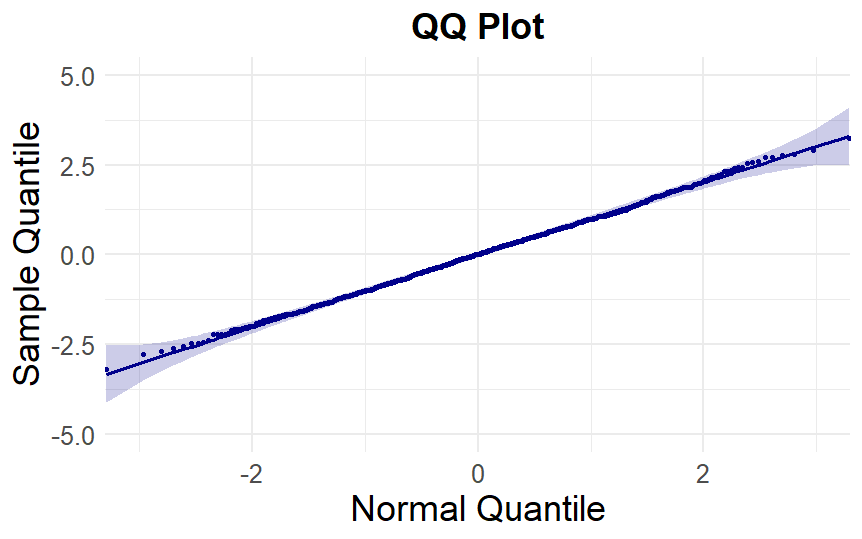}
    \caption{Ellipse, $n = 2000$}
    \label{fig:circle-qq-2k}
  \end{subfigure}\hfill
  \begin{subfigure}[t]{0.32\linewidth}
    \centering
    \includegraphics[width=\linewidth]{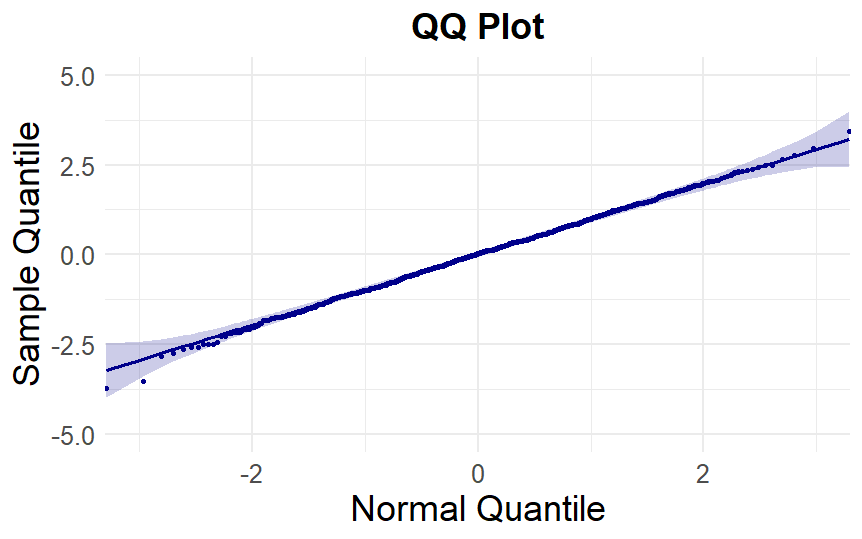}
    \caption{Ellipse, $n = 4000$}
    \label{fig:circle-qq-4k}
  \end{subfigure}
   \vspace{1em}
  \begin{subfigure}[t]{0.32\linewidth}
    \centering
    \includegraphics[width=\linewidth]{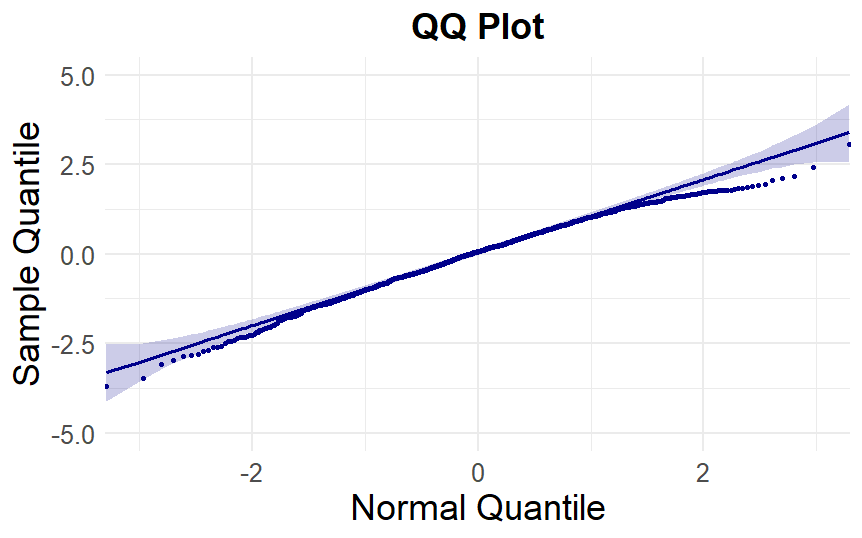}
    \caption{Sphere, $n = 1000$}
    \label{fig:sphere-qq-1k}
  \end{subfigure}\hfill
  \begin{subfigure}[t]{0.32\linewidth}
    \centering
    \includegraphics[width=\linewidth]{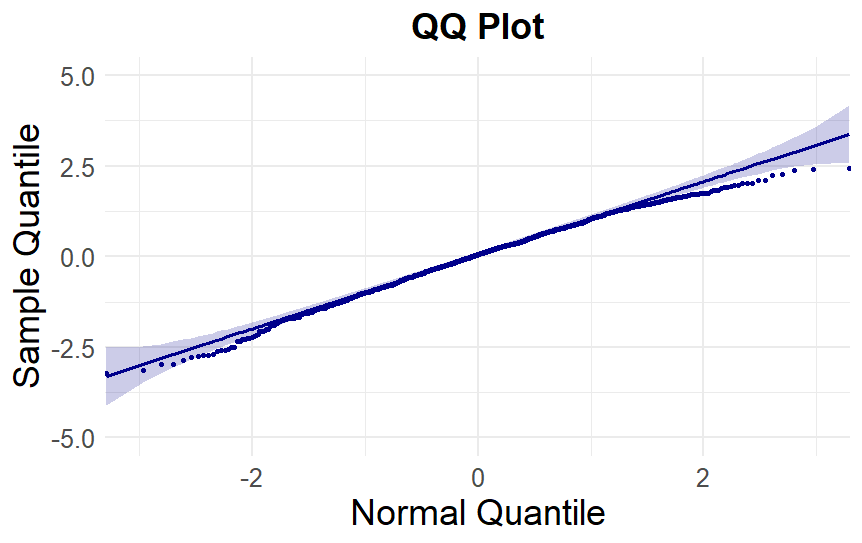}
    \caption{Sphere, $n = 2000$}
    \label{fig:sphere-qq-2k}
  \end{subfigure}\hfill
  \begin{subfigure}[t]{0.32\linewidth}
    \centering
    \includegraphics[width=\linewidth]{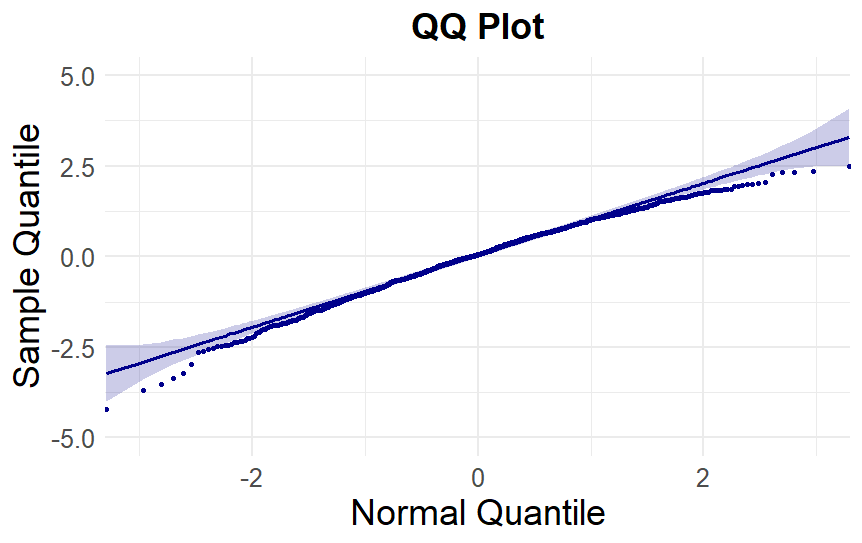}
    \caption{Sphere, $n = 4000$}
    \label{fig:sphere-qq-4k}
  \end{subfigure}
  \vspace{1em} 
  \begin{subfigure}[t]{0.32\linewidth}
    \centering
    \includegraphics[width=\linewidth]{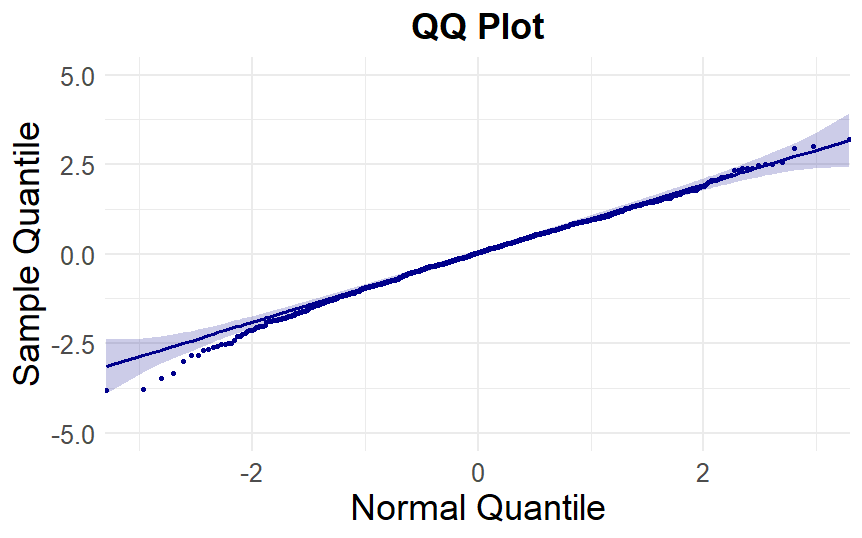}
    \caption{Torus, $n = 1000$}
    \label{fig:torus-qq-1k}
  \end{subfigure}\hfill
  \begin{subfigure}[t]{0.32\linewidth}
    \centering
    \includegraphics[width=\linewidth]{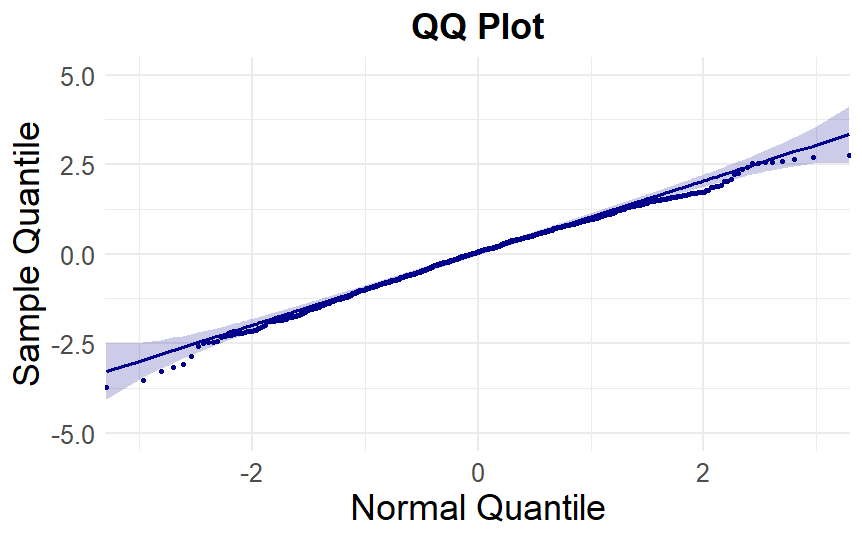}
    \caption{Torus, $n = 2000$}
    \label{fig:torus-qq-2k}
  \end{subfigure}\hfill
  \begin{subfigure}[t]{0.32\linewidth}
    \centering
    \includegraphics[width=\linewidth]{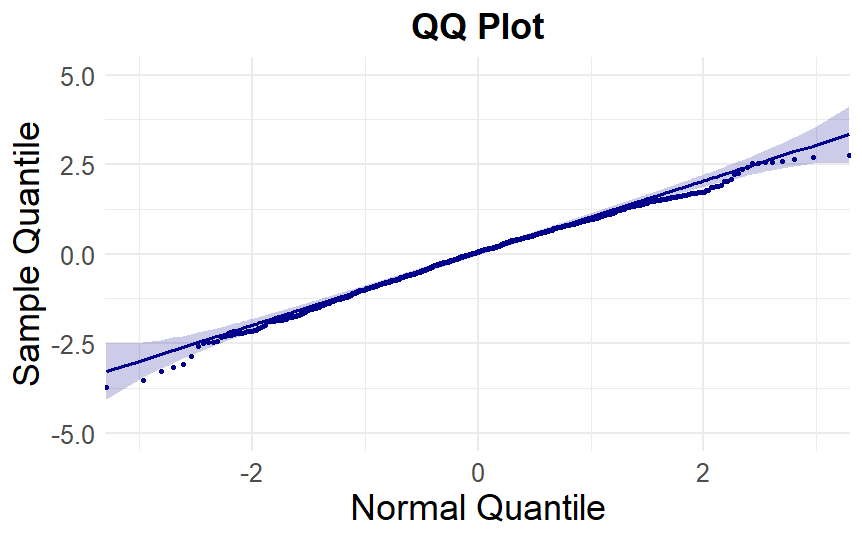}
    \caption{Torus, $n = 4000$}
    \label{fig:torus-qq-4k}
  \end{subfigure}
  \caption{QQ-plots of $\sqrt{n}\hat{\lambda}_2^\varepsilon$ versus the standard normal distribution for different sample sizes~$n$ on an ellipse (top) with radii being 1 and $\sqrt{2}$, a unit 2-sphere (middle) and a unit 2-torus (bottom). Here the $\eps$-graphs are constructed using $\eps = n^{-1/(m+4)}$ for the ellipse and 2-sphere and $\eps = 6 n^{-1/(m+4)}$ for the 2-torus.}
  \label{fig:qq-normal}
\end{figure}

As illustrated in the figures, the distribution of $\sqrt{n}\hat{\lambda}^\varepsilon_2$ tends toward normality, in accordance with our theoretical finding in \Cref{thm1}.  Unlike the ellipse, the 2-sphere and 2-torus exhibit repeated nontrivial eigenvalues in their Laplace-Beltrami operators (cf.~\Cref{tab:LB_2sphere} for the 2-sphere), thereby violating the eigengap condition stated in \cref{assum:eigengap for l}. Nevertheless, the simulation results suggest that this condition might be relaxed.


In addition, we explore the choice of $\varepsilon$ beyond the range specified in \cref{thm1}, and examine whether the finite-sample distribution of $\sqrt{n}\hat{\lambda}^\varepsilon_2$ continues to converge to a limiting distribution. In particular, we consider a small value $\varepsilon = n^{-1/(m+1)}$ and focus on the case of the 2-sphere. The QQ-plot and kernel density estimate shown in \Cref{fig:sm-eps} indicate that the normal approximation remains valid in practice. This observation suggests the potential to further extend the admissible range of $\varepsilon$ in our main theorems.

\begin{figure}[htbp]
\begin{subfigure}[t]{0.4\textwidth}
 \centering  	        \includegraphics[width=\linewidth]{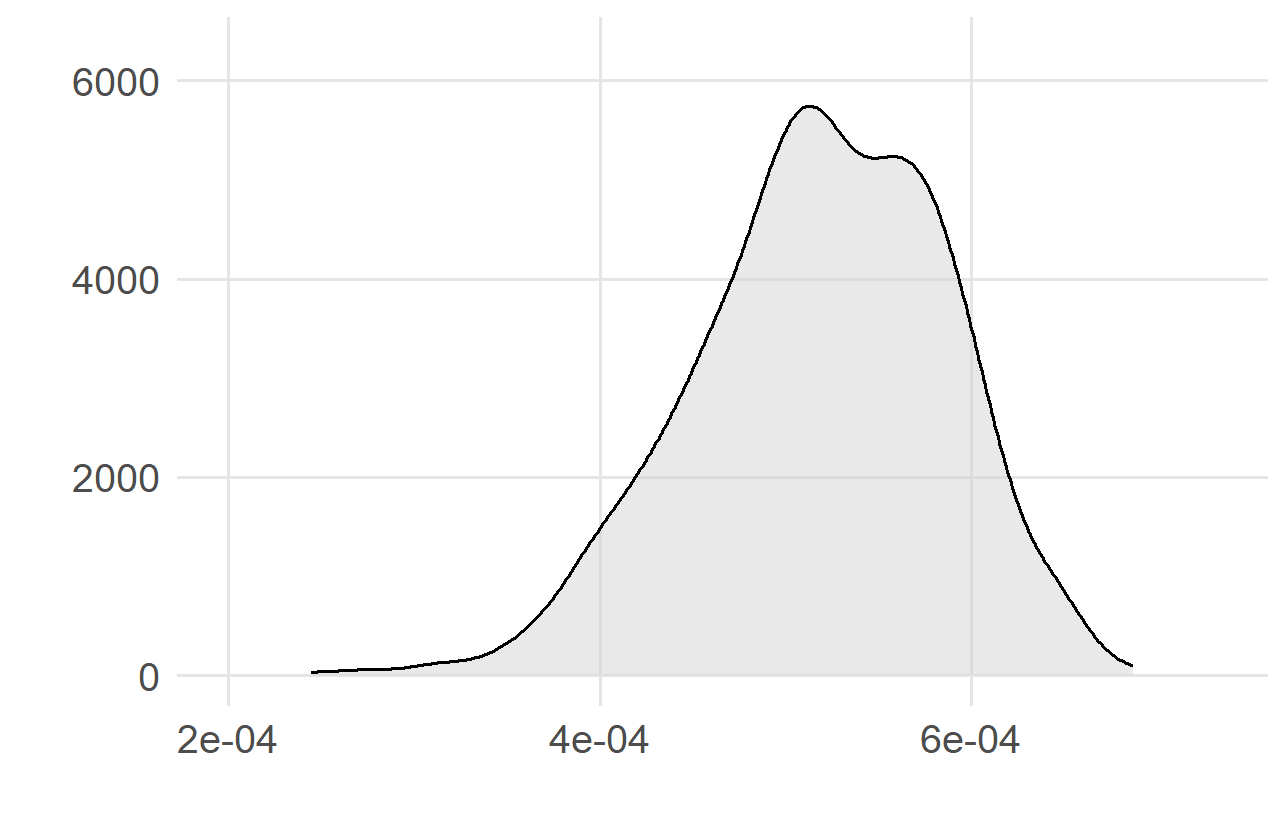}
			\caption{Kernel density estimate}
			\label{fig:den-sm-eps}
\end{subfigure}
%
%
\begin{subfigure}[t]{0.4\textwidth}
\centering
    	         \includegraphics[width=\linewidth]{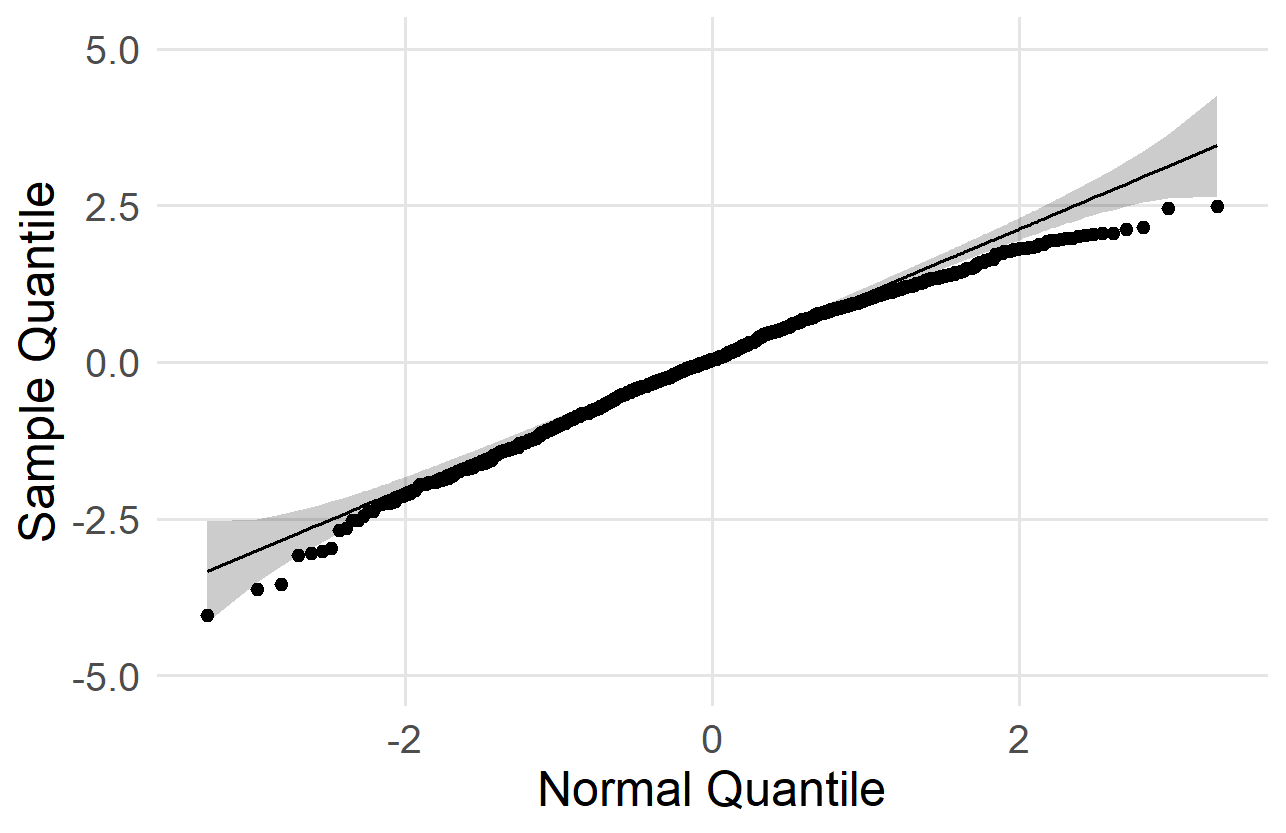}
			\caption{QQ-plot}
			\label{fig:qq-sm-eps}
\end{subfigure}
   \caption{Kernel density estimate (left) and QQ-plot against the standard normal distribution (right) of $\sqrt{n}\hat{\lambda}^\varepsilon_2$ on a 2-sphere with sample size $n = 12000$. Here the $\varepsilon$-graphs are constructed using $\varepsilon = n^{-1/(m+1)}$, which lies outside the admissible range in our main theorems.}\label{fig:sm-eps}
	\end{figure}

\subsection{Asymptotic Variance and Its Estimator}\label{ss:var} As an indicator of convergence, we investigate the finite-sample variance of $\sqrt{n}\hat{\lambda}^\varepsilon_l$ and its closeness to the asymptotic variance $\sigma_l^2$ defined in~\eqref{eq-def:sigma^2}. We focus again on the 2-sphere of unit radius and consider the smallest non-trivial eigenvalue (i.e.\ $l = 2$) of the unnormalized graph Laplacian. \Cref{fig:var-2sphere} displays how the empirical variance of $\sqrt{n}\hat{\lambda}^\varepsilon_2$ scale across varying sample sizes $n \in \{2000,4000,\dots,14000\}$. The stable empirical variance estimates across different sample sizes support the claim in \Cref{thm1}. Although getting a closed form expression for $\sigma_2^2$ is not possible in this setting, we can use numerical integration to deduce that $\sigma_2^2 \approx 12.80$ (see \cref{tab:LB_2sphere}), which is in agreement with the plot in \cref{fig:var-2sphere}.

\begin{figure}[htbp]
\centering
\includegraphics[width=.8\linewidth]{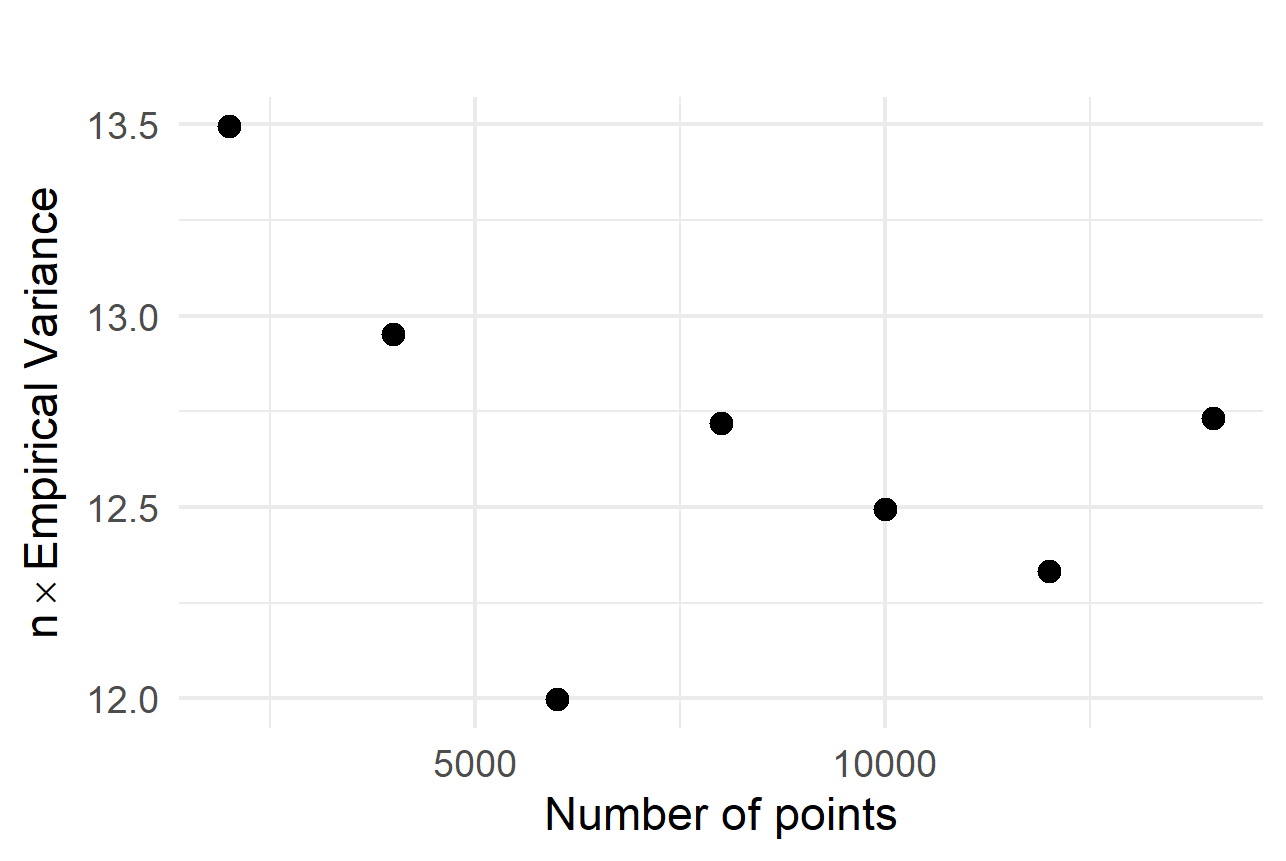}
\caption{The empirical finite-sample variance of $\sqrt{n} \hat{\lambda}^\varepsilon_2$ of graph Laplacian constructed with $n$ points drawn uniformly from a 2-sphere for the values $n=\{2000,4000,6000,8000,10000,12000,14000\}$.}
\label{fig:var-2sphere}
\end{figure}

To enable a more precise comparison of the two variances (the estimated one and the theoretical one), we consider the example of a 1-sphere (i.e., a unit circle) of unit radius, which permits exact computation of the asymptotic variance of the smallest non-trivial eigenvalue (namely, $\hat{\lambda}^\varepsilon_2$) of the graph Laplacian. In this setting, elementary calculations yield the scaling factor $\sigma_\eta = 2/3$ for the kernel $\eta = \mathbf{1}_{[0,1]}$ (which we have used throughout), the smallest non-trivial eigenvalue of the weighted Laplace-Beltrami operator $\lambda_2 = 1/(6\pi)$, and the corresponding eigenfunction $u_2(\theta) = \sqrt{2} \sin \theta$ for $\theta \in [0, 2\pi]$, which is a particular case of the spherical harmonics for the circle. From the definition in \eqref{eq-def:sigma^2}, it follows that the asymptotic variance is $\sigma_2^2 = 81/(8\pi^2)$. In \Cref{fig:var-circle}, we plot the ratio of the finite-sample variance of $\sqrt{n}\hat{\lambda}^\varepsilon_2$ to its asymptotic counterpart $\sigma_2^2$ for various sample sizes $n \in \{1000, 2000, \ldots, 7000\}$.  The resulting curve exhibits a clear pattern of convergence, with the finite-sample and asymptotic variances aligning closely for sample sizes $n\ge 5000$. This observation reflects a commendable rate of convergence toward the derived theoretical limit. 

\begin{figure}[htbp]
		\begin{subfigure}[t]{0.45\textwidth}
			\centering
    	         \includegraphics[width=\linewidth]{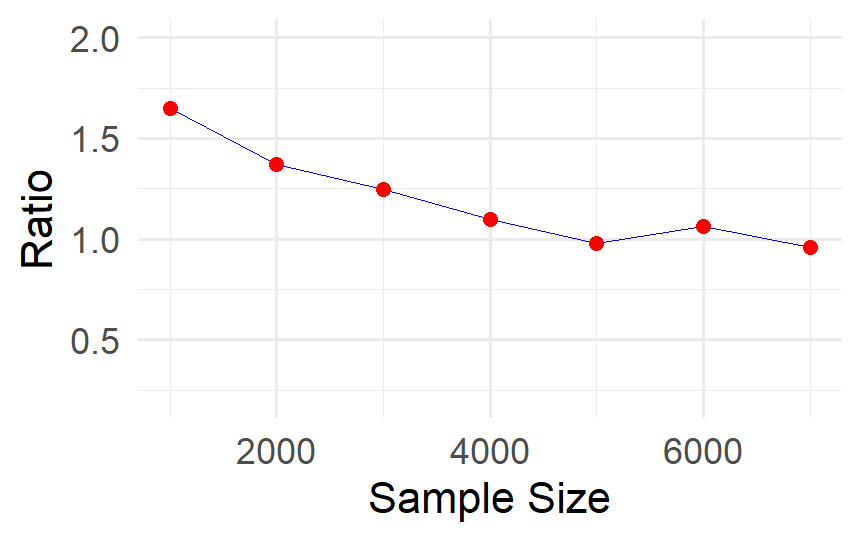}
			\caption{$n\mathrm{Var}(\hat{\lambda}^\varepsilon_2)/\sigma_2^2$}
			\label{fig:var-circle}
			\end{subfigure}
            \begin{subfigure}[t]{0.45\textwidth}
			\centering
    	         \includegraphics[width=\linewidth]{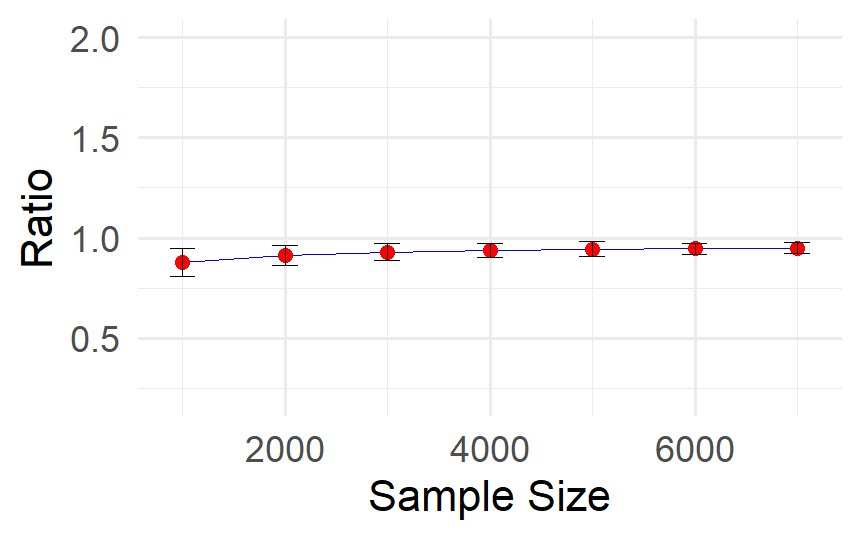}
			\caption{$\hat{\sigma}_2^2/\sigma_2^2$}
			\label{fig:est-circle}
			\end{subfigure}
   \caption{Ratios of the empirical finite-sample variance of $\sqrt{n}\hat{\lambda}^\eps_2$  (left) and of the variance estimator $\hat{\sigma}_2^2$ (right), relative to the asymptotic variance $\sigma_2^2$, evaluated on the unit circle. In the right panel, the curve represents the median over 100 repetitions, and the error bars denote the median absolute deviation.}\label{fig:circle}
\end{figure}

As shown above, the asymptotic variance $\sigma_l^2$ provides a reasonable assessment of the variability of $\hat\lambda_l^\eps$ when the sample size is not too small. However, the computation of  $\sigma_l^2$ requires knowledge of the underlying manifold and the associated probability density $\rho$, both of which are typically unavailable in practice. To address this limitation, we introduced the estimator $\hat\sigma_l^2$ in  subsection \ref{sec:extension}. Continuing with the experimental setup on the unit circle, we examine the performance of $\hat\sigma_2^2$ as defined in \eqref{eq-def:hat variance} in approximating the asymptotic variance $\sigma_2^2$; see \Cref{fig:est-circle}. Remarkably, it performs well even for small sample sizes. This suggests that $\hat\sigma_l/\sqrt{n}$ can be used to estimate the finite-sample standard deviation of $\hat\lambda_l^\eps$ in practical applications.

\subsection{Centralization and Dimension Restriction}\label{ss:dim} Recall from the discussion in \Cref{ss:bias} that the validity of asymptotic distribution of $\hat\lambda_l^\eps$ when centering at its continuum counterpart $\lambda_l$, instead of its expectation $\E[\hat\lambda_l^\eps]$, seems to require an additional restriction on the intrinsic dimension, namely, $m \le 3$. 
Here we  use simulations to investigate whether this restriction arises from technical limitations in the proof or reflects a fundamental characteristic of the problem. Note the decomposition
\[
\sqrt{n}\left(\hat\lambda_l^\eps - \lambda_l\right) = \sqrt{n}\left(\hat\lambda_l^\eps - \E[\hat{\lambda}_l^\eps]\right) +\sqrt{n}\left (\E[\hat{\lambda}_l^\eps]-  \lambda_l\right).
\]
One natural experiment is to examine the relation between the absolute bias $\sqrt{n}\bigl|\E[\hat{\lambda}_l^\eps]-\lambda_l\bigr|$ and the standard deviation of  $\sqrt{n}\hat\lambda_l^\eps$ across different dimensions $m$. To this end, we consider the case of $l=2$ on a $m$-sphere, with varying sample sizes $n \in \{1024, 2048, 4096, 8192, 16384\}$. We estimate $\E[\hat{\lambda}_2^\eps]$ and its standard deviation via their empirical counterparts over 100 independent repetitions. \Cref{fig:bias_std_scale,fig:bias_std_scale_eps=m+0.5} present the results for $\eps$-graphs with $\eps = n^{-\frac{1}{3.5}}$ and $\eps=n^{-\frac1{m+0.5}}$, respectively. Both figures reveal that the dimension restriction discussed in the introduction does seem to be an inherent aspect of our estimation problem. When $\eps\propto n^{-\frac{1}{3.5}}$, the scaled bias decreases in $n$ when $m=1$, and increases when $m\ge 2$. When $\eps \propto n^{-\frac{1}{m+0.5}}$, the scaled bias increases even when $m=1$. The empirical standard deviation of $\sqrt{n}\hat{\lambda}_2$ is, on the other hand, relatively stable in all of those examples, which suggests that, when centralizing over $\E[\hat{\lambda}_2]$, the CLTs for eigenvalues of the graph Laplacian should hold.

\begin{figure}[htbp]
    \centering
     \begin{subfigure}[t]{0.24\textwidth}
        \centering
        \includegraphics[width=\linewidth]{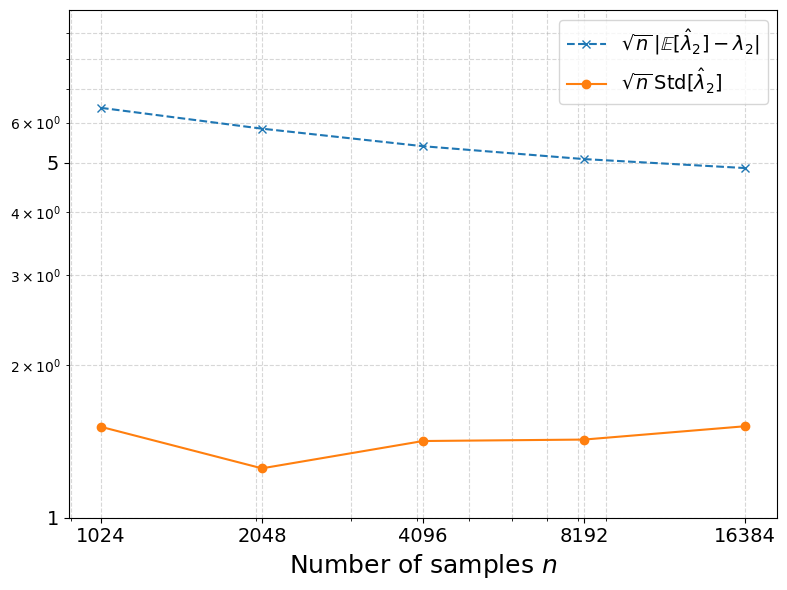}
        \caption{$m=1$}
        \label{fig:large_eps_m1}
    \end{subfigure}
    \hfill
    \begin{subfigure}[t]{0.24\textwidth}
        \centering
        \includegraphics[width=\linewidth]{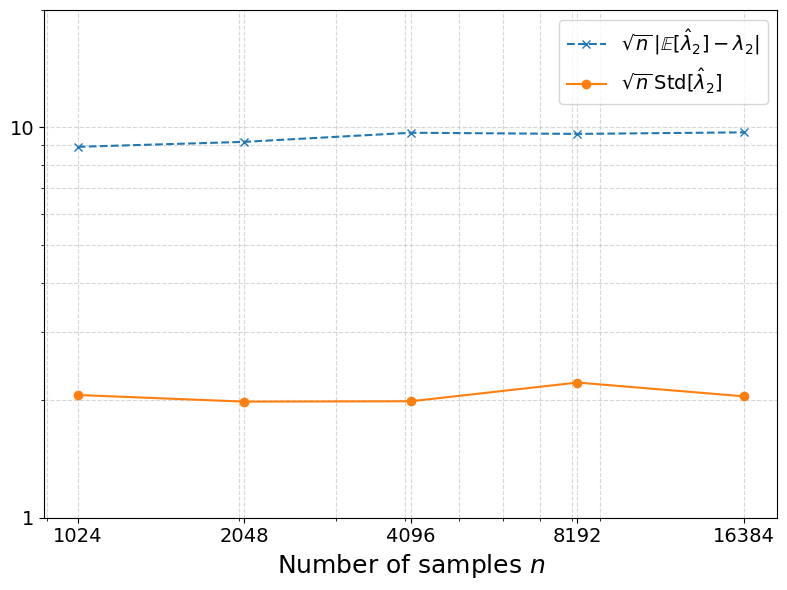}
        \caption{$m=2$}
        \label{fig:large_eps_m2}
    \end{subfigure}
    \hfill
    \begin{subfigure}[t]{0.24\textwidth}
        \centering
        \includegraphics[width=\linewidth]{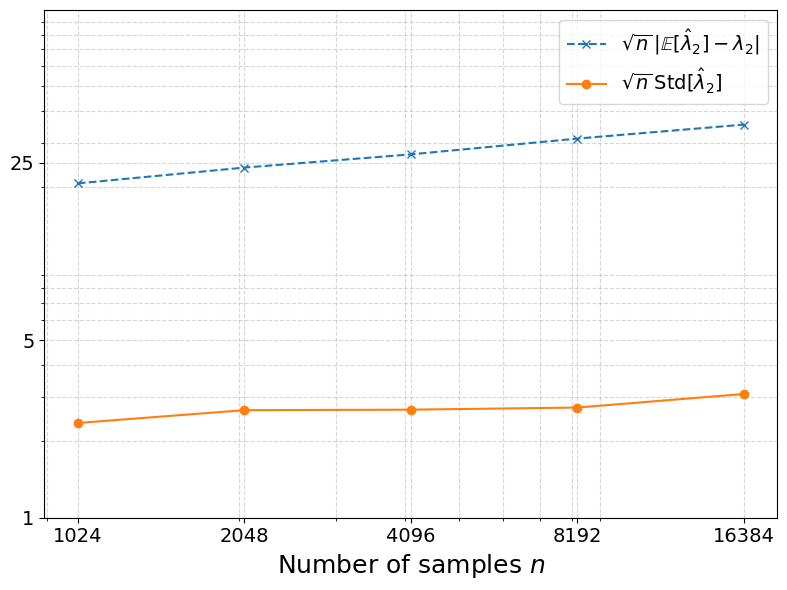}
        \caption{$m=3$}
        \label{fig:large_eps_m3}
    \end{subfigure}
    \hfill
    \begin{subfigure}[t]{0.24\textwidth}
        \centering
        \includegraphics[width=\linewidth]{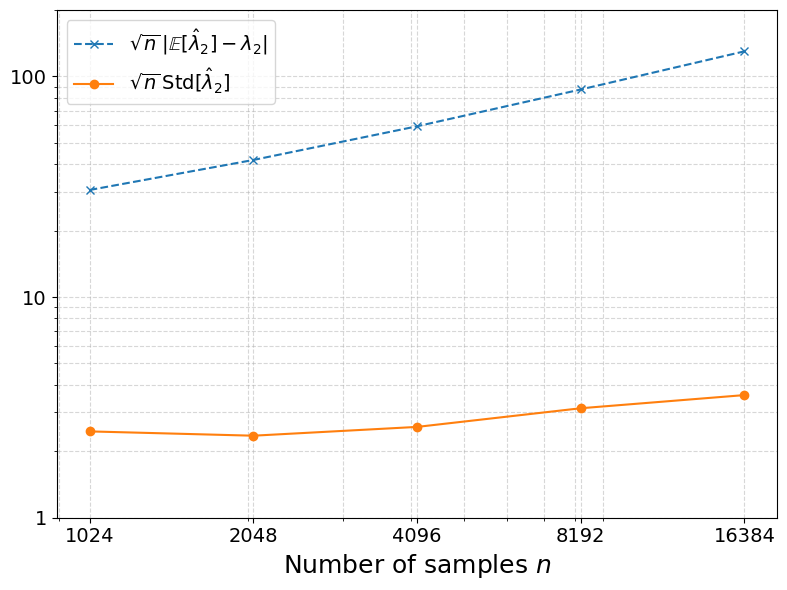}
        \caption{$m=4$}
        \label{fig:large_eps_m4}
    \end{subfigure}

    \caption{Relation of bias $\sqrt{n}|\E[\hat{\lambda}_2^\eps]-\lambda_2|$ (blue dashed) and standard deviation of $\sqrt{n}\hat{\lambda}_2^\eps$ (orange solid) on an $m$-sphere when $\eps=4n^{-\frac{1}{3.5}}$ for $m=1,2,3,4$ (left to right).}
    \label{fig:bias_std_scale}
\end{figure}

\begin{figure}[htbp]
    \centering
    \begin{subfigure}[t]{0.24\textwidth}
        \centering
        \includegraphics[width=\linewidth]{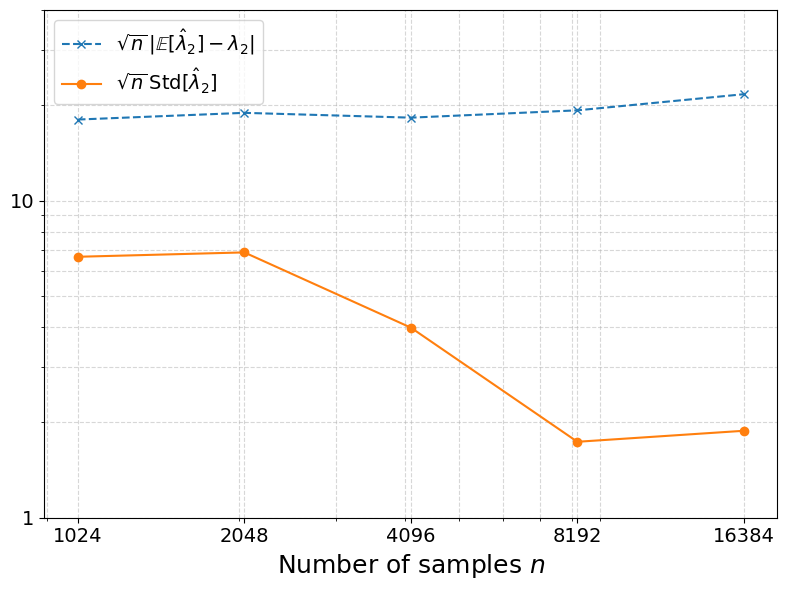}
        \caption{$m=1$}
        \label{fig:small_eps_m1}
    \end{subfigure}
    \hfill
    \begin{subfigure}[t]{0.24\textwidth}
        \centering
        \includegraphics[width=\linewidth]{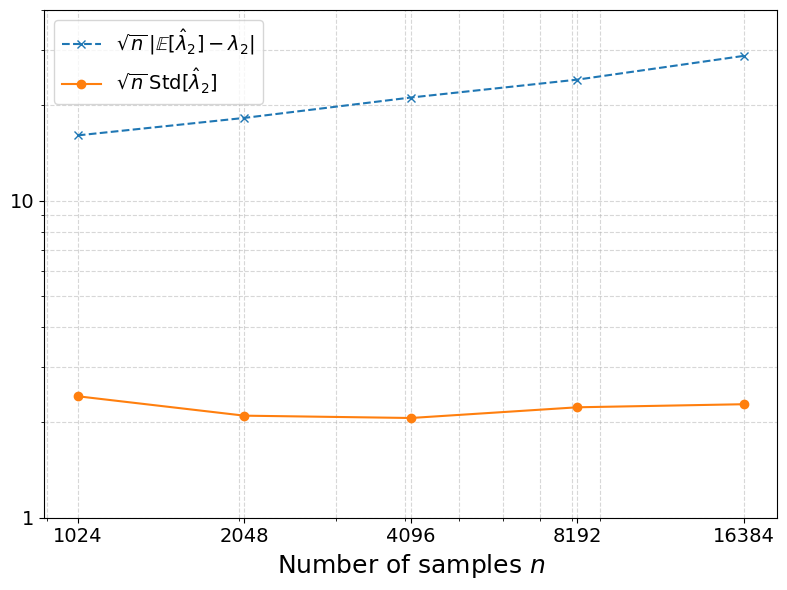}
        \caption{$m=2$}
        \label{fig:small_eps_m2}
    \end{subfigure}
    \hfill
    \begin{subfigure}[t]{0.24\textwidth}
        \centering
        \includegraphics[width=\linewidth]{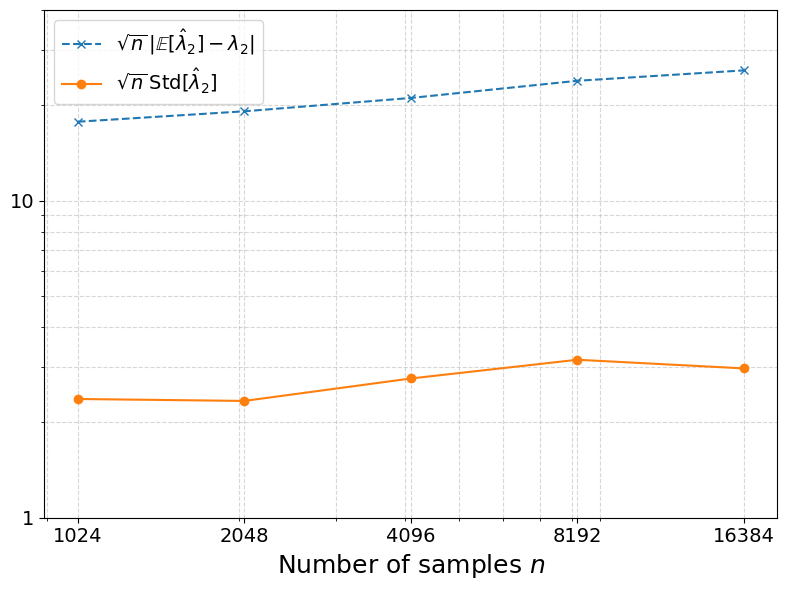}
        \caption{$m=3$}
        \label{fig:small_eps_m3}
    \end{subfigure}
    \hfill
    \begin{subfigure}[t]{0.24\textwidth}
        \centering
        \includegraphics[width=\linewidth]{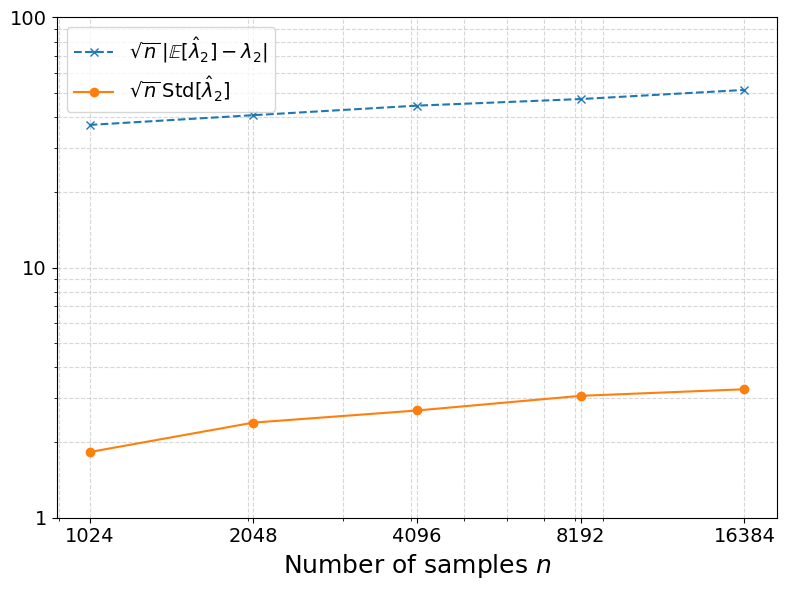}
        \caption{$m=4$}
        \label{fig:small_eps_m4}
    \end{subfigure}
    \caption{Relation of bias \(\sqrt{n}\bigl|\E[\hat{\lambda}_2^\eps]-\lambda_2\bigr|\) (blue dashed) and standard deviation of \(\hat{\lambda}_2^\eps\)  (orange solid) on an \(m\)-sphere when \(\eps=5n^{-\frac1{m+0.5}}\), for \(m=1,2,3,4\) (left to right).}
    \label{fig:bias_std_scale_eps=m+0.5}
\end{figure}


As a reconfirmation of our theoretical results, we further examine the asymptotic normality of $\sqrt{n}\hat\lambda_l^\eps$ in higher dimensions. Specifically, we consider an 8-sphere and assess how closely the finite-sample distribution of $\sqrt{n}\hat\lambda_l^\eps$, for $\eps=n^{-\frac{1}{m+4}}$, approximates a standard normal distribution. For a quantitative evaluation, we employ the Shapiro-Wilk test for normality. With sample sizes $n \in \{16000, 22000, 28000, 32000\}$, the corresponding $p$-values are $2.2 \cdot 10^{-16},\ 2.7 \cdot 10^{-16},\ 2.9 \cdot 10^{-9}$ and $0.26$, respectively. These results indicate that, as the sample size $n$ increases, the distribution of $\sqrt{n}\hat\lambda_l^\eps$ converges toward normality. In particular, for $n = 32000$, the Shapiro-Wilk test will not reject the hypothesis of normality, providing empirical support for asymptotic normality at higher dimensions.

\subsection{Structure of Asymptotic Covariance}\label{ss:cov} In this subsection, we explore the structure of asymptotic covariance matrix defined in \Cref{thm:multi normal distribution}, where the eigengap condition (in \Cref{assum:multiple CLT}) is not fulfilled. Unlike the previous subsections, we employ numerical methods instead of Monte-Carlo simulations. 

We consider the case of $2$-sphere, with $\rho$ taken as the uniform distribution. Due to the multiplicity of eigenvalues, we associate the asymptotic variance $\sigma_l^2$ with individual eigenfunctions rather than eigenvalues. We use the polar coordinate and discretize the whole sphere into a $20000\times 40000$ grid. We compute the discrete gradient on this grid and consider the usual spherical harmonics $Y_k^m$, for $k = 0,1,\ldots$ and $m = -k,\ldots, 0, \ldots, k$, which are eigenfunctions of the Laplace-Beltrami operator. The first 42 variances, ordered by the indices $k$ and $m$ corresponding to the eigenfunctions $Y_k^m$, are reported in \cref{tab:LB_2sphere}.

\begin{table}[h!]
\centering{\footnotesize
\begin{tabular}{c c c}
\toprule
Degree & Eigenvalue& Asymptotic variance associated with eigenfunction \( Y_k^m \)  \\
 \( k \)  &  \(  k(k+1) \)  & scaled by $\bigl(\frac{25.275367}{{\pi}/{3}}\bigr)^2$ for \( m = -k, \ldots, 0 \) \\
\midrule
0 & 0 & \( 0 \) \\
1 & 2 & \( 12.800,12.800  \) \\
2 & 6 & \( 144.003, 144.001, 144.000 \) \\
3 & 12 & \( 760.635, 502.162, 719.264, 667.569\) \\
4 & 20 & \( 2587.584, 1571.661, 1592.385, 2317.991, 2056.767  \) \\
5 & 30 & \( 6791.030, 4087.567 , 3586.921, 4087.523, 5803.878,5002.905 \) \\
6 & 42 & \( 15045.945, 9149.534, 7507.502, 7590.121, 8923.161, 12345.352, 10429.092 \) \\
\bottomrule
\end{tabular}}
\caption{Asymptotic ``variances" corresponding to the first few eigenfunctions of the Laplace-Beltrami operator on a 2-sphere. Note that the variance associated with $Y_k^m$ equals that of $Y_{k}^{-m}$.}
\label{tab:LB_2sphere}
\end{table}

A direct observation is that eigenfunctions corresponding to the same eigenvalue do not necessarily share the same associated asymptotic variance (again, the latter interpreted as a direct evaluation of the expression for $\sigma_l^2$ at different eigenfunctions). This is consistent with the formula given in \cref{thm1}, which indicates that the expressions for the asymptotic variances associated with $Y_k^m$ for $k=0,\dots, k$ need not be equal. We also observe that applying a random rotation to the eigenfunctions  $\{Y_{k}^m\}_{m = -k}^k$ for $m=-k,\dots, l$ results in different values. This dependence on the specific choice of eigenfunctions complicates the formulation of a unique limiting distribution in general. However, our numerical computations suggest an interesting invariance: the sum of the asymptotic variances corresponding to a given eigenvalue appears to remain constant, regardless of the basis chosen for the associated eigenspace.

To further examine the strength of correlations among empirical eigenvalues $\hat\lambda_l^\eps$, we visualize the asymptotic covariance and correlation matrices of the first 49 eigenvalues (counting multiplicities in the way suggested earlier) on the unit $2$-sphere equipped with uniform measure; see \Cref{fig:heatmap2sphere}. The indices of asymptotic variances are listed according to the lexicographic order for the eigenfunctions, which lists eigenfunctions as $Y_{1,1}, Y_{2,-1},Y_{2,0}, Y_{2,1},\dots, Y_{7,6}$. In the resulting heatmaps, each block corresponds to an eigenspace associated with a distinct eigenvalue. The visualizations reveal that correlations within the same eigenspace tend to be stronger than those between different eigenspaces.

\begin{figure}[htbp]
 \centering
 \begin{subfigure}[t]{0.48\textwidth}
\includegraphics[width=\linewidth]{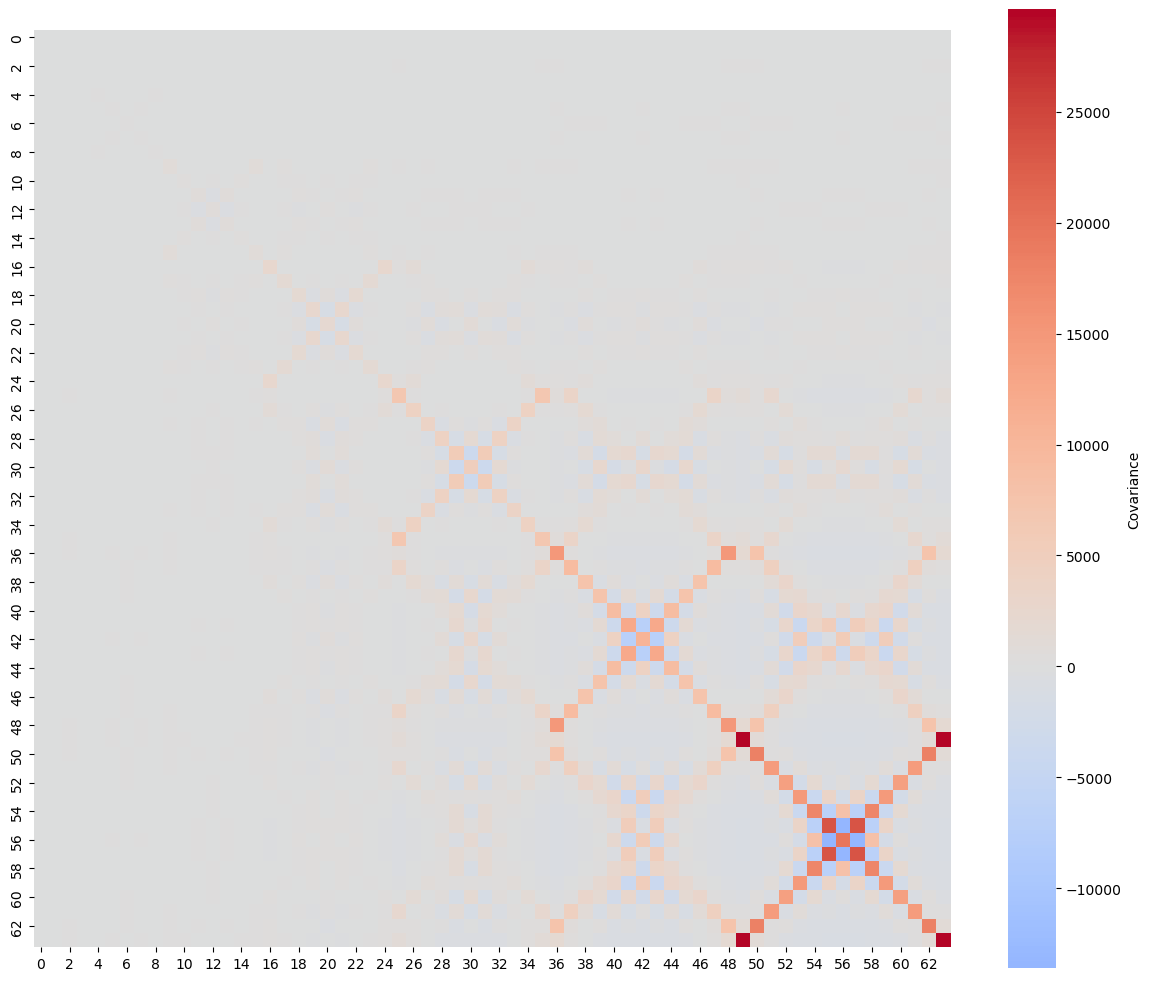}
\caption{Asymptotic covariance}
\label{fig:cov2sphere}
\end{subfigure}
 \begin{subfigure}[t]{0.48\textwidth}
\includegraphics[width=\linewidth]{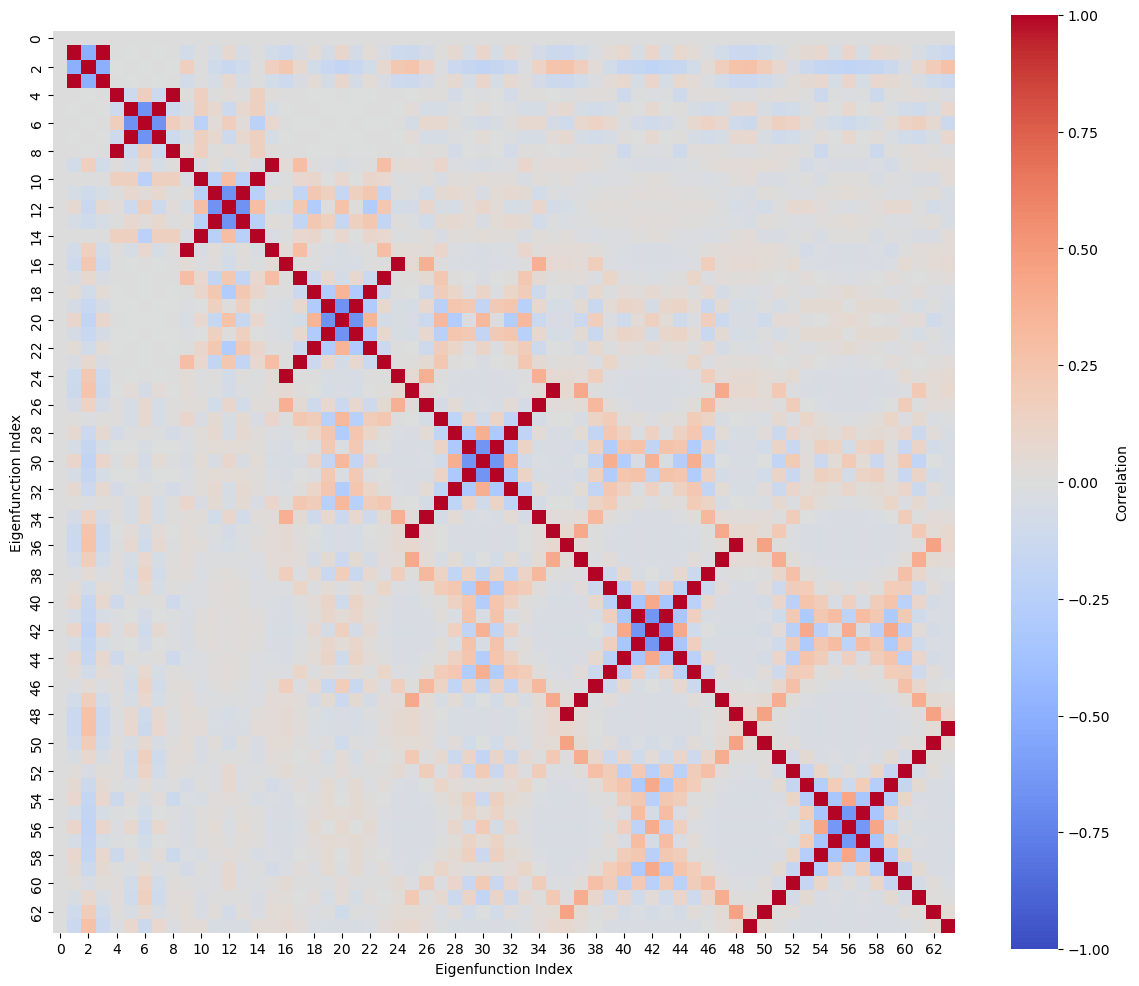}
\caption{Asymptotic correlation}
\label{fig:cor2sphere}
\end{subfigure}
\caption{Asymptotic covariance (left) and correlation (right) matrices associated with the first 63 eigenfunctions of the Laplace-Beltrami operator on a unit $2$-sphere. }
\label{fig:heatmap2sphere}
\end{figure}

For comparison, we also present the corresponding covariance and correlation matrices for the unit circle and the unit 2-torus in \Cref{fig:heatmap1sphere,fig:heatmap2torus}, respectively. In both cases, asymptotic correlation appears only within eigenspaces associated with the same eigenvalue. Interestingly, for the circle, the correlations are always nonpositive, either zero or close to $-1$, whereas for the torus, they are always nonnegative, either zero or close to $+1$.

\begin{figure}[htbp]
 \centering
 \begin{subfigure}[t]{0.48\textwidth}
\includegraphics[width=\linewidth]{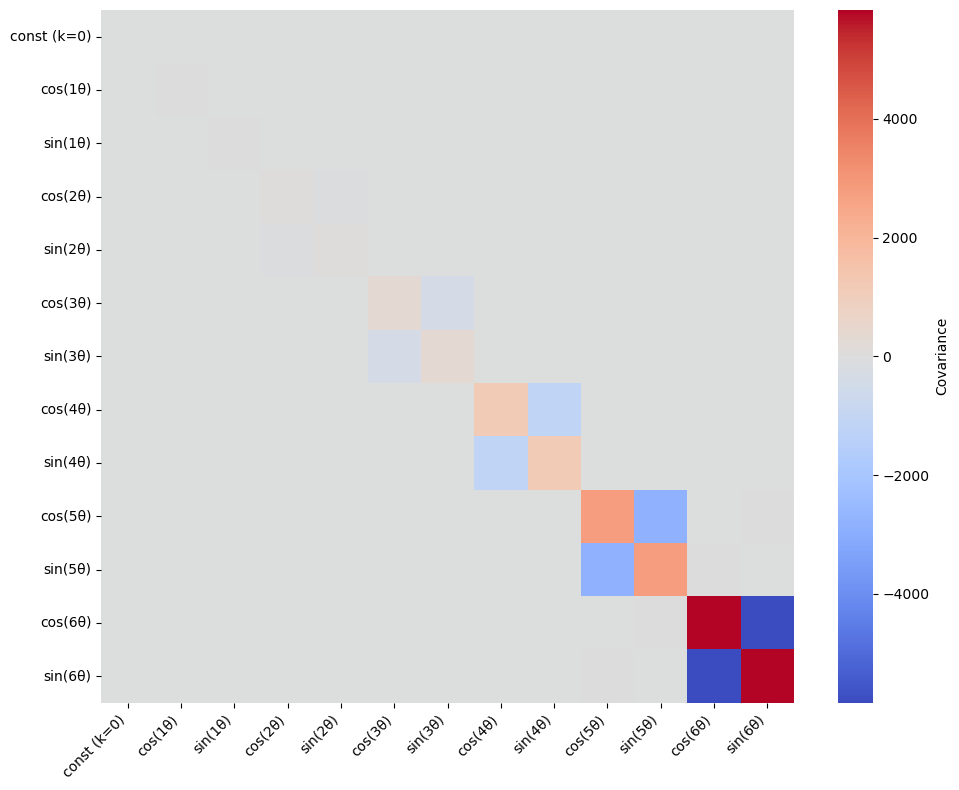}
\caption{Asymptotic covariance }
\label{fig:cov1sphere}
\end{subfigure}
 \begin{subfigure}[t]{0.48\textwidth}
\includegraphics[width=\linewidth]{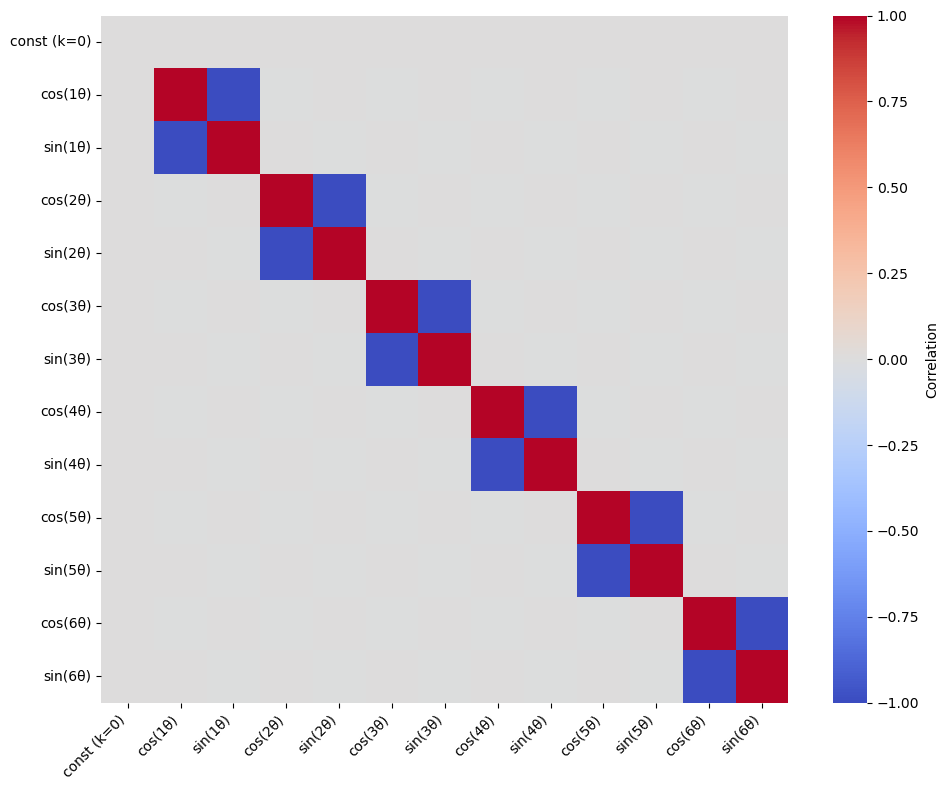}
\caption{Asymptotic correlation}
\label{fig:cor1sphere}
\end{subfigure}
\caption{Asymptotic covariance (left) and correlation (right) matrices associated with the first 13 eigenfunctions of the Laplace-Beltrami operator on a  unit circle. }
\label{fig:heatmap1sphere}
\end{figure}

\begin{figure}[htbp]
 \centering
 \begin{subfigure}[t]{0.48\textwidth}
\includegraphics[width=\linewidth]{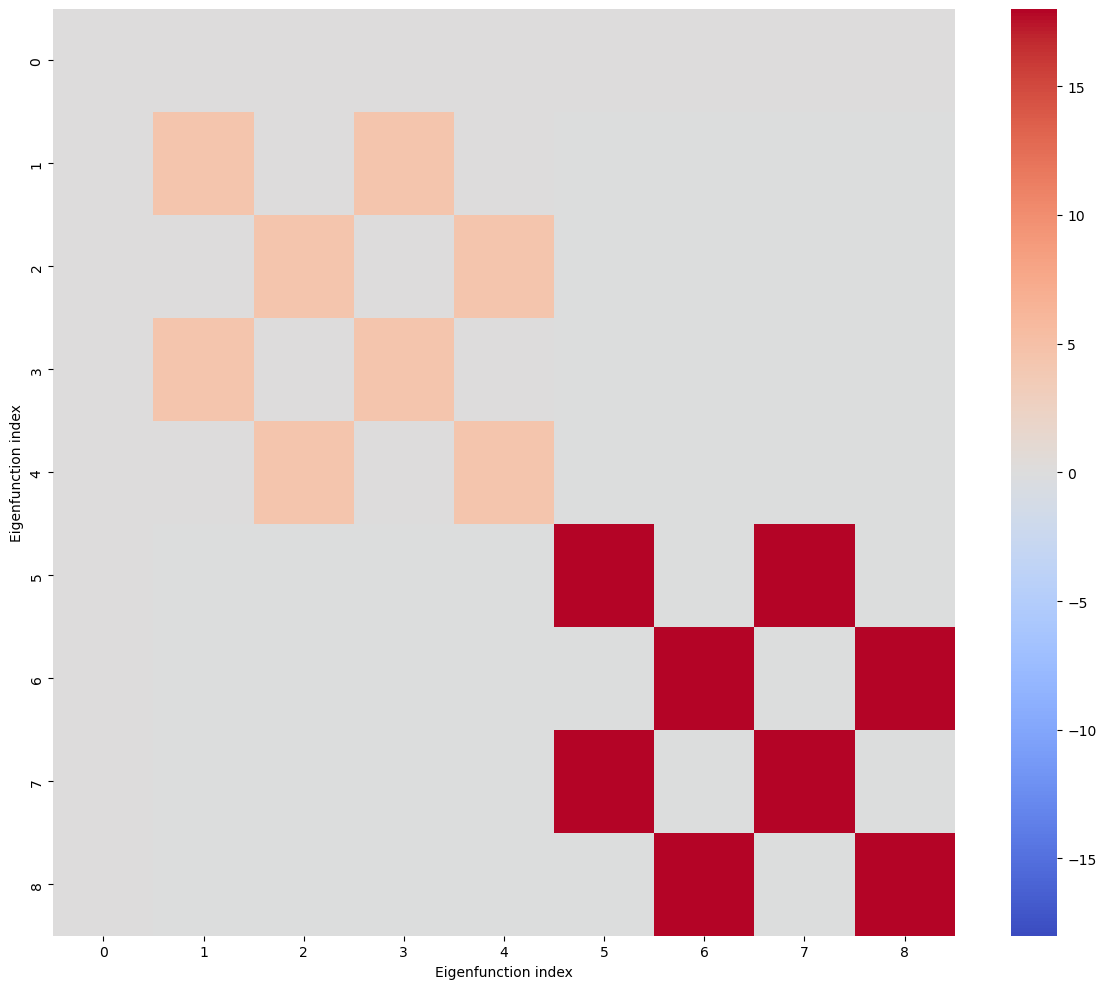}
\caption{Asymptotic covariance}
\label{fig:cov2torus}
\end{subfigure}
 \begin{subfigure}[t]{0.48\textwidth}
\includegraphics[width=\linewidth]{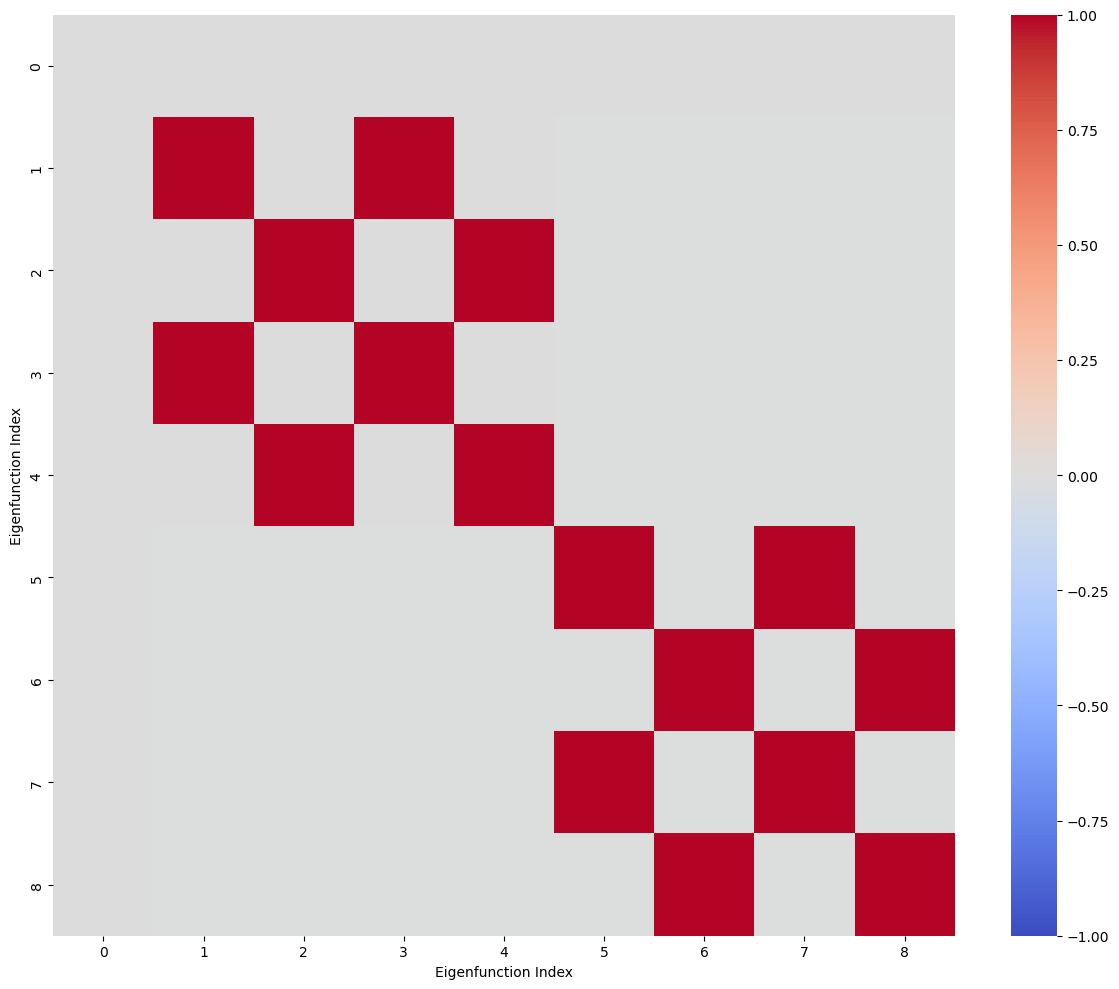}
\caption{Asymptotic correlation}
\label{fig:cor2torus}
\end{subfigure}
\caption{Asymptotic covariance (left) and correlation (right) matrices  associated with the first 9 eigenfunctions of the Laplace-Beltrami operator on a unit $2$-torus.
}
\label{fig:heatmap2torus}
\end{figure}

These observations suggest that the structure of asymptotic correlations is intimately connected to the underlying geometry of the manifold $\M$. A systematic investigation of this relationship presents an intriguing direction for future research.

\section{Conclusions}
\label{sec:Conclusions}

In this paper, we have established CLTs for the eigenvalues of graph Laplacians on data clouds under suitable assumptions on the connectivity parameter $\veps$ and on the data generating model. We have also provided geometric and statistical interpretations for the asymptotic variance appearing in these CLTs. In particular, we have shown that the asymptotic variances of the eigenvalues of the graph Laplacian can be interpreted as statistical lower bounds for the estimation of eigenvalues of certain differential operators. 

Several open questions are motivated by our work. First, we believe that it would be interesting to establish CLTs like the ones presented in this paper but in the case where the eigengap assumption at the population level fails. Our numerical experiments suggest that, even when this eigengap assumption is dropped, we should expect a CLT, although with a different expression for the variance. Secondly, it would be interesting to establish necessary and sufficient conditions under which the CLTs presented in this paper can be centralized around the eigenvalues of the limiting differential operators and not around the expectations of the eigenvalues of the graph Laplacian. As we discussed in subsection \ref{sec:bias}, and explored numerically in subsection \ref{sec:simulation}, we have reasons to believe that the dimension restriction $m \leq 3$ is necessary. Thirdly, it would be interesting to further explore the Cramer-Rao type lower bounds that we established in this paper, potentially extending them to estimators that are not necessarily unbiased. In particular, it would be interesting to prove, in a concrete and rigorous way, that eigenvalues of the graph Laplacian are statistically efficient within a reasonable class of estimators. Fourthly, numerical results in \cref{ss:cov} suggest a connection between the asymptotic correlations and the geometry of the underlying manifold, which merits further exploration. Finally, investigating CLTs for \textit{eigenvectors} of graph Laplacians on proximity graphs remains an interesting and unexplored problem.

To conclude, we would like to highlight that, while this paper has focused exclusively on analyzing the unnormalized graph Laplacian, our proof strategy is quite general and we expect a similar analysis to be applicable to other versions of the graph Laplacian. Particular examples of such generalizations to which we expect our overall approach to be applicable include the random walk graph Laplacian, the symmetric normalized graph Laplacian, graph Laplacians built over $k$-nearest neighbor graphs, etc; see  \cite{von2007tutorial,trillos2019error,calder2022lipschitz}. Analyzing the fluctuations of the eigenvalues of these more general graph-based operators is left as an interesting problem to explore in the future.

\bibliographystyle{abbrv}
\bibliography{references}

\begin{thebibliography}{10}

\bibitem{ando2006learning}
R.~Ando and T.~Zhang.
\newblock Learning on graph with laplacian regularization.
\newblock {\em Advances in neural information processing systems}, 19, 2006.

\bibitem{athreya2022eigenvalues}
A.~Athreya, J.~Cape, and M.~Tang.
\newblock Eigenvalues of stochastic blockmodel graphs and random graphs with
  low-rank edge probability matrices.
\newblock {\em Sankhya A}, pages 1--28, 2022.

\bibitem{athreya2018statistical}
A.~Athreya, D.~E. Fishkind, M.~Tang, C.~E. Priebe, Y.~Park, J.~T. Vogelstein,
  K.~Levin, V.~Lyzinski, Y.~Qin, and D.~L. Sussman.
\newblock Statistical inference on random dot product graphs: a survey.
\newblock {\em Journal of Machine Learning Research}, 18(226):1--92, 2018.

\bibitem{bal2008central}
G.~Bal.
\newblock Central limits and homogenization in random media.
\newblock {\em Multiscale Modeling \& Simulation}, 7(2):677--702, 2008.

\bibitem{belkin2001laplacian}
M.~Belkin and P.~Niyogi.
\newblock Laplacian eigenmaps and spectral techniques for embedding and
  clustering.
\newblock {\em Advances in neural information processing systems}, 14, 2001.

\bibitem{belkin2005towards}
M.~Belkin and P.~Niyogi.
\newblock Towards a theoretical foundation for {L}aplacian-based manifold
  methods.
\newblock In {\em International Conference on Computational Learning Theory},
  pages 486--500. Springer, 2005.

\bibitem{belkin2006manifold}
M.~Belkin, P.~Niyogi, and V.~Sindhwani.
\newblock Manifold regularization: A geometric framework for learning from
  labeled and unlabeled examples.
\newblock {\em Journal of machine learning research}, 7(11), 2006.

\bibitem{biskup2016eigenvalue}
M.~Biskup, R.~Fukushima, and W.~Konig.
\newblock Eigenvalue fluctuations for lattice anderson hamiltonians.
\newblock {\em SIAM Journal on Mathematical Analysis}, 48(4):2674--2700, 2016.

\bibitem{bousquet2011advanced}
O.~Bousquet, U.~von Luxburg, and G.~R{\"a}tsch.
\newblock {\em Advanced Lectures on Machine Learning: ML Summer Schools 2003,
  Canberra, Australia, February 2-14, 2003, T{\"u}bingen, Germany, August 4-16,
  2003, Revised Lectures}, volume 3176.
\newblock Springer, 2011.

\bibitem{bretagnolle1979estimation}
J.~Bretagnolle and C.~Huber.
\newblock Estimation des densit{\'e}s: risque minimax.
\newblock {\em Zeitschrift f{\"u}r Wahrscheinlichkeitstheorie und verwandte
  Gebiete}, 47:119--137, 1979.

\bibitem{BIK}
D.~Burago, S.~Ivanov, and Y.~Kurylev.
\newblock A graph discretization of the {L}aplace-{B}eltrami operator.
\newblock {\em Journal of Spectral Theory}, 4(4):675--714, 2014.

\bibitem{calder2019consistency}
J.~Calder.
\newblock Consistency of lipschitz learning with infinite unlabeled data and
  finite labeled data.
\newblock {\em SIAM Journal on Mathematics of Data Science}, 1(4):780--812,
  2019.

\bibitem{calder2019improved}
J.~Calder and N.~{Garc\'ia Trillos}.
\newblock Improved spectral convergence rates for graph laplacians on
  $\varepsilon$-graphs and k-nn graphs.
\newblock {\em Applied and Computational Harmonic Analysis}, 60:123--175, 2022.

\bibitem{calder2022lipschitz}
J.~Calder, N.~Garci\'ia~Trillos, and M.~Lewicka.
\newblock Lipschitz regularity of graph laplacians on random data clouds.
\newblock {\em SIAM Journal on Mathematical Analysis}, 54(1):1169--1222, 2022.

\bibitem{Cencov:62SM}
N.~N. Cencov.
\newblock Evaluation of an unknown distribution density from observations.
\newblock {\em Soviet Math.}, 3:1559--1562, 1962.

\bibitem{cencov2000statistical}
N.~N. Cencov.
\newblock {\em Statistical decision rules and optimal inference}.
\newblock American Mathematical Soc., 2000.

\bibitem{cheng2022eigen}
X.~Cheng and N.~Wu.
\newblock Eigen-convergence of gaussian kernelized graph laplacian by manifold
  heat interpolation.
\newblock {\em Applied and Computational Harmonic Analysis}, 61:132--190, 2022.

\bibitem{coifman2005geometric}
R.~R. Coifman, S.~Lafon, A.~B. Lee, M.~Maggioni, B.~Nadler, F.~Warner, and
  S.~W. Zucker.
\newblock Geometric diffusions as a tool for harmonic analysis and structure
  definition of data: Diffusion maps.
\newblock {\em Proceedings of the national academy of sciences},
  102(21):7426--7431, 2005.

\bibitem{do1992riemannian}
M.~P. Do~Carmo and F.~Flaherty.
\newblock {\em Riemannian geometry}, volume~6.
\newblock Springer, 1992.

\bibitem{DunsonWuWu}
D.~B. Dunson, H.~T. Wu, and N.~Wu.
\newblock Spectral convergence of graph laplacian and heat kernel
  reconstruction in $l^\infty$ from random samples.
\newblock {\em Applied and Computational Harmonic Analysis}, 55:282--336, 2021.

\bibitem{evans}
L.~C. Evans.
\newblock {\em Partial differential equations}, volume~19 of {\em Graduate
  Studies in Mathematics}.
\newblock American Mathematical Society, Providence, RI, 1998.

\bibitem{fan2022asymptotic}
J.~Fan, Y.~Fan, X.~Han, and J.~Lv.
\newblock Asymptotic theory of eigenvectors for random matrices with diverging
  spikes.
\newblock {\em Journal of the American Statistical Association},
  117(538):996--1009, 2022.

\bibitem{FernndezReal2022}
X.~Fernández-Real and X.~Ros-Oton.
\newblock {\em Regularity Theory for Elliptic PDE}.
\newblock EMS Press, Dec. 2022.

\bibitem{trillos2019error}
N.~{Garc{\'\i}a Trillos}, M.~Gerlach, M.~Hein, and D.~Slep{\v{c}}ev.
\newblock Error estimates for spectral convergence of the graph laplacian on
  random geometric graphs toward the laplace--beltrami operator.
\newblock {\em Foundations of Computational Mathematics}, pages 1--61, 2019.

\bibitem{trillos2023large}
N.~Garc\'{i}a~Trillos, P.~He, and C.~Li.
\newblock Large sample spectral analysis of graph-based multi-manifold
  clustering.
\newblock {\em Journal of Machine Learning Research}, 24(143):1--71, 2023.

\bibitem{trillos2018variational}
N.~{Garc\'{i}a Trillos} and D.~Slep{\v{c}}ev.
\newblock A variational approach to the consistency of spectral clustering.
\newblock {\em Applied and Computational Harmonic Analysis}, 45(2):239--281,
  2018.

\bibitem{trillos2024}
N.~García~Trillos, C.~Li, and R.~Venkatraman.
\newblock Minimax rates for the estimation of eigenpairs of weighted
  laplace-beltrami operators on manifolds.
\newblock {\em arXiv preprint arXiv:2506.00171}, 2025.

\bibitem{GK}
E.~Gin{\'e} and V.~Koltchinskii.
\newblock Empirical graph {L}aplacian approximation of {L}aplace-{B}eltrami
  operators: large sample results.
\newblock In {\em High dimensional probability}, volume~51, pages 238--259.
  JSTOR, 2006.

\bibitem{gu2016scaling}
Y.~Gu and J.-C. Mourrat.
\newblock Scaling limit of fluctuations in stochastic homogenization.
\newblock {\em Multiscale Modeling \& Simulation}, 14(1):452--481, 2016.

\bibitem{haochen2021provable}
J.~Z. HaoChen, C.~Wei, A.~Gaidon, and T.~Ma.
\newblock Provable guarantees for self-supervised deep learning with spectral
  contrastive loss.
\newblock {\em Advances in Neural Information Processing Systems}, 34, 2021.

\bibitem{hein2007graph}
M.~Hein, J.-Y. Audibert, and U.~v. Luxburg.
\newblock Graph laplacians and their convergence on random neighborhood graphs.
\newblock {\em Journal of Machine Learning Research}, 8(Jun):1325--1368, 2007.

\bibitem{hein2005graphs}
M.~Hein, J.-Y. Audibert, and U.~Von~Luxburg.
\newblock From graphs to manifolds--weak and strong pointwise consistency of
  graph laplacians.
\newblock In {\em International Conference on Computational Learning Theory},
  pages 470--485. Springer, 2005.

\bibitem{Kato}
T.~Kato.
\newblock {\em Perturbation theory for linear operators}, volume Band 132 of
  {\em Die Grundlehren der mathematischen Wissenschaften}.
\newblock Springer-Verlag New York Inc., 1966.

\bibitem{khas1979lower}
R.~Z. Khas’~minskii.
\newblock A lower bound on the risks of non-parametric estimates of densities
  in the uniform metric.
\newblock {\em Theory of Probability \& Its Applications}, 23(4):794--798,
  1979.

\bibitem{li2025consistency}
C.~Li and A.~M. Neuman.
\newblock Consistency of augmentation graph and network approximability in
  contrastive learning.
\newblock {\em arXiv preprint arXiv:2502.04312}, 2025.

\bibitem{li2023spectral}
C.~Li, R.~Sonthalia, and N.~Garc\'{i}a~Trillos.
\newblock Spectral neural networks: Approximation theory and optimization
  landscape.
\newblock {\em arXiv preprint arXiv:2310.00729}, 2023.

\bibitem{Lu2019GraphAT}
J.~Lu.
\newblock Graph approximations to the laplacian spectra.
\newblock {\em Journal of Topology and Analysis}, 14(01):111--145, 2022.

\bibitem{ng2001spectral}
A.~Ng, M.~Jordan, and Y.~Weiss.
\newblock On spectral clustering: Analysis and an algorithm.
\newblock {\em Advances in neural information processing systems}, 14, 2001.

\bibitem{PowellStockStoker1989}
J.~L. Powell, J.~H. Stock, and T.~M. Stoker.
\newblock Semiparametric estimation of index coefficients.
\newblock {\em Econometrica}, 57(6):1403--1430, 1989.

\bibitem{santambrogio2015optimal}
F.~Santambrogio.
\newblock Optimal transport for applied mathematicians.
\newblock {\em Birk{\"a}user, NY}, 55(58-63):94, 2015.

\bibitem{singer2006graph}
A.~Singer.
\newblock From graph to manifold laplacian: The convergence rate.
\newblock {\em Applied and Computational Harmonic Analysis}, 21(1):128--134,
  2006.

\bibitem{smola2003kernels}
A.~J. Smola and R.~Kondor.
\newblock Kernels and regularization on graphs.
\newblock In {\em Learning Theory and Kernel Machines: 16th Annual Conference
  on Learning Theory and 7th Kernel Workshop, COLT/Kernel 2003, Washington, DC,
  USA, August 24-27, 2003. Proceedings}, pages 144--158. Springer, 2003.

\bibitem{tang2022asymptotically}
M.~Tang, J.~Cape, and C.~E. Priebe.
\newblock Asymptotically efficient estimators for stochastic blockmodels: The
  naive mle, the rank-constrained mle, and the spectral estimator.
\newblock {\em Bernoulli}, 28(2):1049--1073, 2022.

\bibitem{tang2025eigenvector}
M.~Tang and J.~R. Cape.
\newblock Eigenvector fluctuations and limit results for random graphs with
  infinite rank kernels.
\newblock {\em arXiv preprint arXiv:2501.15725}, 2025.

\bibitem{tang2016limit}
M.~Tang and C.~E. Priebe.
\newblock Limit theorems for eigenvectors of the normalized laplacian for
  random graphs.
\newblock {\em The Annals of Statistics}, 46(5):2360--2415, 2018.

\bibitem{Shi2015}
W.~Tao and Z.~Shi.
\newblock Convergence of laplacian spectra from random samples.
\newblock {\em Journal of Computational Mathematics}, 38(6):952--984, 2020.

\bibitem{ting2010analysis}
D.~Ting, L.~Huang, and M.~I. Jordan.
\newblock An analysis of the convergence of graph laplacians.
\newblock In {\em Proceedings of the 27th International Conference on
  International Conference on Machine Learning}, page 1079–1086, 2010.

\bibitem{von2007tutorial}
U.~Von~Luxburg.
\newblock A tutorial on spectral clustering.
\newblock {\em Statistics and computing}, 17(4):395--416, 2007.

\bibitem{WormellReich}
C.~L. Wormell and S.~Reich.
\newblock Spectral convergence of diffusion maps: Improved error bounds and an
  alternative normalization.
\newblock {\em SIAM Journal on Numerical Analysis}, 59(3):1687--1734, 2021.

\end{thebibliography}

\appendix
	
\section{Some notions from Riemannian geometry}

\label{App:GeoBack}

Let $\M$ be a smooth and compact $m$-dimensional manifold embedded in $\R^d$. For a given $x \in \M$, the \textit{exponential map} $\exp_x$ at the point $x$ is the map $\exp_x : T_x \M \rightarrow \M$ (here $T_x \M$ denotes the tangent plane at $x$) with the property that, for every $v\in T_x \M $, the curve $t \in \R_+ \mapsto \exp_x(t v) $ is the unique constant speed geodesic that starts at $x$ with initial velocity $v$.
	
	It turns out that, for small enough $r>0$, the exponential map 
	\begin{equation} \label{expmap}
	{\rm exp}_{x}:  B_r(0) \subseteq T_{x}\mathcal{M} \to B_\M(x,r) \subseteq\mathcal{M}
	\end{equation}
	is a diffeomorphism between the~$m$-dimensional Euclidean ball $B_r(0)$ in the tangent space $T_{x}\mathcal{M}$ and the geodesic ball of radius $r$ centered at $x$, which we denote by $B_\M(x,r)$. The \textit{injectivity radius} $i_\M$ is the largest $r$ such that all the exponential maps $\{\exp_x\}_{x \in \M}$ are diffeomorphisms, as described above. For $r<i_\M$, we can thus introduce the diffeomorphic inverse of $\exp_x$, the \textit{logarithmic map}
	\begin{equation} \label{logmap}
	{\rm log}_{x}:  B_\M(x,r)  \subseteq \M \to B_r(0)\subseteq T_{x}\mathcal{M}\,. 
	\end{equation}
	Given $y \in B_\M(x,r)$ (for $r < i_\M$), $v=\log_x(y) \in T_x \M$ can be interpreted as the initial velocity of the minimizing geodesic that at time $t=0$ starts at $x$ and at time $1$ ends at $y$ —i.e., the curve $t\in [0,1] \mapsto \exp_x(tv)$. Moreover, we have the relation 
	\[ d(y, x ) = |v|, \]
	and ${\rm exp}_{x}(0) = x.$ 
	
	By \textit{normal coordinates} around a point $x \in \M$, we simply mean the parameterization of $B_\M(x,r)$ via the exponential map $\exp_x$. Integrals of functions $g$ supported on  $B_\M(x,r)$ can be written, in normal coordinates, as
	\[ \int_{B_\M(x,r)} g(x) \dx =  \int_{ B_r(0) \subseteq T_x \M }  g(v) J_x(v) \dd v, \]
	where $J_x(\cdot)$ is the Jacobian of the exponential map, i.e.,
	\begin{equation*}
	J_x(v): = |\det D_v \bigl( \exp_x(v)\bigr)|\,.
	\end{equation*}
	Since~$D_v(\exp_x(0)) = I,$ it is well-known that the Jacobian admits a Taylor expansion about~$v = 0$ given by  
	\begin{equation} \label{e.Jactaylor}
	|J_x(v)-1|\leq C|v|^2,
	\end{equation}
	where $C$ only depends on \textit{scalar} curvature bounds on $\M$ (see \cite[Chapter 4]{do1992riemannian}), and is, in particular, uniform in~$x\in \M$; the latter fact follows from Rauch comparison theorem; see \cite[Section 2.2]{BIK}.

For $x,y\in\M$ such that $|x-y|\leq \e$ where $\e$ satisfies \eqref{assum:Epsilon}, it can be shown that 
	\begin{equation} \label{e.comparegeodesic}
	0 \le d(x,y)-|x-y|\leq C|x-y|^3\,,
	\end{equation}
	for a constant~$C > 0$ that depends on a bound on the \textit{second fundamental form} of $\M$ (as an embedded manifold in $\R^d$); see \cite[Chapter 6]{do1992riemannian} for a definition of the second fundamental form of a manifold embedded into another. Indeed, a geometric quantity that bounds the second fundamental form is the \textit{reach} of the manifold (denoted by $R_\M$), and it can be defined as the largest radius $r>0$ such that for every $x \in \R^d$ with $ \inf_{y \in \M}|x -y| \leq r$ there is a unique closest point to $x$ in $\M$; see \cite[Proposition 2]{trillos2019error}).

\section{Regularity of eigenfunctions of $\Delta_\rho$}
\label{sec:Regularity}
In normal coordinates around a given point $x \in \M$, the operator $\Delta_\rho$ can be written as
\begin{equation}
\label{eqn:LaplacianCoordinates}
    (\Delta_\rho f)(\exp_x(v)) = - \frac{1}{\rho \sqrt{\det(g)}} \sum_{i=1}^m \sum_{j=1}^m \frac{\partial }{\partial v_i} \left( \sqrt{\det(g)} \rho^2 g^{ij} \frac{\partial}{\partial v_j} f(\exp_x(v)) \right)  ,
\end{equation}
where $g^{ij}$ are the components of the inverse of the matrix $g(v):= D_v(\exp_x(v)) \cdot  D_v(\exp_x(v))^{{T}}$ (i.e., the metric tensor in normal coordinates.). For $\M$ and $\rho$ satisfying Assumption \ref{assump:DataGenerating}, it follows that the coefficients $\{ a_{ij}\}_{ij}$ defined by
\[ a_{ij}:= \sqrt{\det(g)} \rho^2 g^{ij}  \]
fulfill the uniform ellipticity condition in \cite[Equation (2.32)]{FernndezReal2022} and are such that their $C^{2,\alpha}$ norm is controlled. Using a standard iteration procedure, we conclude from standard Schauder estimates (see, e.g., \cite[Corollary 2.29]{FernndezReal2022}) that the solution $u_l$ of the equation
\[ \Delta_\rho u_l = \lambda_l u_l \]
is at least $C^{3,\alpha}$.

\label{rem:RegCoeffciients}

\nc

	\section{Estimates From Perturbation Theory of Elliptic Operators}
	\label{app:EstimatesPerturbationTheory}
	
Let $\rho \in \mathcal{P}_\M^\alpha$ be such that $\lambda_l(\rho)$ is simple and suppose that $\xi$ is a $C^{2, \alpha}(\M)$ function with
\[ \int_\M \xi(x) \rho(x) \dx =0.  \]
In particular, in the notation from section \ref{sec:FisherRao} we have $\xi \in \T_\rho \mathcal{P}(\M)$.
	For a~$t \in \R$ with~$|t| $ small enough, we define $\rho_t$ via
	\begin{equation*}
	\rho_t:= \rho_0+t \xi \rho \,. 
	\end{equation*}
	Since~$\int_\M\xi \rho \dx = 0$ it follows that~$\int_\M \rho_t\dx = \int_\M \rho\dx=1$ for every~$t.$ Moreover, for~$|t|$  small,~$\rho_t \in \mathcal{P}^\alpha_{\M}$. 
    
    Let us now introduce the family of operators
	\begin{equation*}
	\mathcal{L}_t:= \Delta_{\rho_t}= -\frac{1}{\rho_t} \div(\rho_t^2\nabla \cdot  )\,, \quad t \in \R.
	\end{equation*}
	Then, clearly, for small enough $t$~$\mathcal{L}_t$ defines a family of elliptic second-order differential operators that depend analytically on~$t$ in a neighborhood of~$t = 0.$ It is well-known~(see \cite{Kato}) that if~$\lambda_0 = \lambda_{0,l}$ is a \emph{simple} eigenvalue of~$\Delta_{\rho_0}$ then there exists~$t_{l} > 0$ and an analytic branch~$\{\lambda_{t,l}, u_{t,l}\}_{|t|< t_{l}}$ of \emph{simple} eigenvalue-eigenfunction pairs for the operator~$\mathcal{L}_t.$ Our goal is to compute the infinitesimal quantity
	\begin{equation*}
	\dot{\lambda}  :=  \frac{\,d}{\,dt}\big\vert_{t=0} \lambda_{t,l}.
	\end{equation*}
	
	For notational ease, we set~$\dot{u}  :=  \frac{\,d}{\,dt}\big\vert_{t=0} u_{t,l}.$ Since we have 
	\begin{equation}
	\label{e.LtPDE}
	-\mathrm{div}\bigl( \rho_t^2 \nabla u_{t,l} \bigr) = \lambda_{t,l}\rho_{t,l}u_{t,l}\,,
	\end{equation}
	with the normalization
	\begin{equation} \label{e.Ltnormal}
	\int_\M u_{t,l}^2 \rho_t  \dx = 1\,
	\end{equation}
	for~$|t|$ small, we may differentiate \eqref{e.LtPDE} with respect to~$t$ and set~$t=0$ to obtain  
	\begin{equation} \label{e.phidot}
	-\mathrm{div}\bigl( \rho^2 \nabla \dot{u}\bigr) - \lambda \rho \dot{u} = 2 \mathrm{div}\bigl( \rho \dot{\rho} \nabla u\bigr) + \bigl(\lambda \dot{\rho} + \dot{\lambda} \rho \bigr)u\,,
	\end{equation}
	where we write~ $\lambda:= \lambda_{0,l}$ and $u :=  u_{0,l}$ for brevity. Also, $\dot{\rho}= \xi \rho$.
	
	By the Fredholm alternative applied to the operator $-\div(\rho^2 \nabla \cdot) - \lambda \rho \cdot $, and the fact that~$\lambda$ was assumed to be simple, the right hand side of~\eqref{e.phidot} must be orthogonal to~$u$ in the standard (i.e., unweighted) $L^2$ inner product (since $u$ is in the kernel of the aforementioned operator). So, testing this equation with~$u$ yields 
	\begin{equation*}
	0 = - 2\int_{\M} \rho \dot{\rho} |\nabla u|^2 \dx + \lambda \int_\M \dot{\rho} u^2 \dx + \dot{\lambda}.
	\end{equation*}
	Rearranging, we obtain
	\begin{equation}
	\frac{d}{dt}|_{t=0} \lambda_l(\rho_t) = \dot{\lambda} = - \lambda_{l}(\rho) \int_{\M} \xi u^2 \rho \dx + 2 \int_{\M} \xi  |\nabla u|^2 \rho^2\,\dx\,.  
    \label{eqn:PerturbationEigenvalues}
	\end{equation}

\nc





\end{document}